\documentclass[twoside]{article}

\usepackage[accepted]{aistats_template/aistats2023}
\usepackage[round]{natbib}

\usepackage[dvipsnames,table,]{xcolor}
\usepackage[colorlinks,linktoc=all]{hyperref}
\usepackage[all]{hypcap}
\hypersetup{citecolor=MidnightBlue}
\hypersetup{linkcolor=MidnightBlue}
\hypersetup{urlcolor=MidnightBlue}

\usepackage{comment}
\usepackage{algorithmic}
\usepackage{algorithm2e}
\RestyleAlgo{ruled}

\usepackage{amssymb}
\usepackage{amsmath}

\usepackage{graphicx}
\usepackage{subcaption}
\usepackage{wrapfig}
\usepackage{booktabs}
\usepackage{multirow}

\newcommand{\mat}[1]{\boldsymbol{#1}}
\newcommand{\norm}[1]{\left\lVert\mat{#1}\right\rVert}
\newcommand{\dotprod}[2]{\mat{#1}^{\top} \mat{#2}}
\newcommand{\dotprodi}[3]{\mat{#1}_{#3}^{\top} \mat{#2}}

\newcommand{\bigO}{\mathcal{O}}
\newcommand{\iu}{\mathrm{i}\mkern1mu}

\newcommand{\BlackBox}{\rule{1.5ex}{1.5ex}}  
\newenvironment{proof}{\par\noindent{\bf Proof\ }}{\hfill\BlackBox\\[2mm]}
 
\newtheorem{theorem}{Theorem}[section]
\newtheorem{lemma}[theorem]{Lemma}

\newtheorem{corollary}[theorem]{Corollary}

\begin{document}

%

\runningtitle{Complex-to-Real Sketches for Tensor Products}

%
\runningauthor{Jonas Wacker, Ruben Ohana, Maurizio Filippone}

\twocolumn[


\aistatstitle{Complex-to-Real Sketches for Tensor Products\\with Applications to the Polynomial Kernel}

\aistatsauthor{ Jonas Wacker \And Ruben Ohana \And  Maurizio Filippone }

\aistatsaddress{ EURECOM, France \And  CCM, Flatiron Institute, USA \And EURECOM, France } ]

\begin{abstract}
Randomized sketches of a tensor product of $p$ vectors follow a tradeoff between statistical efficiency and computational acceleration. Commonly used approaches avoid computing the high-dimensional tensor product explicitly, resulting in a suboptimal dependence of $\bigO(3^p)$ in the embedding dimension. We propose a simple Complex-to-Real (CtR) modification of well-known sketches that replaces real random projections by complex ones, incurring a lower $\bigO(2^p)$ factor in the embedding dimension. The output of our sketches is real-valued, which renders their downstream use straightforward. In particular, we apply our sketches to $p$-fold self-tensored inputs corresponding to the feature maps of the polynomial kernel. We show that our method achieves state-of-the-art performance in terms of accuracy \textit{and} speed compared to other randomized approximations from the literature.
\end{abstract}

\section{INTRODUCTION}

Randomized linear sketching \citep{Woodruff2014} is a computationally efficient method for dimensionality reduction, where an input point $\mat{x} \in \mathbb{R}^d$ is multiplied by a random $D$-by-$d$ matrix $\mat{S}$ to yield a low-distortion embedding. When $D \ll d$, the sketched data is more compact, accelerating downstream learning algorithms with statistical guarantees. It is well-known that an optimal choice of $\mat{S}$ requires an embedding dimension $D = \Theta(\log(1/\delta) \epsilon^{-2})$ to guarantee that $\norm{Sx}_2$ lies within $(1 \pm \epsilon) \norm{x}_2$ with probability at least $1-\delta$ \citep{Larsen2017}.

Here we consider sketches of tensor products $\otimes_{i=1}^p \mat{x}_i$ for some arbitrary vectors $\mat{x}_1 \in \mathbb{R}^{d_1}, \dots, \mat{x}_p \in \mathbb{R}^{d_p}$. Storing $\otimes_{i=1}^p \mat{x}_i$ takes $\bigO(\prod_{i=1}^p d_i)$ memory and becomes infeasible when $p$ or $\{d_i\}_{i=1}^p$ are moderately large, impeding the construction of an explicit sketch. To solve this problem, \textit{implicit} sketching methods have been developed in the past \citep[e.g.,][]{Kar2012, Pham2013} that compute $\mat{S} (\otimes_{i=1}^p \mat{x}_i)$ without ever forming $\otimes_{i=1}^p \mat{x}_i$.

Sketches for tensor products have been successfully applied to compress deep neural networks for the tasks of fine-grained visual recognition \citep{Gao2016} and multi-modal fusion \citep{Fukui2016}. Furthermore, when considering the special case of self-tensored inputs (we set $\mat{x} := \mat{x}_1 = \dots = \mat{x}_p$), then $\otimes_{i=1}^p \mat{x}_i$ corresponds to the feature map of the polynomial kernel. For two inputs $\mat{x}, \mat{y} \in \mathbb{R}^d$, the sketch thus yields a randomized approximation $\hat{k}(\mat{x}, \mat{y}) = (\mat{S} (\otimes_{i=1}^p \mat{x}))^\top \mat{S} (\otimes_{i=1}^p \mat{y})$ of the polynomial kernel $k(\mat{x}, \mat{y}) = (\mat{x}^\top \mat{y})^p$. This observation connects these sketching methods to random feature maps originally proposed for shift-invariant kernels \citep{Rahimi2007}. Polynomial kernels are among the most popular kernels and have proven effective in applications such as natural language processing \citep{Goldberg08}, recommender systems \citep{Rendle2010}, and genomic data analysis \citep{Aschard2016}. Moreover, more general dot product kernels can be formulated as a positively weighted sum of polynomial kernels through a Taylor expansion \citep{Kar2012}. An extended version of this expansion also exists for the Gaussian kernel \citep{Cotter2011}.

Although it is of high interest to accelerate the aforementioned applications via sketching, commonly used methods proposed in the past require a suboptimal embedding dimension $D = \bigO(3^p \log(1/\delta) \epsilon^{-2})$ as shown by \citet{Avron2014} and \citet[Appendix A.2]{Ahle2020}, thus
trading statistical efficiency for computational accelerations.
\citet{Ahle2020} improve the dependence on $p$ to polynomial by composing well-known base sketches, but require a more expensive meta-algorithm \citep{Song21c}.

In this work, we address this issue from another angle by studying simple complex-valued modifications of existing sketches. These can yield much lower variances as shown in \citet{wacker2022}, but may render a downstream task such as ridge regression more expensive due to linear algebra operations being applied to complex data. Moreover, \citet{wacker2022} do not provide guarantees on the preservation of the L2-norm, nor do they provide an intuitive explanation for the improved statistical properties of such sketches.
In this sense, our work continues where the previous work falls short. We show that complex sampling distributions have smaller higher-order moments than real-valued analogs while also yielding valid sketches, and we provide an in-depth analysis of resulting theoretical guarantees. We further show that a concatenation of the real and imaginary parts of a complex sketch inherits its statistical advantages and we call the real-valued result a {\em Complex-to-Real (CtR)} sketch. CtR-sketches are simple to construct and can be used in any downstream task without requiring the model to handle complex data.

More precisely, we make the following main contributions: 1) In Section~\ref{sec:complex-bounds}, we show that complex sketches preserve the L2-norm of an input vector using only $D=\bigO(2^p)$ instead of $D=\bigO(3^p)$ required by their real analogs, while explaining the intuition for this improvement. 2) In Section~\ref{sec:ctr-concentration}, we show that these results readily extend to CtR-sketches resulting in the same guarantees for the approximate matrix product. 3) In Section~\ref{sec:ctr-variance}, we focus on polynomial kernels and derive the variances of kernel approximations obtained through CtR-sketches, while comparing them against real-valued analogs. 4) In Section~\ref{sec:experiments}, we empirically compare a newly developed structured CtR-sketch against the state-of-the-art.

We made the code for this work publicly available.\footnote{\url{https://github.com/joneswack/dp-rfs}}

\section{PRELIMINARIES}

\label{sec:preliminaries}

\paragraph{Notation}

We denote the tensor product of two vectors $\mat{a}, \mat{b}$ as $\mat{a} \otimes \mat{b} = {\rm vec}(\mat{ab}^\top)$. For $p$ vectors $\{ \mat{a}_i \}_{i=1}^p$, we use $\otimes_{i=1}^p \mat{a}_i$. In particular, we write $\mat{a}^{\otimes p} := \otimes_{i=1}^p \mat{a}$ when this operation is applied to a vector with itself. For two matrices $\mat{A}, \mat{B}$, we denote their element-wise product as $\mat{A} \odot \mat{B}$. When they are positive semi-definite (psd), we write $\mat{A} \preceq \mat{B}$ if $\mat{B}-\mat{A}$ is psd. The Frobenius norm is defined as $\norm{A}_F = (\sum_{i,j} A_{i,j}^2)^{1/2}$. For a random variable $X$, we denote its expected value by $\mathbb{E}[X]$ and its variance by $\mathbb{V}[X]$. Its $L^t$-norm is $\| X \|_{L^t} = \mathbb{E} [|X|^t]^{1/t}$ for $t \geq 1$.


We define $\iu := \sqrt{-1}$. The real-valued standard normal distribution is defined as $\mathcal{N}(\mat{0}, \mat{I})$, and the complex one as $\mathcal{CN}(\mat{0}, \mat{I})$. The real Rademacher distribution is denoted by ${\rm Unif}(\{ 1, -1 \})$, and the complex one by ${\rm Unif }(\{ 1, -1, \iu, -\iu \})$. A Rademacher vector has its elements drawn i.i.d. from the Rademacher distribution.

\paragraph{Polynomial kernel}

In this work, we consider polynomial kernels of the form
\begin{equation} \label{eq:poly-kernels}
    k(\mat{x}, \mat{y}) = (\gamma \mat{x}^\top \mat{y} + \nu  )^p
\end{equation}
for some $\mat{x}, \mat{y} \in \mathbb{R}^d$, where $\gamma, \nu \geq 0$ and $p \in \mathbb{N}$. Both parameters $\gamma$ and $\nu$ can be absorbed by the input vectors by setting $\tilde{\mat{x}} :=  (\sqrt{\gamma} \mat{x}^\top, \sqrt{\nu})^\top \in \mathbb{R}^{d+1}$ and $\tilde{\mat{y}} :=  (\sqrt{\gamma} \mat{y}^\top, \sqrt{\nu})^\top \in \mathbb{R}^{d+1}$. Therefore, without loss of generality, we assume the kernel to be {\em homogeneous}, i.e., it can be written as
\begin{equation}
    (\gamma \mat{x}^\top \mat{y} + \nu  )^p = (\tilde{\mat{x}}^\top \tilde{\mat{y}})^p
    = (\tilde{\mat{x}}^{\otimes p})^{\top} \tilde{\mat{y}}^{\otimes p}.
\end{equation}
Although its feature maps $\tilde{\mat{x}}^{\otimes p}, \tilde{\mat{y}}^{\otimes p}$ can be computed explicitly,
they are $(d+1)^p$-dimensional and therefore infeasible to construct when $d$ or $p$ are large. For $n$ data points, applying the kernel trick costs at least $\bigO(n^2)$ and is not possible when $n$ is large. This makes randomized sketching, i.e., reducing the dimensionality of $\tilde{\mat{x}}^{\otimes p}$ and $\tilde{\mat{y}}^{\otimes p}$ through linear random projections, an attractive choice.

\subsection{Sketching Tensor Products}

\label{sec:real-complex-sketches}

We study sketches of tensor products $\otimes_{i=1}^p \mat{x}_i$ for some $\mat{x}_1 \in \mathbb{R}^{d_1}, \dots, \mat{x}_p \in \mathbb{R}^{d_p}$. There exist several sketching techniques for this purpose (see Section \ref{sec:related}). Here we focus on the following construction.

We generate $p \times D$ i.i.d. random weights $\mat{w}_{i, \ell} \in \mathbb{C}^{d_i}$ satisfying $\mathbb{E} [\mat{w}_{i,\ell} \overline{\mat{w}_{i,\ell}}^{\top}] = \mat{I}_{d_i}$ for $i \in \{1, \dots, p\}, \ell \in \{1, \dots, D\}$, where $\mat{I}_{d_i}$ is the identity matrix of size $d_i$. E.g., $\mat{w}_{i, \ell}$ can be a (complex) Rademacher vector or be sampled from the (complex) standard normal distribution.

We define a sketch $\mat{S} = (\mat{s}_1, \dots, \mat{s}_D)^\top \in \mathbb{C}^{D \times d_1 \cdots d_p}$ with $\mat{s}_\ell = \otimes_{i=1}^p \mat{w}_{i, \ell} / \sqrt{D}$. A naive computation of $\mat{S} (\otimes_{i=1}^p \mat{x}_i)$ would cost $\bigO(D \prod_{i=1}^p d_i)$ time and memory, but we can exploit the following property of the tensor product:
\begin{align}
    \label{eqn:sketch-element}
    (\otimes_{i=1}^p \mat{w}_{i, \ell})^\top (\otimes_{i=1}^p \mat{x}_i)
    = \prod_{i=1}^p \mat{w}_{i, \ell}^\top \mat{x}_i
\end{align}
that lets us compute $\mat{S} (\otimes_{i=1}^p \mat{x}_i)$ in $\bigO(D \sum_{i=1}^p d_i)$ using the r.h.s. of Eq.~\ref{eqn:sketch-element}. In particular, $\otimes_{i=1}^p \mat{x}_i$ never needs to be constructed explicitly in this case.

Although our sketches are applicable to arbitrary tensor products, in this work we focus on feature maps of the polynomial kernel. That is, we set $\mat{x} = \mat{x}_1 = \cdots = \mat{x}_p$, such that $\otimes_{i=1}^p \mat{x}_i = \mat{x}^{\otimes p}$. For two inputs $\mat{x}, \mat{y} \in \mathbb{R}^d$, we define the approximate kernel $\hat{k}(\mat{x}, \mat{y}) := (\mat{S} \mat{x}^{\otimes p})^\top (\overline{\mat{S} \mat{y}^{\otimes p}})$, which is unbiased because
\begin{align*}
    \mathbb{E}\left[\hat{k}(\mat{x}, \mat{y})\right]
    = \frac{1}{D} \sum_{\ell=1}^D \prod_{i=1}^p \mat{x}^{\top} \mathbb{E} [\mat{w}_{i,\ell} \overline{\mat{w}_{i,\ell}}^{\top}] \mat{y}
    = (\dotprod{x}{y})^p.
\end{align*}
In this case, we may alternatively call $\Phi(\mat{x}) := \mat{S} \mat{x}^{\otimes p}$ a random feature map, which we express as
\begin{align}
    \label{eqn:polynomial-estimator}
    \Phi(\mat{x})
    =  (\mat{W}_1 \mat{x} \odot \cdots \odot \mat{W}_p \mat{x}) / \sqrt{D},
\end{align}
where $\mat{W}_i := (\mat{w}_{i, 1}, \dots, \mat{w}_{i, D})^\top$, to simplify the notation.

The random feature map (\ref{eqn:polynomial-estimator}) has originally been proposed by \citet{Kar2012} and been further studied in \citet{Hamid2014, Meister2019, Ahle2020} for the case of real-valued $\{ \mat{W}_i \}_{i=1}^p$. Recently, \citet[Thm. 3.1]{wacker2022} derived a variance lower bound for $\hat{k}(\mat{x}, \mat{y})$, which can be obtained through Rademacher weights. They further showed that lower variances can be achieved using more general complex-valued $\{ \mat{W}_i \}_{i=1}^p$ that subsume the real-valued case \citep[Thm. 3.3]{wacker2022}. Hereafter, we use $\Phi_{\rm R}$ to denote a real-valued and $\Phi_{\rm C}$ to denote a complex-valued random feature map, thus emphasizing their difference. The caveat of using $\Phi_{\rm C}$ is that it requires the downstream model to handle complex data, which may incur additional computational costs.

The purpose of this work instead, is to analyze the real-valued kernel estimate $\hat{k}_{\rm CtR}(\mat{x}, \mat{y}) := {\rm Re} \{ \hat{k}(\mat{x}, \mat{y}) \}$, which can be written as
\begin{align}
    \label{eqn:ctr-kernel}
    \hat{k}_{\rm CtR}&(\mat{x}, \mat{y})
    = {\rm Re}\{\Phi_{\rm C}(\mat{x})\}^\top {\rm Re} \{ \Phi_{\rm C}(\mat{y}) \} \\
    &+ {\rm Im} \{ \Phi_{\rm C}(\mat{x}) \}^\top {\rm Im} \{ \Phi_{\rm C}(\mat{y}) \} = \Phi_{\rm CtR}(\mat{x})^\top \Phi_{\rm CtR}(\mat{y}),\nonumber
\end{align}
where we call $\Phi_{\rm CtR}(\mat{x})$ a \textit{Complex-to-Real (CtR)} sketch. Since it is real-valued, it can be used as a drop-in replacement for any input to a downstream model. The downside of CtR-sketches is that they are $2D$-dimensional. In order to yield a fair comparison with real sketches, we reduce the dimension of CtR-sketches to $D$ by using half the number of rows for $\{ \mat{W}_i \}_{i=1}^p$ from now onward. We summarize the construction of CtR-sketches in Alg.~\ref{alg:ctr-algorithm}.




\begin{algorithm}[tb]
   \caption{Complex-to-Real (CtR) Sketches}
   \label{alg:ctr-algorithm}
\begin{algorithmic}
    \STATE {\bfseries Input:} Data point $\mat{x} \in \mathbb{R}^d$
    \STATE Choose dimension $D=2k$ ($k \in \mathbb{N}$), degree $p \in \mathbb{N}$
    \STATE Sample $\{\mat{W}_i\}_{i=1}^p$ with $\mat{W}_i \in \mathbb{C}^{D/2 \times d}$ independently according to one of the following sketch distributions:
    \begin{itemize}
        \item Gaussian: $(\mat{W}_i)_{\ell, k} \stackrel{i.i.d.}{\sim} \mathcal{CN}(0,1)$
        \item Rademacher: $(\mat{W}_i)_{\ell, k} \stackrel{i.i.d.}{\sim} {\rm Unif}(\{1, -1, \iu, -\iu\})$
        \item ProductSRHT: $\mat{W}_i=\mat{P}_i \mat{H} \mat{D}_i$ (see Appendix~\ref{sec:app-structured-sketches})
    \end{itemize}
    \STATE Compute $\Phi_{\rm C}(\mat{x}) := \sqrt{2/D} \, (\mat{W}_1 \mat{x} \odot \dots \odot \mat{W}_p \mat{x})$
    \STATE {\bfseries Return:}
    \STATE $\Phi_{\rm CtR}(\mat{x}) := ( {\rm Re}\{\Phi_{\rm C}(\mat{x})_1\}, \dots, {\rm Re}\{\Phi_{\rm C}(\mat{x})_{D/2}\},$
    \STATE $\quad\quad\quad\quad {\rm Im}\{\Phi_{\rm C}(\mat{x})_1\}, \dots, {\rm Im}\{\Phi_{\rm C}(\mat{x})_{D/2}\} )^{\top} \in \mathbb{R}^{D}$
\end{algorithmic}
\end{algorithm}

\section{ANALYSIS OF CtR-SKETCHES}

The following section is dedicated to the theoretical analysis of CtR-sketches for tensor products and feature maps of the polynomial kernel. For the first part of our analysis, we treat them as \textit{linear} sketches in a high-dimensional tensor-product space (see Section \ref{sec:preliminaries}). In the second part, we focus on the variances of CtR-sketches for the particular case of feature maps of the polynomial kernel in order to obtain useful insights for their practical application.

\subsection{Concentration Bounds for Complex Sketches}

\label{sec:complex-bounds}

We start by analyzing the complex sketch $\mat{S} = (\mat{s}_1, \dots, \mat{s}_D)^\top \in \mathbb{C}^{D \times d_1 \cdots d_p}$ (see Section \ref{sec:preliminaries}). Recall that $\mat{s}_\ell = \otimes_{i=1}^p \mat{w}_{i, \ell} / \sqrt{D}$ with $\mat{w}_{i,\ell} \in \mathbb{C}^{d_i}$ i.i.d. $\mat{x} \in \mathbb{R}^{d_1 \cdots d_p}$ can take on any value that may not necessarily result from a tensor product in our analysis. This makes our results more general and is required to derive the spectral guarantee (\ref{eqn:spectral-order}) in Section \ref{sec:ctr-concentration}. Moreover, we only study (complex) Gaussian/Rademacher distributions for $\mat{w}_{i,\ell}$ here since Rademacher distributions achieve a variance lower bound for the sketch in Eq.~\ref{eqn:polynomial-estimator} as we show later in Thm.~\ref{thrm:rad-ctr-advantage}.


The following key lemma shows that $\mat{s}_\ell^\top \mat{x}$ has lower absolute moments if $\mat{s}_\ell$ is sampled from a complex Gaussian/Rademacher distribution instead of a real one. It is an extension of \citet[Lem. 19]{Ahle2020} to complex $\mat{s}_\ell$.

\begin{lemma}[Absolute Moment Bound]
\label{lemma:abs-moment-bound}
Let $t \geq 2, p \in \mathbb{N}$, $C_t > 0$, $\mat{x} \in \mathbb{R}^{d_1 \cdots d_p}$ and $\mat{w}_i \in \mathbb{C}^{d_i}$ for $i=1,\dots,p$.
If $\| \mat{w}_i^\top \mat{a} \|_{L^t} \leq C_t \| \mat{a} \|_2$ for all $\mat{a} \in \mathbb{R}^{d_i}$ and $\{ \mat{w}_{i} \}_{i=1}^p$, then
$$
\textstyle
\| (\otimes_{i=1}^p \mat{w}_i)^\top \mat{x} \|_{L^t} \leq C_t^p \| \mat{x} \|_2 \quad \text{holds.}
$$
In particular, for $t = 2 k$ with $k \in \mathbb{N}$, we obtain:
\begin{align*}
    &C_{t} = \sqrt{2} \pi^{- 1 / (2t)} \Gamma((t + 1)/2)^{1/t}
    & \text{(real Gauss./Rad.)} \\
    &C_{t} = \Gamma(t/2 + 1)^{1/t}
    & \text{(complex Gauss./Rad.)}
\end{align*}
which are tight constants. $\Gamma(\cdot)$ is the Gamma function.
\end{lemma}

\begin{proof}
Appendix \ref{sec:app-complex-moments}, where we also bound $C_t$ if $t \neq 2k$.
\end{proof}

The left plot of Fig.~\ref{fig:moment-bound} shows the constants $C_t$ for different values of $t$ and it becomes clear that higher order moments for the complex Gaussian/Rademacher distribution are smaller than for the real-valued one, with an increasing gain for larger $t$. This effect is again amplified with a larger $p$ that enters the moment bounds exponentially.

\begin{figure}[t]
\centerline{\includegraphics[width=1.0\linewidth]{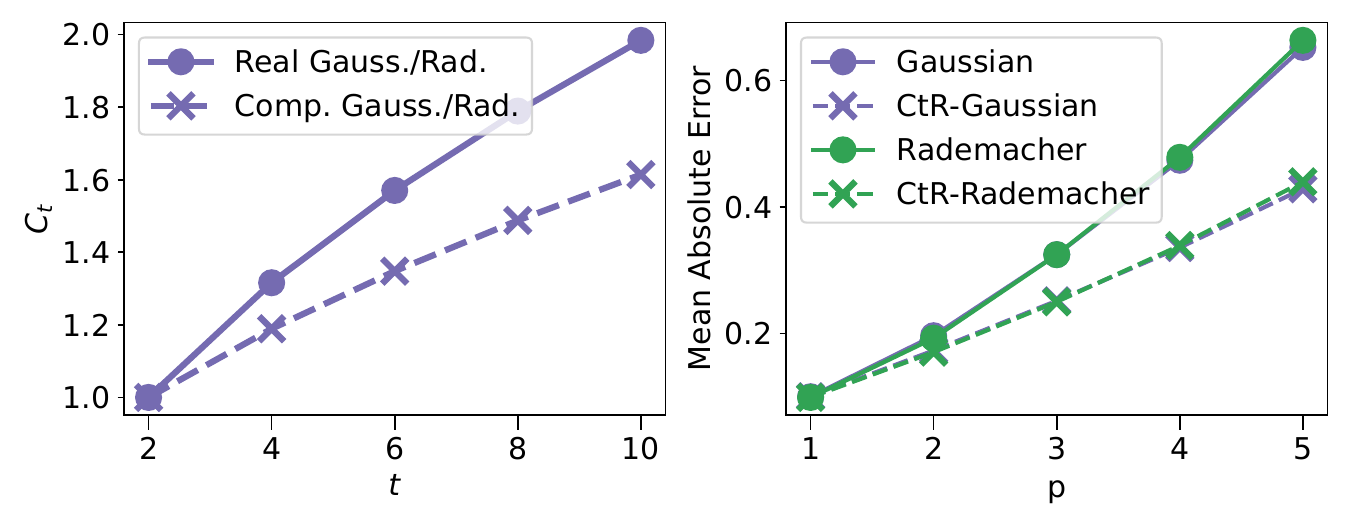}}
\caption{(Left) $C_t$ over $t=2k, k \in \mathbb{N}$.
(Right) Mean $| \norm{Sx}_2^2 - \norm{x}_2^2 |$ over $3 \cdot 10^4$ samples of $\mat{S}$ for real/CtR-sketches with an equal number of 128 rows and $\mat{x} = \mat{a}^{\otimes p}, \mat{a} = (1/\sqrt{d}, \dots, 1/\sqrt{d})^\top \in \mathbb{R}^d, d=64$.
}
\label{fig:moment-bound}
\end{figure}

The following theorem shows that complex sketches $\mat{Sx}$ thus require a lower sketching dimension $D$ than real ones to preserve the norm of $\mat{x}$, which is a direct consequence of the tighter moment bounds in Lem.~\ref{lemma:abs-moment-bound}.

\begin{theorem}[Norm Preservation]
\label{thrm:complex-norm-pres}

Let $0 < \epsilon, 0 < \delta < \exp(-2), \mat{x} \in \mathbb{R}^{d_1 \cdots d_p}, \mat{S} = (\mat{s}_1, \dots, \mat{s}_D)^\top \in \mathbb{C}^{D \times d_1 \cdots d_p}$ with $\mat{s}_\ell = \otimes_{i=1}^p \mat{w}_{i, \ell} / \sqrt{D}$ and $\mat{w}_{i,\ell} \in \mathbb{C}^{d_i}$ be i.i.d. Gaussian/Rademacher samples.
In order to guarantee
\begin{align*}
    &{\rm Pr} \left\{ | \norm{Sx}_2^2 - \norm{x}_2^2 | \leq \epsilon \norm{x}_2^2 \right\}
    \geq 1 - \delta, \quad \text{we need}
\end{align*}
$D = \bigO(\max\{C_4^{4p} \log(1/\delta) \epsilon^{-2}, (C_4^2 e / 2)^p \log^p(1/\delta) \epsilon^{-1}\})$,
where $C_4$ is defined in Lem.~\ref{lemma:abs-moment-bound} for the real/complex case.
\end{theorem}
\begin{proof}
Appendix \ref{sec:app-complex-norm-pres}, where we provide an additional bound for the case when $\delta \in (0,1)$.
\end{proof}
The upper bound on $D$ in Thm.~\ref{thrm:complex-norm-pres} is hence controlled by $C_4^{4p} = 3^p (2^p)$ and $C_4^{2p} = \sqrt{3}^p (\sqrt{2}^p)$ for real (complex) Gaussian/Rademacher sketches leading to a sharper dependence on $p$ for the complex case. In particular,
the $3^p$ dependence is tight for the real case as shown in the lower bound on $D$ in \citet[Appendix A.1]{Ahle2020}, which makes our bound a remarkable improvement. It thus takes us one step closer to reaching the optimal $D = \Theta(\log(1/\delta) \epsilon^{-2})$ for Johnson-Lindenstrauss embeddings \citep{Larsen2017} that is independent from $p$, but has a prohibitive $\bigO(D \prod_{i=1}^p d_i)$ computational cost. Lastly, our result improves over \citet[Thm. 3.4]{wacker2022} that bounds errors relative to the L1-norm instead of the L2-norm, which makes their bound much looser than ours as explained in Appendix \ref{sec:app-wacker-comparison}.

\subsection{Concentration Bounds for CtR-Sketches}

\label{sec:ctr-concentration}

It is easy to see that CtR-sketches directly inherit the guarantees in Thm.~\ref{thrm:complex-norm-pres}. Following the construction of CtR-features in Eq.~\ref{eqn:ctr-kernel}, we define the $2D$-dimensional CtR-sketch
\begin{align*}
    \mat{S}_{\rm CtR}
    := ({\rm Re} \{\mat{s}_1\}, \dots, {\rm Re} \{\mat{s}_{D}\}, {\rm Im} \{\mat{s}_1\}, \dots, {\rm Im} \{\mat{s}_{D}\})^\top
\end{align*}
giving $\| \mat{S}_{\rm CtR} \mat{x} \|_2^2
= \sum_{\ell=1}^D {\rm Re} \{ \mat{s}_\ell^\top \mat{x} \}^2
+ {\rm Im} \{ \mat{s}_\ell^\top \mat{x} \}^2
= \norm{Sx}_2^2$. We can thus substitute $\mat{S}$ in Thm.~\ref{thrm:complex-norm-pres} by $\mat{S}_{\rm CtR}$ to obtain the same guarantees. For a fair comparison, we need to multiply the required number of features $D$ in Thm.~\ref{thrm:complex-norm-pres} by two when using the \textit{same} number of rows for $\mat{S}_{\rm CtR}$ and $\mat{S}$. Crucially however, the improved dependence on $p$ remains the same implying that CtR-sketches must outperform real-valued analogs when $p$ is large enough. The right plot of Fig.~\ref{fig:moment-bound} shows that this is already the case from $p \geq 2$ with a larger gain for larger $p$. A more detailed variance comparison follows in Section~\ref{sec:ctr-variance}.

The following corollary of Thm.~\ref{thrm:complex-norm-pres} shows that inner products as well as matrix products are preserved under the same conditions provided in the theorem.
\begin{corollary}[Approximate Matrix Product]
\label{cor:approx-matrix}
Let $0 < \epsilon, 0 < \delta < \exp(-2), \mat{x}, \mat{y} \in \mathbb{R}^{d_1 \cdots d_p}$ and $\mat{S}_{\rm CtR}$ defined as in Section~\ref{sec:ctr-concentration}. In order to guarantee
\begin{align*}
    {\rm Pr} \left\{ | (\mat{S}_{\rm CtR} \mat{x})^\top (\mat{S}_{\rm CtR} \mat{y}) - \dotprod{x}{y} | \leq \epsilon \norm{x} \norm{y} \right\}
    \geq 1 - \delta,
\end{align*}
or for two matrices $\mat{X} \in \mathbb{R}^{d_1 \cdots d_p \times n}, \mat{Y} \in \mathbb{R}^{d_1 \cdots d_p \times m}$
\begin{align*}
    {\rm Pr} \left\{ \frac{\| (\mat{S}_{\rm CtR} \mat{X})^\top (\mat{S}_{\rm CtR} \mat{Y}) - \dotprod{X}{Y} \|_F}{\norm{X}_F \norm{Y}_F} \leq \epsilon \right\}
    \geq 1 - \delta,
\end{align*}
$\mat{S}_{\rm CtR}$ needs to have $2D$ rows with $D$ being the same as in Thm.~\ref{thrm:complex-norm-pres}.
\end{corollary}
\begin{proof}
Appendix~\ref{sec:app-approx-matrix}.
\end{proof}
In particular, Cor.~\ref{cor:approx-matrix} gives guarantees on the approximation error of polynomial kernels using CtR-sketches. In this case, we simply set $\mat{X} = \mat{A}^{\otimes p}, \mat{Y} = \mat{B}^{\otimes p}$, with $\mat{A}^{\otimes p}, \mat{B}^{\otimes p}$ being matrices whose columns are the polynomial kernel feature maps of some data points $\{\mat{a}_i\}_{i=1}^{n}$ and $\{\mat{b}_i\}_{i=1}^{m}$, respectively, where $\mat{a}_i, \mat{b}_i \in \mathbb{R}^d$. 

We can directly derive spectral kernel approximation guarantees from the approximate matrix product property as shown in Appendix~\ref{sec:app-ose}. Let $\mat{K} := (\mat{A}^{\otimes p})^\top \mat{A}^{\otimes p} \in \mathbb{R}^{n \times n}$ be the gram matrix for the points $\{\mat{a}_i\}_{i=1}^n$ and $\hat{\mat{K}} := (\mat{S}_{\rm CtR} \mat{A}^{\otimes p})^\top (\mat{S}_{\rm CtR} \mat{A}^{\otimes p})$ be its randomized approximation. Then with probability at least $1-\delta$, we have
\begin{align}
    \label{eqn:spectral-order}
    (1-\epsilon) (\mat{K} + \lambda \mat{I})
    \preceq
    \hat{\mat{K}} + \lambda \mat{I}
    \preceq (1+\epsilon) (\mat{K} + \lambda \mat{I})
\end{align}
for some $\lambda \geq 0$, if $\mat{S}_{\rm CtR}$ has $2 D s_{\lambda}(\mat{K})^2$ rows with $D$ being the same as in Thm.~\ref{thrm:complex-norm-pres} and $s_{\lambda}(\mat{K}) = {\rm Tr } \{\mat{K} (\mat{K} + \lambda \mat{I})^{-1}\} \leq n$ being the $\lambda$-statistical dimension of $\mat{K}$.

The spectral approximation guarantee directly implies statistical guarantees for downstream kernel-based
learning applications, such as bounds on the empirical risk of kernel ridge regression \citep[Lem. 2]{avron17a}. The quadratic dependence on $s_{\lambda}(\mat{K})$ is not optimal and arises due to the element-wise error bound in Cor.~\ref{cor:approx-matrix}. A linear dependence could be achieved by bounding the operator norm instead as it is done in \citet[Section 5]{Ahle2020}.
As the focus of this work is to obtain a sharp dependence w.r.t. $p$ and $\delta$, we leave this issue to future work and focus on a careful variance analysis of CtR-sketches instead.

%
\subsection{\hspace{-0.4em}Variances of CtR-Sketches for Polynomial Kernels}

\label{sec:ctr-variance}

%
%
In this section, we derive the closed form variances of CtR-sketches for the specific task of polynomial kernel approximation, and compare them against their real-valued analogs. This analysis is crucial, since we saw in Thm.~\ref{thrm:complex-norm-pres} that the improvement of CtR over real-valued sketches is because of a lower fourth moment $C_4$ as defined in Lem.~\ref{lemma:abs-moment-bound}. To be more precise, let $\mat{s}_\ell^\top \mat{x}$ be a single element of our complex sketch $\mat{Sx} \in \mathbb{C}^D$ as defined in Section~\ref{sec:complex-bounds}. Then we have $\mathbb{E} [| \dotprodi{s}{x}{\ell} |^4] \leq (C_4 \norm{x}_2 / \sqrt{D})^4$ as implied by Lem.~\ref{lemma:abs-moment-bound}. This is equal to the second moment $\mathbb{E} [| \dotprodi{s}{x}{\ell} \overline{\dotprodi{s}{y}{\ell}} |^2]$ for two inputs $\mat{x}, \mat{y}$ when $\mat{x}=\mat{y}$, which in turn is directly linked to the variance of $\dotprodi{s}{x}{\ell} \overline{\dotprodi{s}{y}{\ell}}$ via
$$
    \mathbb{E} \left[\left| \dotprodi{s}{x}{\ell} \overline{\dotprodi{s}{y}{\ell}} \right|^2\right]
    = \mathbb{V}\left[\dotprodi{s}{x}{\ell} \overline{\dotprodi{s}{y}{\ell}}\right] + \frac{1}{D^2} (\dotprod{x}{y})^2.
$$
The purpose of this section is therefore to carry out a careful variance analysis, elucidating conditions on $\mat{x}$ and $\mat{y}$ under which CtR-sketches perform better than real-valued analogs \textit{in practice}, and \textit{beyond the worst-case scenario}: $\mat{x}=\mat{y}$ as explained in Appendix~\ref{sec:app-approx-matrix}.

As we focus on polynomial kernels from now onward, we restrict our input space to vectors $\mat{x}^{\otimes p}, \mat{y}^{\otimes p} \in \mathbb{R}^{d^p}$ for some $\mat{x}, \mat{y} \in \mathbb{R}^d$. In this case, we can write:
\begin{align*}
    (\mat{S}_{\rm CtR} \mat{x}^{\otimes p})^\top (\mat{S}_{\rm CtR} \mat{y}^{\otimes p})
    = \hat{k}_{\rm CtR}(\mat{x}, \mat{y})
    \quad \text{(as in Eq.~\ref{eqn:ctr-kernel})}
\end{align*}
Hence, the approximate kernel and its variance $\mathbb{V} [\hat{k}_{\rm CtR}(\mat{x}, \mat{y})]$ depend directly on $\mat{x}, \mat{y} \in \mathbb{R}^d$. Let further $\hat{k}_{\rm C}(\mat{x}, \mat{y}) = (\mat{S} \mat{x}^{\otimes p})^\top (\overline{\mat{S} \mat{y}^{\otimes p}})$. In Appendix \ref{sec:ctr-structure}, we show that CtR-sketches have the following variance structure:
\begin{align}
    \label{eqn:ctr-var}
    &\mathbb{V} [\hat{k}_{\rm CtR}(\mat{x}, \mat{y})]
    = \frac{1}{2} \left( \mathbb{V} [\hat{k}_{\rm C}(\mat{x}, \mat{y})] + \mathbb{PV} [\hat{k}_{\rm C}(\mat{x}, \mat{y})] \right) \\
    \nonumber
    &\text{with} \quad \mathbb{PV} [\hat{k}_{\rm C}(\mat{x}, \mat{y})] := \mathbb{E} [\hat{k}_{\rm C}(\mat{x}, \mat{y})^2] - (\dotprod{x}{y})^{2p}
\end{align}
being the {\em pseudo-variance} of $\hat{k}_{\rm C}(\mat{x}, \mat{y})$ and $\mathbb{V} [\hat{k}_{\rm C}(\mat{x}, \mat{y})]$ its variance.

Our major contribution of this section is to derive $\mathbb{V} [\hat{k}_{\rm C}(\mat{x}, \mat{y})]$ and $\mathbb{PV} [\hat{k}_{\rm C}(\mat{x}, \mat{y})]$ for Gaussian/Rademacher sketches in Section \ref{sec:gaus-rad-variances} and we summarize these results in Table \ref{tbl:variances}. For a direct comparison, we also add the variances of real sketches $\Phi_{\rm R}$ (see Section \ref{sec:real-complex-sketches}) to Table \ref{tbl:variances}.
The question that we address in the following is:
Does the CtR estimator in Eq.~\ref{eqn:ctr-kernel} yield lower variances than $\hat{k}_{\rm R}(\mat{x}, \mat{y}) = \Phi_{\rm R}(\mat{x})^{\top} \Phi_{\rm R}(\mat{y})$ if $\Phi_{\rm CtR}$ and $\Phi_{\rm R}$ have the {\em same} output dimension $D$? We show next that this is indeed the case.

\begin{table*}[t]
\caption{Variances of complex $\hat{k}_{\rm C}(\mat{x}, \mat{y})$ and real $\hat{k}_{\rm R}(\mat{x}, \mat{y})$ and pseudo-variances of complex $\hat{k}_{\rm C}(\mat{x}, \mat{y})$ are shown.\\
$\mathbb{V}_{\rm Rad.}^{(p)}, \mathbb{V}_{\rm Rad.}^{(1)}$ and $\mathbb{PV}_{\rm Rad.}^{(p)}, \mathbb{PV}_{\rm Rad.}^{(1)}$ are the Rademacher variances / pseudo-variances for a given $p$ and $p=1$, respectively.}
\label{tbl:variances}
\begin{center}
\small
\resizebox{1 \linewidth}{!}{
\begin{tabular}{l | l | l}
\toprule
\textbf{Sketch} &
\textbf{Variance} $\mathbb{V} [\hat{k}_{\rm C}(\mat{x}, \mat{y})] \, (q=1)$ and $\mathbb{V} [\hat{k}_{\rm R}(\mat{x}, \mat{y})] \, (q=2)$ &
\textbf{Pseudo-Variance} $\mathbb{PV} [\hat{k}_{\rm C}(\mat{x}, \mat{y})]$ \\
\midrule
Gaussian &
$D^{-1} [ ( \norm{x}^2 \norm{y}^2 + q (\dotprod{x}{y})^2 )^p - (\dotprod{x}{y})^{2p} ]$ &
$D^{-1} [ (2 (\dotprod{x}{y})^2)^p - (\dotprod{x}{y})^{2p}]$ \\
Rademacher &
$D^{-1} [ ( \norm{x}^2 \norm{y}^2 + q ((\dotprod{x}{y})^2 - \sum_{i=1}^d x_i^2 y_i^2))^p - (\dotprod{x}{y})^{2p} ]$ &
$D^{-1} [ ( 2 (\dotprod{x}{y})^2 - \sum_{i=1}^d x_i^2 y_i^2 )^p - (\dotprod{x}{y})^{2p}]$ \\
ProductSRHT & $\mathbb{V}_{\rm Rad.}^{(p)} - (1-1/D) \cdot [(\dotprod{x}{y})^{2p} - ({\rm C}_{\rm Var.})^p]$
&
$\mathbb{PV}_{\rm Rad.}^{(p)} - (1 - 1/D) \cdot [(\dotprod{x}{y})^{2p} - ({\rm C}_{\rm PVar.})^p]$ \\
& ${\rm C}_{\rm Var.} = (\dotprod{x}{y})^2 - (\lceil D/d \rceil d - 1)^{-1} \ \mathbb{V}_{\rm Rad.}^{(1)} $
&
${\rm C}_{\rm PVar.}=(\dotprod{x}{y})^2 - (\lceil D/d \rceil d-1)^{-1} \ \mathbb{PV}_{\rm Rad.}^{(1)}$ \\
\bottomrule
\end{tabular}
}
\end{center}
\vspace{-1em}
\end{table*}

\subsection{Variance Reduction of CtR-Sketches}

\label{sec:var-reduction-properties}

We begin by studying the variance reduction properties of Gaussian/Rademacher CtR-sketches over their real-valued analogs. Let $\Phi_{\rm R}: \mathbb{R}^d \rightarrow \mathbb{R}^{D}$ be a real-valued sketch (see Section \ref{sec:real-complex-sketches}) and $\Phi_{\rm CtR}: \mathbb{R}^d \rightarrow \mathbb{R}^{D}$ a CtR-sketch as defined in Alg.~\ref{alg:ctr-algorithm}. Let $\hat{k}_{\rm R}(\mat{x}, \mat{y})$ and $\hat{k}_{\rm CtR}(\mat{x}, \mat{y})$ (\ref{eqn:ctr-kernel}) be the respective approximate kernels for some $\mat{x}, \mat{y} \in \mathbb{R}^d$. Then we can provide the following theorem for Rademacher sketches.
\begin{theorem}[CtR-Rademacher advantage]
    \label{thrm:rad-ctr-advantage}
    Let $a = \sum_{i \neq j'}^d x_i x_{j'} y_i y_{j'}$, $b_j = (\norm{x} \norm{y})^{2j} - (\sum_{i} x_i^2 y_i^2)^j \geq 0$. Then $\mathbb{V}[\hat{k}_{\rm R}(\mat{x}, \mat{y})] - \mathbb{V} [\hat{k}_{\rm CtR}(\mat{x}, \mat{y})]$ is equal to
    \begin{align*}
        \frac{1}{D} \sum_{k=2}^p \sum_{j=0}^{k-1} \binom{p}{k} \binom{k}{j} b_j \ a^{p-j} \geq 0 \quad \text{if} \quad a \geq 0.
    \end{align*}
    Furthermore, if $a \geq 0$, CtR-Rademacher sketches achieve the lowest possible variance for $\hat{k}_{\rm CtR}(\mat{x}, \mat{y})$ (\ref{eqn:ctr-kernel}) assuming the entries of $\{\mat{W}_i\}_{i=1}^p$ in Eq.~\ref{eqn:polynomial-estimator} are i.i.d. If $a < 0$, the lowest possible variance is attained by real Rademacher sketches instead, i.e., using $\hat{k}_{\rm R}(\mat{x}, \mat{y})$.
\end{theorem}
\begin{proof}
    The variance reduction and lowest variance property are proved in Appendix \ref{sec:proof-rad-ctr-advantage} and \ref{sec:gaus-rad-variances}, respectively.
\end{proof}
The theorem tells us that $\Phi_{\rm CtR}$ should be preferred over $\Phi_{\rm R}$ when $a \geq 0$ for two given inputs $\mat{x}, \mat{y} \in \mathbb{R}^d$ and the variance gap increases as $p$ increases. The condition
$$
\textstyle
a = \sum_{i=1}^d \sum_{j' \neq i}^d x_i x_{j'} y_i y_{j'} = (\dotprod{x}{y})^2 - \sum_{i=1}^d x_i^2 y_i^2 \geq 0
$$
always holds if $\mat{x},\mat{y}$ are non-negative or if they are parallel, thus leading to improved worst-case guarantees in Thm.~\ref{thrm:complex-norm-pres} and Cor.~\ref{cor:approx-matrix}. Non-negative data typically appears in applications for polynomial kernels such as categorical and image data, as well as outputs of convolutional neural networks. We carry out corresponding numerical experiments in Section~\ref{sec:experiments}. We also note that CtR-Rademacher sketches outperform real-valued analogs when the condition $a \geq 0$ is not always met as shown in Appendix~\ref{sec:further-var-comp}. This is because $a\geq 0$ always holds for the diagonal elements of the kernel matrix, leading to an inherent bias towards $a \geq 0$.

We can additionally provide the following theorem for Gaussian sketches proved in Appendix~\ref{sec:proof-gaus-ctr-advantage}.

\begin{theorem}[CtR-Gaussian advantage]
    \label{thrm:gauss-ctr-advantage}
    For any $\mat{x}, \mat{y} \in \mathbb{R}^d$, $\mathbb{V}[\hat{k}_{\rm R}(\mat{x}, \mat{y})] - \mathbb{V} [\hat{k}_{\rm CtR}(\mat{x}, \mat{y})]$ is equal to
    \begin{align*}
        \frac{1}{D} \sum_{k=0}^{p-1} \binom{p}{k} (2^k - 1) (\dotprod{x}{y})^{2k} \left( \norm{x}^2 \norm{y}^2 \right)^{p-k}
    \geq 0.
    \end{align*}
\end{theorem}

Thus, regardless of the input data, $\Phi_{\rm CtR}$ should be preferred over $\Phi_{\rm R}$ when using Gaussian sketches. The advantage again increases with $p$.



\section{ProductSRHT}

\label{sec:app-structured-sketches}

In this section, we propose a novel structured Rademacher sketch. Our sketch is called \textit{ProductSRHT} and is closely related to TensorSRHT \citep[Def. 15]{Ahle2020}. A major difference is that we are able to obtain the variance of ProductSRHT in closed form showing its statistical advantages over unstructured sketches. The variance derivation is contained in Appendix~\ref{sec:structured-variance-derivation}. We also embed ProductSRHT into our CtR-framework and compare its variance against CtR-Rademacher sketches in Section~\ref{sec:app-product-srht-var-comparison}.

Both ProductSRHT and TensorSRHT achieve a $\bigO(p(D + d \log d))$ runtime through structured Hadamard matrices that we introduce in the following. Let $n := 2^m$ with $m \in \mathbb{N}$, and $\mat{H}_n \in \{ 1, - 1\}^{n \times n}$ be the unnormalized Hadamard matrix, which is recursively defined as 
\begin{align*}
    \mat{H}_{2n} :=
    \begin{bmatrix}
        \mat{H}_n & \mat{H}_n \\
        \mat{H}_n & -\mat{H}_n
    \end{bmatrix},
    \quad
    \text{with}
    \quad
    \mat{H}_2 :=
    \begin{bmatrix}
        1 & 1\\
        1 & -1
    \end{bmatrix}.
\end{align*}
From now onward, we always use $\mat{H}_d \in \{ 1, -1 \}^{d \times d}$ with $d$ being the dimension of the input vectors, assuming $d = 2^m$ for some $m \in \mathbb{N}$. If $d \not= 2^m$ for any $m$, we pad the input vectors with $0$ until their dimension becomes $2^m$ for some $m$. We have $\mat{H}_d \mat{H}_d^{\top} = \mat{H}_d^{\top} \mat{H}_d = d \mat{I}_d$ and the recursive definition of $\mat{H}_d$ gives rise to the Fast Walsh-Hadamard transform \citep{Fino1976} that multiplies $\mat{H}_d$ with a vector $\mat{a} \in \mathbb{R}^d$ in $\bigO(d \log d)$ instead of $\bigO(d^2)$ time, while the matrix $\mat{H}_d$ does not need to be stored in memory. We drop the subscript $d$ from now for ease of presentation.


We describe (CtR-)ProductSRHT in Alg.~\ref{alg:product-srht-algorithm}. It uses structured matrices $\{\mat{W}_i\}_{i=1}^p$ in Eq.~\ref{eqn:polynomial-estimator}, which are formed through an element-wise multiplication of the rows of $\mat{H}$ with a Rademacher vector, imposing an orthogonality condition on these rows. This ultimately leads to a variance reduction that we analyze next. Finally, the rows of $\{\mat{W}_i\}_{i=1}^p$ are randomly up/downsampled to cover the case $D \neq d$.

\begin{algorithm*}[ht]
\caption{(CtR-) ProductSRHT}
\label{alg:product-srht-algorithm}
\SetAlgoLined
\textbf{Input:} Data point $\mat{x} \in \mathbb{R}^d$, projection dimension $D \in \mathbb{N}$

Pad $\mat{x}$ with zeros so that $d$ becomes a power of $2$,
let $B = \left\lceil \frac{D}{d} \right\rceil$ be the number of stacked projection blocks

\ForAll{$i \in \{1, \dots, p\}$}{
    Generate a diagonal matrix $\mat{D}_i \in \mathbb{C}^{d \times d}$ with diagonal elements:
    
    $\quad (\mat{D}_i)_{1,1}, \dots, (\mat{D}_i)_{d,d} \stackrel{i.i.d.}{\sim} {\rm Unif}(\{1, -1\})$ \hfill (real case)
    
    $\quad (\mat{D}_i)_{1,1}, \dots, (\mat{D}_i)_{d,d} \stackrel{i.i.d.}{\sim} {\rm Unif}(\{1, -1, \iu, -\iu\})$ \hfill (complex case)

    Generate a random sampling matrix $\mat{P}_i \in \{1, 0\}^{D \times d}$ as follows:
    
    $\quad$Let $\mat{p}_i = (p_{i,1}, \dots, p_{i, Bd})^\top \in \mathbb{R}^{Bd}$ be the $B$-times concatenation of $(1, \dots, d)^\top$
    
    $\quad$Randomly permute the indices $1, \dots, Bd$ to $\pi(1), \dots, \pi(Bd)$
    
    $\quad$Set $\mat{P}_i = (\mat{e}_{p_{i,\pi(1)}}, \dots, \mat{e}_{p_{i,\pi(D)}})^{\top}$, where $\mat{e}_{p_{i,\pi(\ell)}} \in \{1, 0\}^d$ is equal to 1 at position $p_{i,\pi(\ell)}$ and 0 elsewhere
    
    Set $\mat{W}_i = \mat{P}_i \mat{H} \mat{D}_i$
}
\textbf{Return:}
$\Phi_{\rm R}(\mat{x})$ using Eq.~\ref{eqn:polynomial-estimator} for ProductSRHT
$\quad \text{or} \quad \Phi_{\rm CtR}(\mat{x})$ using Alg.~\ref{alg:ctr-algorithm} for CtR-ProductSRHT

\end{algorithm*}

\subsection{Variance of CtR-ProductSRHT}
\label{sec:app-product-srht-var-comparison}

A major contribution of this work is to derive the variance of our proposed CtR-ProductSRHT sketch in closed form, which requires the derivation of the variance and pseudo-variance of complex ProductSRHT as shown in Eq.~\ref{eqn:ctr-var}. They are derived in Section \ref{sec:structured-variance-derivation} and we summarize them in Table~\ref{tbl:variances}. We also derive the variance of real ProductSRHT in Section \ref{sec:tensor-srht-var-derivation} and add it to Table~\ref{tbl:variances} for comparison.

ProductSRHT can yield lower variances than Rademacher sketches as it removes the i.i.d. constraint between the $\{\mat{w}_{i, \ell}\}_{\ell=1}^D$ in Eq.~\ref{eqn:sketch-element}. In fact, these vectors are mutually orthogonal for two $\ell \neq \ell'$ when the $p_{i, \pi(\ell)}$-th and the $p_{i, \pi(\ell')}$-th column of $\mat{H}$ are distinct, since $\mat{H}$ has orthogonal rows and columns. This dependence introduces the term
$$
\textstyle
    {\rm R}_{\rm Var. / PVar.} := (1-1/D) [(\dotprod{x}{y})^{2p} - (\textrm{C}_{\rm Var. / PVar.})^p]
$$
that is subtracted from the original Rademacher variance $\mathbb{V}^{(p)}_{\rm Rad.}$ and pseudo-variance $\mathbb{PV}^{(p)}_{\rm Rad}$, respectively, as shown in Table~\ref{tbl:variances}, where we also define $\textrm{C}_{\rm Var.}$ and $\textrm{C}_{\rm PVar.}$.

If $p$ is odd, $(\textrm{C}_{\rm Var.})^p \leq (\dotprod{x}{y})^{2p}$ holds because $\mathbb{V}_{\rm Rad.}^{(1)} \geq 0$ and therefore ${\rm R}_{\rm Var.} \geq 0$ holds. In this case, the variance of complex/real ProductSRHT is upper-bounded by the complex/real Rademacher variance $\mathbb{V}^{(p)}_{\rm Rad.} \geq 0$.

If we further have $\mathbb{PV}_{\rm Rad.}^{(1)} = \sum_{i=1}^d \sum_{j \neq i}^d x_i x_j y_i y_j \geq 0$, the pseudo-variance of complex ProductSRHT is also upper-bounded by the Rademacher pseudo-variance. This is because $0 \leq ({\rm C}_{\rm PVar.})^p \leq (\dotprod{x}{y})^{2p}$ and ${\rm R}_{\rm PVar.} \geq 0$ hold. Note that this is exactly the same condition as $a \geq 0$ in Thm.~\ref{thrm:rad-ctr-advantage}. CtR-ProductSRHT thus has a lower pseudo-variance than CtR-Rademacher sketches exactly when CtR-Rademacher sketches are better than real ones.

As both the variance and the pseudo-variance of complex ProductSRHT are upper-bounded by the ones of complex Rademacher sketches under the above conditions, CtR-ProductSRHT is guaranteed to have a lower variance than CtR-Rademacher sketches through Eq.~\ref{eqn:ctr-var} in this case. Moreover, CtR-ProductSRHT inherits the variance reduction of CtR-Rademacher sketches over their real analogs because the Rademacher variance and pseudo-variance both enter the ones of complex ProductSRHT (see Table~\ref{tbl:variances}).

\section{RELATED WORK}

\label{sec:related}

In this work, we study the sketches for tensor products presented in Section~\ref{sec:real-complex-sketches} building on previous works by \citet{Kar2012, Hamid2014, Meister2019, Ahle2020, wacker2022}. However, there exist alternatives that we have not mentioned so far.

\citet{Pham2013} have proposed \textbf{TensorSketch}, which is a convolution of CountSketches \citep{Charikar2002}.
TensorSketch  requires $D=\bigO(3^p s_{\lambda}(\mat{K})^2 / (\delta \epsilon^{2}))$  to satisfy Eq.~\ref{eqn:spectral-order}  \citep{Avron2014} and thus has weaker guarantees w.r.t. $\delta$ and $p$ than CtR-Gaussian/Rademacher sketches.
There is also no closed form variance formula available for this sketch\footnote{\citet{Pham2013} contains a variance formula, but makes the simplifying assumption that TensorSketch has the same variance as CountSketch applied to tensorized inputs. \citet{Avron2014} conduct a more careful analysis to obtain an upper bound.}.
Yet, it achieves state-of-the-art performance in practice as we show in Section~\ref{sec:experiments}. It is also faster than Gaussian/Rademacher sketches taking only $\bigO(p(D \log D + d))$ instead of $\bigO(pdD)$ via the {\em Fast Fourier Transform}.

Structured Rademacher sketches based on the {\em Subsampled Randomized Hadamard Transform (SRHT)} \citep{Tropp2011} have been proposed by \citet{Hamid2014}, and a similar sketch called \textbf{TensorSRHT} by \citet{Ahle2020}, referring to the fact that SRHT is implicitly applied to a tensorized version of the input. Both sketches use the Fast Walsh-Hadamard Transform \citep{Fino1976} for faster projections. Our (CtR-) ProductSRHT sketch is closely related. Notably, both TensorSRHT and our sketch have a runtime of $\bigO(p(d \log d + D))$ and are thus faster than TensorSketch when $D>d$. Unlike previous works, we derive the variance for our ProductSRHT sketch in closed form, showing statistical advantages over Rademacher sketches.


Recent research has focused on \textit{meta-algorithms} that aim to improve the approximation error of existing sketches \citep{Hamid2014, Ahle2020, Song21c}. 
In particular, \citet{Ahle2020} managed to reduce the exponential dependence of $D$ on $p$ to polynomial by using a hierarchical construction.
The sketches proposed in this work are compatible with these methods and can serve as their base sketches. In fact, we combine the \textbf{hierarchical} construction by \citet{Ahle2020} and \textbf{CRAFT maps} by \citep{Hamid2014} with CtR-sketches in Section~\ref{sec:experiments}.

A fundamentally different approach are \textbf{Spherical Random Features (SRF)} \citep{Pennington2015} that require a preprocessing step and yield biased polynomial kernel approximations for data on the unit-sphere. SRF can only be applied to inhomogeneous polynomial kernels and work well for large $p$. We adapt our experiments in Section~\ref{sec:experiments} accordingly to accommodate a comparison against SRF.

\section{EXPERIMENTS}

\label{sec:experiments}

In this section, we carry out a systematic comparison of the CtR-sketches presented in this work against their real-valued analogs as well as TensorSketch and SRF. We also combine CtR-ProductSRHT and TensorSketch with \citet[Alg. 1]{Ahle2020} denoted as Hierarchical TensorSketch/CtR-ProductSRHT. Moreover, we add CRAFT maps \citep{Hamid2014} denoted as CRAFT TensorSketch/CtR-ProductSRHT to this comparison.

\subsection{Experimental Setup}

\label{sec:experimental-setup}

\paragraph{Data sets} We use MNIST \citep{lecun1998}, and convolutional features\footnote{For CIFAR-10 (CUB-200), we use convolutional outputs of a ResNet34 \citep{He2016} (VGG-M \citep{ChatfieldSVZ14}) pretrained on ImageNet \citep{Russakovsky2015}.} for CIFAR-10 \citep{krizhevsky2009learning} and CUB-200 \citep{welinder2010} as our data sets for the evaluation in this section. All three data sets contain only non-negative inputs to ensure that the condition of Thm.~\ref{thrm:rad-ctr-advantage} is met. Additional experiments with zero-centered data and more data sets are contained in Appendix~\ref{sec:app-further-experiments}.

\paragraph{Target kernel and its approximation}

Except for Section~\ref{sec:online-learning-experiment}, we follow \citet{Pennington2015} and restrict our experiments to the polynomial kernel
$$
    k(\mat{x}, \mat{y}) = \left(\left(1-2/a^2\right) + 2/a^2 \, \dotprod{x}{y}\right)^p, \quad \text{with} \quad a \geq 2,
$$
with for $\mat{x}, \mat{y} \in \mathbb{R}^d$ having unit-norm, as this allows a comparison against SRF. In particular, we set $a=2$ to assign the largest weight possible to high polynomial degrees in the binomial expansion of the kernel, thereby making its approximation more challenging. Further results for non unit-norm data are contained in Section \ref{sec:online-learning-experiment}.
We denote by $D$ the feature map dimension that we ensure to be equal between CtR- and non-CtR sketches. For CRAFT maps, the intermediate up-projection dimension is fixed to $E=2^{15}$.
We measure the kernel approximation quality through the relative Frobenius norm error, which is defined as $\|\hat{\mat{K}} - \mat{K}\|_F / \|\mat{K}\|_F$, where $\hat{\mat{K}}$ is the random feature approximation of the exact kernel matrix $\mat{K}$ evaluated on a subset of the test data of size $1000$ that is resampled for each seed used in these experiments.

All time benchmarks are run on an NVIDIA P100 GPU and PyTorch 1.10 \citep{paszke2019pytorch} with CUDA 10.2.

\subsection{Kernel Approximation and GP Classification}

\label{sec:gp-classification}

We start by comparing the sketches discussed in this work for the downstream task of Gaussian Process (GP) classification. We model GP classification as a multi-class GP regression problem with transformed labels \citep{Milios2018}, for which we obtain closed-form solutions to measure the effects of the random feature approximations in isolation without the need for convergence verification.

Fig.~\ref{fig:final-comparison} shows the result of this comparison. CtR-sketches generally result in lower kernel approximation errors than their real-valued analogs, with an increased effect for a larger degree $p=7$. Overall, we see that ProductSRHT performs better than Gaussian and Rademacher sketches, and comparable to TensorSketch. The hierarchical extension of CtR-ProductSRHT/TensorSketch only improves results for $p=7$, and performs worse for $p=3$. CRAFT maps on the other hand always improve results.

Although similar trends can be observed for test errors, differences between the methods become only strongly noticeable for large $p=7$. This makes sense, since all sketches become less optimal for larger $p$, amplifying their difference in statistical efficiency. For $D=2^9$ and $p=7$, CtR-sketches yield around 5\% / 2.5\% / 1\% improvement for Gaussian / Rademacher / ProductSRHT, respectively. The absolute error difference among all methods decreases for larger $D$, but their relative improvement remains.

Kernel approximations for SRF are generally biased with a decreasing bias for larger $p$ \citep[Section 4]{Pennington2015}. We thus see that the relative Frobenius norm error for SRF stagnates for $p=3$ when $D$ is large, while the one for the other sketches continues decreasing. Test errors for SRF are also worse in this case. For $p=7$, SRF kernel approximation errors tend to be lower than for CtR-ProductSRHT/TensorSketch and are comparable to their CRAFT extensions (slightly worse for MNIST, slightly better for CIFAR-10). SRF test errors are slightly worse than for the CRAFT sketches. In summary, CRAFT CtR-ProductSRHT and CRAFT TensorSketch yield the lowest test errors and kernel approximation errors (except for CIFAR-10 and $p=7$, where SRF has lower errors).


\begin{figure*}[ht]
\begin{center}
\centerline{\includegraphics[width=0.98\linewidth]{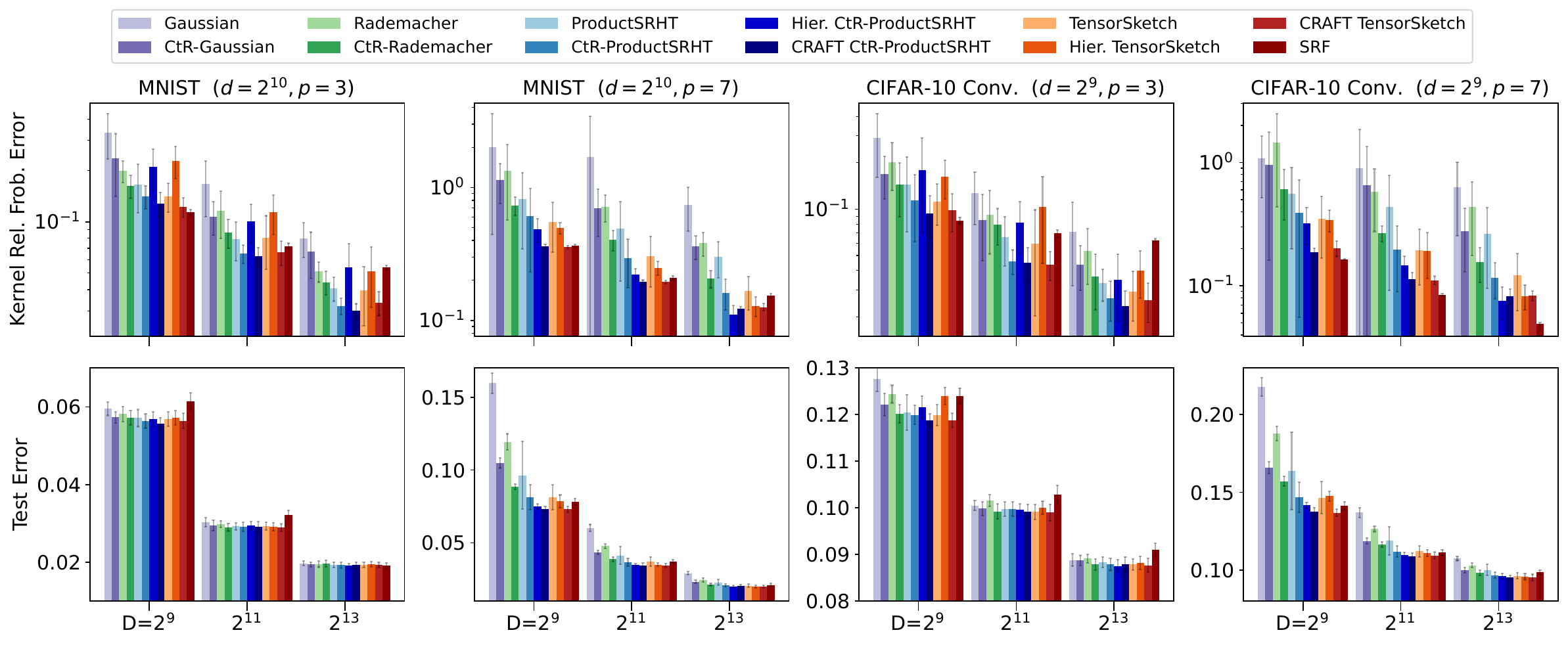}}
\caption{
MNIST and CIFAR-10 comparison for $p=3$ and $p=7$ averaged over 20 seeds. Due to space limitations, we only show results for $D \in \{ 2^9, 2^{11}, 2^{13} \}$. Results for $D \in \{ 2^i \}_{i=8}^{13}$ and more data sets are contained in Appendix~\ref{sec:app-further-experiments}.
}
\label{fig:final-comparison}
\end{center}
\vspace{-2em}
\end{figure*}

\subsection{Feature Construction Time Comparison}

In the following, we carry out a feature construction time comparison of the methods presented in this work against TensorSketch that has a time complexity of $\bigO(p(D \log D + D))$ and SRF.
Recall that our proposed ProductSRHT approach in Section \ref{sec:app-structured-sketches} has a time complexity of $\bigO(p(d \log d + D))$ and is thus faster in theory when $D > d$. The left plot in Fig.~\ref{fig:feature-construction-time} shows that this is also the case in practice.

The construction times of real ProductSRHT and CtR-ProductSRHT have a smaller slope with respect to $D$ than the other sketches leading to the lowest feature construction times together with SRF, in particular when $D \gg d$. There is a small computational overhead for CtR-ProductSRHT compared to ProductSRHT because CtR-ProductSRHT initially requires two Hadamard-projections (real and imaginary parts), but uses the same upsampling matrix leading to the same scaling property with respect to $D$. The right plot of Fig.~\ref{fig:feature-construction-time} shows that SRF kernel approximations are strongly biased, making CtR-ProductSRHT the most accurate sketch for $p=3$.

Faster feature construction times matter in practice. Although CRAFT maps enjoy a strong performance in Section~\ref{sec:gp-classification}, they can be expensive to compute due to the up-projection to $E=2^{15}$ before down-projecting to $D$. Table~\ref{tbl:feature-downstream-ratio} shows the ratio of feature construction time against solving the downstream GP model for MNIST. The ratio decays with larger $D$ since solving the downstream model scales as $\bigO(n D^2 + D^3)$. For small $D$, the feature construction may dominate however. We also see that (CtR-) ProductSRHT is generally faster than TensorSketch. Moreover, feature construction times can heavily influence online learning scenarios in which the optimization algorithm requires a forward and backward pass through the feature map for every iteration.

\begin{figure}[ht]
\begin{center}
\centerline{\includegraphics[width=0.95\linewidth]{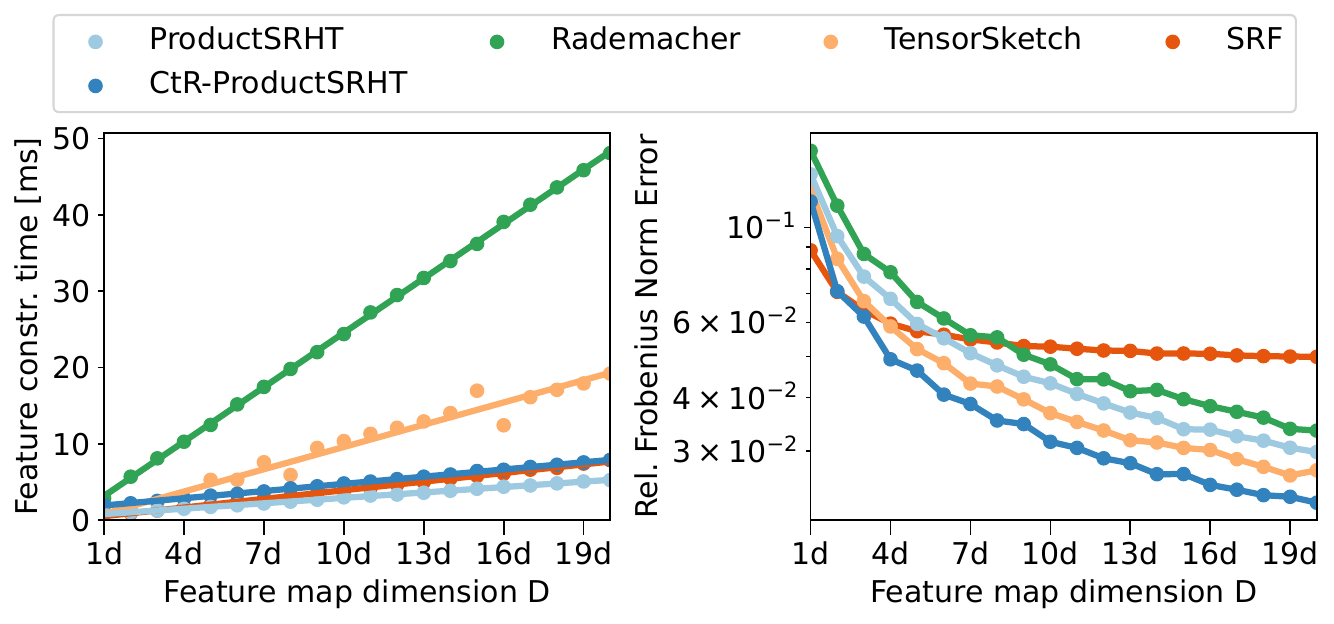}}
\caption{
(Left) Feature construction time, (right) kernel approximation error, against feature map dimension $D$ for $p=3$ on 1000 random MNIST samples.
}
\label{fig:feature-construction-time}
\end{center}
\vspace{-2em}
\end{figure}

\begin{table}[t]
\caption{Projection time / downstream time ratio ($p=3$).}
\label{tbl:feature-downstream-ratio}
\begin{center}
\small
\resizebox{1 \linewidth}{!}{
\begin{tabular}{c | c | c c c | c c}
\toprule
\multirow{2}{*}{\textbf{D}} &
\multirow{2}{*}{\textbf{SRF}} &
\textbf{Prod.} &
\multirow{2}{*}{\textbf{+ CtR}} &
\textbf{+ CtR} &
\textbf{Tensor-} &
\multirow{2}{*}{\textbf{+ CRAFT}}
\\
& & \textbf{SRHT} & & \textbf{+ CRAFT} & \textbf{Sketch} &
\\
\midrule

$2^9$ & 3.51 & 2.96 & 3.32 & 6.09 & 5.13 & 9.98 \\

$2^{11}$ & 0.40 & 0.33 & 0.38 & 0.67 & 0.58 & 1.08 \\

$2^{13}$ & 0.04 & 0.03 & 0.04 & 0.06 & 0.06 & 0.10 \\

\bottomrule

\end{tabular}
}
\end{center}
\vspace{-2em}
\end{table}

\subsection{Online Learning for Fine-Grained Recognition}

\label{sec:online-learning-experiment}

We follow \citet{Gao2016} and carry out an online learning experiment using convolutional features from the CUB-200 data set (see Section~\ref{sec:experimental-setup}). The task of fine-grained visual recognition is about the classification of pictures \textit{within} their subordinate categories (200 bird species in this case). Here feature maps of low-degree polynomial kernels have proven very effective, but lead to classification layers with too many parameters due to high-dimensional inputs.

\citet{Gao2016} therefore use TensorSketch to reduce the dimension of explicit polynomial feature maps. We compare our methods against theirs and against SRF in Appendix~\ref{sec:app-fine-grained}. (CtR-) ProductSRHT achieves the same test errors as TensorSketch, while being faster, especially when using CRAFT maps (almost 2x speedup). SRF is fast, but achieves only 75\% test error compared to 30\% achieved by the other methods. This is because SRF requires the unit-normalization of the convolutional features, hence loses important information. Polynomial kernels have thus important applications \textit{beyond unit-normalized} data, which is neglected in \citet{Pennington2015}.

\subsection{Error Bound Comparison}

Lastly, we compare the empirical error probability of (CtR-) Rademacher/ProductSRHT against TensorSketch\footnote{We did not add SRF to this comparison because $\|\Phi(\mat{x})\|_2^2 = \Phi(\mat{x})^\top \Phi(\mat{x}) = \frac{1}{D} \sum_{\ell=1}^D \cos(\mat{\omega}_{\ell}^\top (\mat{x} - \mat{x})) = 1$ has zero variance.} for two fixed vectors $\mat{x} \in \mathbb{R}^d$ with a different maximum-to-norm-ratio $r := \norm{x}_{\infty} / \norm{x}_2$. TensorSketch can be seen as a CountSketch \citep{Charikar2002} in a tensorized vector space \citep{Pham2013}. \citet[Thm. 3]{Weinberger2009} show that the error probability of CountSketches is heavily influenced by $r$ due to hashing collisions.

This is also the case for TensorSketch as shown in Fig.~\ref{fig:empirical-error-prob}, i.e., it converges much slower for $r=0.58$ with an empirical error probability that is two orders of magnitude larger than for our methods for large $D$. For an extended discussion, see \citet[Section 4.1]{Meister2019}.

\begin{figure}[t]
\centerline{\includegraphics[width=0.95\linewidth]{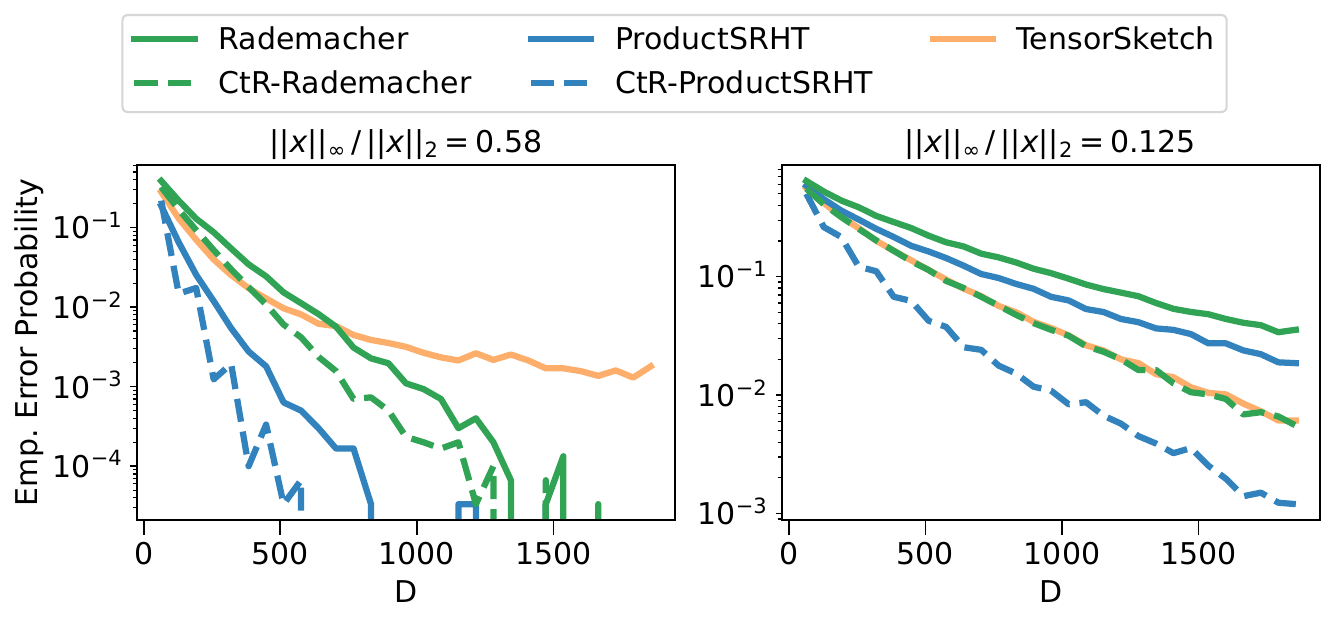}}
\caption{Emp. ${\rm Pr} \{ | \| \mat{S} \mat{x}^{\otimes p} \|_2^2 - \| \mat{x}^{\otimes p} \|_2^2 | \geq \epsilon \| \mat{x}^{\otimes p} \|_2^2 \}$ for $\epsilon = 0.25$, $d=64$, $p=2$. (Left) $\mat{x} = (\sqrt{d},\sqrt{d},1, \dots, 1)^\top$, $r=0.58$; (Right) $\mat{x} = (1, \dots, 1)^\top$, $r=0.125$.}
\label{fig:empirical-error-prob}
\vspace{-0.5em}
\end{figure}





\section{CONCLUSION}

The goal of research on random projections for tensor products is to achieve the optimal Johnson-Lindenstrauss embedding dimension of $D = \Theta(\epsilon^{-2} \log(1/\delta))$, i.e., without dependence on $p$, and without high computational costs. A recent work by \citet{Ahle2020} has improved the unwanted exponential dependence of $D=\bigO(3^p)$ to polynomial by using a hierarchical construction of well-known base sketches.
However, we showed empirically in Section~\ref{sec:experiments} that their method only yields improvements for large $p$ and yields worse performance for small $p$.

In this work, we took a different approach by modifying the base sketch sampling distribution directly, thus achieving $D=\bigO(2^p)$ instead of $D=\bigO(3^p)$. Although still not being optimal, our method already leads to improvements from $p \geq 2$ and can be combined with other meta-algorithms. Moreover, we achieved state-of-the-art results in terms of accuracy \textit{and} speed in our experiments. We thus uncovered an exciting angle of improvement that can be further leveraged in future research.

\subsubsection*{Acknowledgements}
We thank Motonobu Kanagawa for helpful discussions. MF gratefully acknowledges support from the AXA Research Fund and the Agence Nationale de la Recherche (grant ANR-18-CE46-0002 and ANR-19-P3IA-0002). RO started this work while interning at the Criteo AI lab in Paris.


\bibliography{bibliography}
\bibliographystyle{icml_template/icml2021}

\appendix

\onecolumn






\section*{STRUCTURE OF THE APPENDIX}

\begin{itemize}
    \item Appendix~\ref{sec:app-concentration} contains proofs for the concentration results of Sections and \ref{sec:complex-bounds} and \ref{sec:ctr-concentration} of the main paper.
    \item Appendix~\ref{sec:app-ctr-variances} contains variance derivations for non-structured CtR-sketches and proofs for the theorems in Section~\ref{sec:ctr-variance}.
    \item Appendix~\ref{sec:app-structured-variances} extends these variance derivations to (CtR-)ProductSRHT.
    \item Appendix~\ref{sec:app-further-experiments} contains additional numerical experiments to complement Section~\ref{sec:experiments} of the main paper.
\end{itemize}

\section{CONCENTRATION RESULTS}

\label{sec:app-concentration}

This section contains the proofs of Sections and \ref{sec:complex-bounds} and \ref{sec:ctr-concentration} of the main paper. Many results in this section build on top of the work by \citet{Ahle2020}. More precisely, we extend \citet[Lem. 9, 11, 19, Thm. 42]{Ahle2020} to the case of CtR-sketches and show the improvements that our methods bring about. In particular, we derive absolute moments for CtR-sketches and show that these lead to sharper results.

\subsection{Derivation of Moment Bounds and Proof of Lemma~\ref{lemma:abs-moment-bound}}

\label{sec:app-complex-moments}

We restate Lem.~\ref{lemma:abs-moment-bound} here for ease of presentation. It is a complex extension of \citet[Lem. 19]{Ahle2020}.

\begin{lemma}[Absolute Moment Bound]
\label{lemma:app-absolute-moment-bound}
Let $t \geq 2, p \in \mathbb{N}$, $C_t > 0$, $\mat{x} \in \mathbb{R}^{d_1 \cdots d_p}$ and $\mat{w}_i \in \mathbb{C}^{d_i}$ for $i=1,\dots,p$.
If $\| \mat{w}_i^\top \mat{a} \|_{L^t} \leq C_t \| \mat{a} \|_2$ for all $\mat{a} \in \mathbb{R}^{d_i}$ and $\{ \mat{w}_{i} \}_{i=1}^p$, then the following holds:
$\| (\otimes_{i=1}^p \mat{w}_i)^\top \mat{x} \|_{L^t} \leq C_t^p \| \mat{x} \|_2$.
\end{lemma}

Before proving Lem.~\ref{lemma:app-absolute-moment-bound}, we start by deriving the moment bounds $C_t$ for (complex) Gaussian and Rademacher sketches.

\subsubsection{Moment Bounds for Gaussian and Rademacher Sketches}

W.l.o.g., we assume $\norm{a}_2=1$, since both sides of $\| \mat{w}_i^\top \mat{a} \|_{L^t} = C_t \norm{a}_2$ can be divided by $\norm{a}_2$.

\paragraph{Gaussian distribution}

For the simpler Gaussian case, we obtain $C_t$ that is not only a tight upper bound, but an exact value for the $t$-th moment. That is, we have $\| \mat{w}_i^\top \mat{a} \|_{L^t} = C_t \norm{a}_2$. Moreover, we obtain values for $t > -1 \in \mathbb{R}$, and not only for even integers $t$.

We start with the real case, i.e., $\mat{w}_i \sim \mathcal{N}(\mat{0}, \mat{I}_{d_i})$. Then $\mat{w}_i^\top \mat{a} \sim \mathcal{N}(0, 1)$, since Gaussians are closed under linear transformations. The $t$-th absolute moment of a Gaussian random variable is well-known. For $t > -1$, it is
\begin{align*}
    \text{(Real case)} \quad \quad
    \mathbb{E} [|(\mat{w}_i)^\top \mat{a} |^t]
    = 2^{t/2} \Gamma \left(\frac{t+1}{2} \right)/\sqrt{\pi}
    \quad \iff \quad
    \| (\mat{w}_i)^\top \mat{a} \|_{L^t}
    = C_t = \sqrt{2} \pi^{-1/(2t)} \Gamma \left(\frac{t+1}{2} \right)^{1/t},
\end{align*}
where $\Gamma(\cdot)$ is the Gamma function.

For the complex case, we have $\mat{w}_i \sim \mathcal{CN}(\mat{0}, \mat{I}_{d_i})$, which is equivalent to $\mat{w}_i = 1/\sqrt{2} (\mat{u}_i + \iu \, \mat{v}_i)$ with $\mat{u}_i, \mat{v}_i \sim \mathcal{N}(\mat{0}, \mat{I}_{d_i})$ being independent. Then we have
\begin{align}
    \label{eqn:comp-gauss-expansion}
    \mathbb{E} [|(\mat{w}_i)^\top \mat{a} |^t]
    = \mathbb{E} [| \sqrt{1/2} (\dotprodi{u}{a}{i} + \iu \, \dotprodi{v}{a}{i}) |^t]
    = (1/2)^{t/2} \, \mathbb{E} [(| \dotprodi{u}{a}{i} |^2 + |\dotprodi{v}{a}{i}|^2)^{t/2}].
\end{align}
Now we observe that $\dotprodi{u}{a}{i}, \dotprodi{v}{a}{i} \sim \mathcal{N}(0, 1)$. So $| \dotprodi{u}{a}{i} |^2 + |\dotprodi{v}{a}{i}|^2$ is chi-square distributed with two degrees of freedom. The $t'$-th moment of a chi-square distributed variable $X$ with $k \in \mathbb{N}$ degrees of freedom is \citep[Thm. 3.3.2.]{hogg2005}:
\begin{align*}
    \mathbb{E}[X^{t'}]
    = 2^{t'} \Gamma(t' + k/2) / \Gamma(k/2),
    \quad \text{if} \quad t' > -k/2,
\end{align*}
By setting $k=2, t'=t/2$ and noting $\Gamma(1)=1$, we obtain
\begin{align*}
    \text{(Complex case)} \quad \quad
    \| (\mat{w}_i)^\top \mat{a} \|_{L^t}
    = C_t = \Gamma(t/2 + 1)^{1/t},
\end{align*}
where $t > -2$ covers $t \geq 2$ required by the lemma.

\paragraph{Rademacher distribution}

For the real case, we can directly apply Khintchine's inequality stating $\| (\mat{w}_i)^\top \mat{a} \|_{L^t} \leq C_t \norm{a}_2$ with $0 < t < \infty$. \citet{Haagerup1981} derived tight values for $C_t$ yielding
\begin{align*}
    \text{(Real case)} \quad
    \| (\mat{w}_i)^\top \mat{a} \|_{L^t}
    \leq C_t
    = \left\{
        \begin{array}{ll}
        1 & \text{for} \quad 0 < t \leq 2 \\
        \sqrt{2} \pi^{-1/(2t)} \Gamma \left(\frac{t+1}{2} \right)^{1/t} & \text{for} \quad t > 2.
        \end{array}
    \right.
\end{align*}
For the complex Rademacher case, we note that $|\exp(\iu \pi/4) \mat{w}_i^\top \mat{a}| = |\mat{w}_i^\top \mat{a}|$ since $|\exp(\iu \pi/4)| = 1$. The elements of $\mat{w}'_i := \exp(\iu \pi/4) \mat{w}_i$ are then sampled from
$$
{\rm Unif} (\{\exp(\iu \pi/4), -\exp(\iu \pi/4), \iu \exp(\iu \pi/4), -\iu \exp(\iu \pi/4)\})
= \frac{1}{\sqrt{2}}
{\rm Unif}(\{1 + \iu, -1 - \iu, -1 + \iu, 1 - \iu\}).
$$
We can thus rewrite $\mat{w}'_i = 1/\sqrt{2} (\mat{u}_i + \iu \mat{v}_i)$ with $\mat{u}_i, \mat{v}_i$ having elements sampled i.i.d. from ${\rm Unif}(\{1, -1\})$. For $t=2k, k \in \mathbb{N}$, we can further expand
\begin{align}
    \label{eqn:complex-rad-expansion}
    \mathbb{E} [|\mat{w}_i^\top \mat{a}|^t]
    = (1/2)^{t/2} \mathbb{E} \left[(| \dotprodi{u}{a}{i} |^2 + |\dotprodi{v}{a}{i}|^2)^{t/2} \right]
    = (1/2)^{t/2} \sum_{n=0}^{t/2} \mathbb{E} \left[ | \dotprodi{u}{a}{i} |^{2n} \right] \mathbb{E} \left[ |\dotprodi{v}{a}{i}|^{2(t/2-n)} \right]
\end{align}
using the binomial theorem. Eq.~\ref{eqn:complex-rad-expansion} must be upper bounded by Eq.~\ref{eqn:comp-gauss-expansion}, since both have the same structure and, by Khintchine's inequality, the moments on the r.h.s. of Eq.~\ref{eqn:complex-rad-expansion} are upper bounded by the ones of the Gaussian distribution. Hence, we obtain
\begin{align*}
    \text{(Complex case)} \quad \quad
    \| (\mat{w}_i)^\top \mat{a} \|_{L^t}
    \leq C_t = \Gamma(t/2 + 1)^{1/t}
    \quad \text{for} \quad t=2k, k \in \mathbb{N}.
\end{align*}
Since we have only derived $C_t$ for $t = 2k, k\in\mathbb{N}$ for complex Rademacher sketches, we derive an interpolation strategy when $t \neq 2k$ in the following.

\paragraph{Interpolation of $L^t$-norms for complex Rademacher sketches}

\begin{figure}[t]
\begin{center}
\centerline{\includegraphics[width=0.8\linewidth]{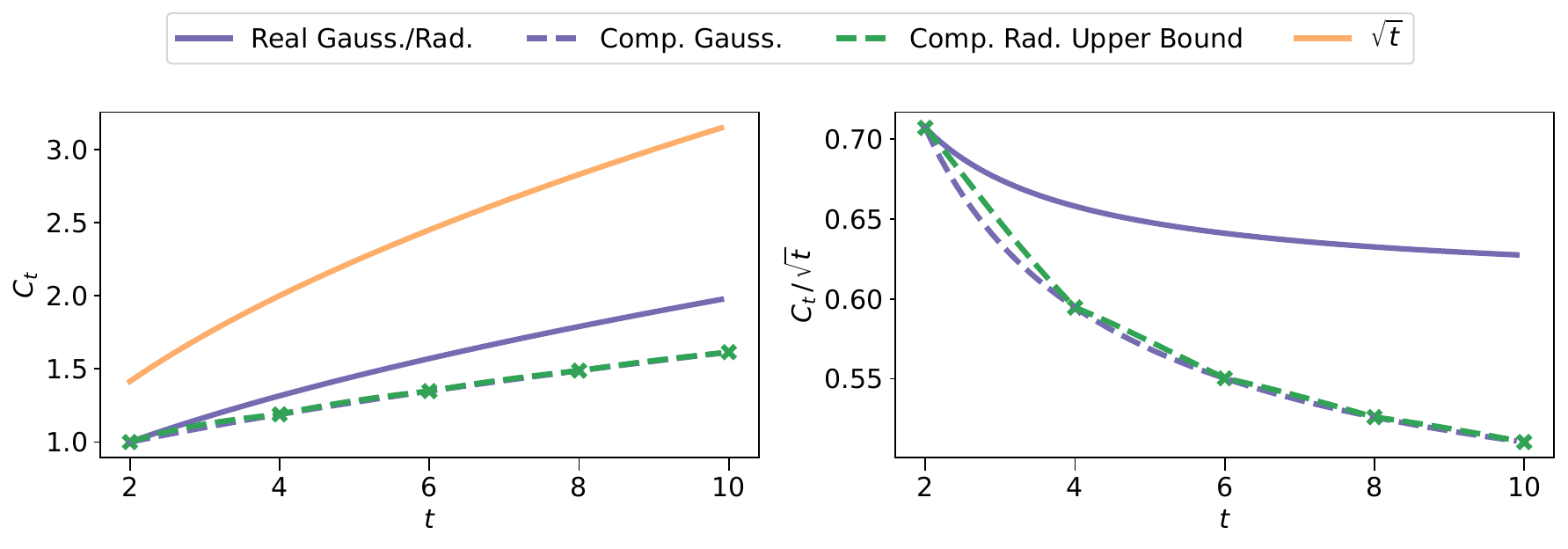}}
\caption{
(Left) $C_t$ values over $t \geq 2$. Values for the complex Rademacher case are interpolations between $t=2k, k \in \mathbb{N}$. (Right) $C_t$ values after division by $\sqrt{t}$.
}
\label{fig:moment-bound-extended}
\end{center}
\end{figure}

For two random variables $X,Y \in \mathbb{C}$, H\"older's inequality gives
\begin{align*}
    \| XY \|_{L^1} \leq \|X\|_{L^a} \|Y\|_{L^b}
    \quad \text{for} \quad
    a, b > 1
    \quad \text{with} \quad
    1/a + 1/b = 1.
\end{align*}
We now define $a' := (b-a)/(b-t)$ and $b' := (b-a)/(t-a)$ for $a < t < b$, such that $a', b' > 1$ and $1/a' + 1/b' = 1$ are satisfied. We further have $a/a' + b/b' = t$. So we define a random variable $Z \in \mathbb{C}$ and obtain
\begin{align*}
    \| Z^t \|_{L^1}
    = \| Z^{\frac{a}{a'}} Z^{\frac{b}{b'}} \|_{L^1}
    \leq \| Z^{\frac{a}{a'}} \|_{L^{a'}} \| Z^{\frac{b}{b'}} \|_{L^{b'}}
    = \| Z \|_{L^a}^{a/a'} \| Z \|_{L^b}^{b/b'}
    \iff
    \| Z \|_{L^t}
    \leq \| Z \|_{L^a}^{a/(ta')} \| Z \|_{L^b}^{b/(tb')}
\end{align*}
via H\"older's inequality. We can thus set $a$ and $b$ equal to the closest even integer values below/above $t$ and therefore obtain an upper bound for $C_t$. That is, we bound $\| Z \|_{L^a}^{a/(ta')} \leq C_a^{a/(ta')}$ and $\| Z \|_{L^b}^{b/(tb')} \leq C_b^{b/(tb')}$. Since $\|Z\|_{L^t} \leq C_t$ is assumed to be tight, we must have $\|Z\|_{L^t} \leq C_t \leq C_a^{a/(ta')} C_b^{b/(tb')}$.

The left plot of Fig.~\ref{fig:moment-bound-extended} shows $C_t$ over $t$ including our proposed interpolation for $t \neq 2k$ for the complex Rademacher case. We see that the upper bound matches the values for the complex Gaussian distribution almost exactly from $t \geq 4$. This is a strong indication for the fact that the actual $C_t$ values are the same for both distributions, as we already showed for the real case. Furthermore, the upper bound for the complex Rademacher $C_t$ values remains smaller than the $C_t$ values for the real case. All functions grow more slowly than $\sqrt{t}$ as shown in the right plot of Fig.~\ref{fig:moment-bound-extended}.

\subsubsection{Proof of Lemma~\ref{lemma:app-absolute-moment-bound} (Lemma~\ref{lemma:abs-moment-bound} in the Main Paper)}

Having derived the $C_t$ values for (complex) Gaussian and Rademacher distributions, we are ready to prove Lem.~\ref{lemma:app-absolute-moment-bound}. The proof closely follows \citet[Lem. 19]{Ahle2020}, but extends it to the case of complex $\{\mat{w}_i\}_{i=1}^p$. We therefore provide the whole proof for completeness.

The proof is by induction. The initial case $p=1$ is trivially fulfilled by the previous derivations. For the induction step, we assume that the claim is true for $p-1$. So we assume $\| (\otimes_{i=1}^{p-1} \mat{w}_i)^\top \mat{x} \|_{L^t} \leq C_t^{p-1} \| \mat{x} \|_2$. We now index the vector $\mat{x} \in \mathbb{R}^{d_1 \cdots d_p}$ in a tensorized fashion. So a single element of $\mat{x}$ is indexed as $x_{I_1, \dots, I_p}$ for indices $I_i \in \{1, \dots, d_i\}$. Let further $B_{I_1, \dots, I_{p-1}} = \sum_{I_p \in [d_p]} (\mat{w}_p)_{I_p} x_{I_1, \dots, I_p} \in \mathbb{C}$. Then $\mu_t := \mathbb{E}[|(\otimes_{i=1}^{p} \mat{w}_i)^\top \mat{x}|^t]$ yields
\begin{align*}
    \mu_t = \mathbb{E} \left[ \left| \sum_{I_1 \in [d_1], \dots, I_p \in [d_p]} \left( \prod_{i \in [p]} (\mat{w}_i)_{I_i} \right) x_{I_1, \dots, I_p} \right|^t \right]
    = \mathbb{E} \left[ \left| \sum_{I_1 \in [d_1], \dots, I_{p-1} \in [d_{p-1}]} \left( \prod_{i \in [p-1]} (\mat{w}_i)_{I_i} \right) B_{I_1, \dots, I_{p-1}} \right|^t \right].
\end{align*}
By the law of total expectation, this gives
\begin{align*}
    \mu_t = \mathbb{E} \left[ \mathbb{E} \left[ \left| \sum_{I_1 \in [d_1], \dots, I_{p-1} \in [d_{p-1}]} \left( \prod_{i \in [p-1]} (\mat{w}_i)_{I_i} \right) B_{I_1, \dots, I_{p-1}} \right|^t \Bigg| \mat{w}_p \right] \right].
\end{align*}
By the induction assumption, we get
\begin{align*}
    \mu_t \leq C_t^{t(p-1)} \mathbb{E} \left[ \left| \sum_{I_1 \in [d_1], \dots, I_{p-1} \in [d_{p-1}]} |B_{I_1, \dots, I_{p-1}}|^2 \right|^{1/2 \cdot t} \right]
    = C_t^{t(p-1)} \left\| \sum_{I_1 \in [d_1], \dots, I_{p-1} \in [d_{p-1}]} |B_{I_1, \dots, I_{p-1}}|^2 \right\|_{L^{t/2}}^{t/2}.
\end{align*}
Since $t/2 \geq 1$, we use Minkowski's inequality (triangle inequality for the $L_t$-norm) to move the norm inside the sum:
\begin{align*}
    \mu_t \leq C_t^{t(p-1)} \left(\sum_{I_1 \in [d_1], \dots, I_{p-1} \in [d_{p-1}]} \| |B_{I_1, \dots, I_{p-1}}|^2 \|_{L^{t/2}}\right)^{t/2}
    = C_t^{t(p-1)} \left(\sum_{I_1 \in [d_1], \dots, I_{p-1} \in [d_{p-1}]} \| B_{I_1, \dots, I_{p-1}} \|_{L^t}^2 \right)^{t/2}.
\end{align*}
Now, recall that $B_{I_1, \dots, I_{p-1}}$ is a weighted sum with weights $\{(\mat{w}_p)_{I_p}\}_{I_p=1}^{d_p}$. So we can use the initial assumption $\| \dotprodi{w}{a}{p} \|_{L^t} \leq C_t \norm{a}_2$ for all $\mat{a} \in \mathbb{R}^{d_p}$. Therefore, we have $\|B_{I_1, \dots, I_{p-1}}\|_{L^t} \leq C_t (\sum_{I_p \in [d_p]} (x_{I_1, \dots, I_p})^2)^{1/2}$, and finally
\begin{align*}
    \mu_t \leq C_t^{t(p-1)} \left( \sum_{I_1 \in [d_1], \dots, I_{p-1} \in [d_{p-1}]} C_t^2 \sum_{I_p \in [d_p]} (x_{I_1, \dots, I_p})^2 \right)^{t/2}
    = C_t^{tp} \norm{x}_2^{2 \cdot t/2}
    = C_t^{tp} \norm{x}_2^t,
\end{align*}
which proves the claim $\| (\otimes_{i=1}^{p} \mat{w}_i)^\top \mat{x}\|_{L^t} \leq C_t^p \norm{x}_2$. \hfill\BlackBox

When $\mat{x} = \otimes_{i=1}^p \mat{x}_i$ is a tensor product for some $\mat{x}_i \in \mathbb{R}^{d_i}$, $i=1, \dots, p$. Then we have
$$
\mathbb{E} [|(\otimes_{i=1}^{p} \mat{w}_i)^\top (\otimes_{i=1}^{p} \mat{x}_i)|^t]
= \mathbb{E} \left[ \left|\prod_{i=1}^p\mat{w}_i^\top \mat{x}_i \right|^t \right]
= \prod_{i=1}^p \mathbb{E} \left[ \left|\mat{w}_i^\top \mat{x}_i \right|^t \right]
\leq \prod_{i=1}^p C_t^t \| \mat{x}_i \|_2^t
= C_t^{tp} \| \otimes_{i=1}^p \mat{x}_i \|_2^t.
$$
Since $\mathbb{E} [ |\mat{w}_i^\top \mat{x}_i |^t ] \leq C_t^t \| \mat{x}_i \|_2^t$ is tight by the assumption that $C_t$ are tight constants, the bound in Lem.~\ref{lemma:app-absolute-moment-bound} becomes tight too in this case.

\subsection{Proof of Theorem~\ref{thrm:complex-norm-pres}}

\label{sec:app-complex-norm-pres}

We prove Thm.~\ref{thrm:complex-norm-pres} for the more general case of $\delta \in (0, \exp(-2 p \gamma))$, for which we introduce an additional variable $\gamma > 0$:
\begin{theorem}
Let $\epsilon, \gamma > 0, p \in \mathbb{N}, \delta \in (0, \exp(-2 p \gamma)), \mat{x} \in \mathbb{R}^{d_1 \cdots d_p}, \mat{S} = (\mat{s}_1, \dots, \mat{s}_D)^\top \in \mathbb{C}^{D \times d_1 \cdots d_p}$ with $\mat{s}_\ell = \otimes_{i=1}^p \mat{w}_{i, \ell} / \sqrt{D}$ and $\mat{w}_{i,\ell} \in \mathbb{C}^{d_i}$ be i.i.d. Gaussian/Rademacher samples as in Lem.~\ref{lemma:abs-moment-bound}.
In order to guarantee
\begin{gather*}
    {\rm Pr} \left\{ | \norm{Sx}_2^2 - \norm{x}_2^2 | \leq \epsilon \norm{x}_2^2 \right\}
    \geq 1 - \delta, \quad \text{we need} \\
    D = \bigO \left(\max\left\{(C_4 e^{\gamma/2})^{4p} \left(\frac{\log(1/\delta)}{p \gamma}\right) \epsilon^{-2} ,
    \quad
    (C_4^2 e/2 e^\gamma)^{p} \left(\frac{\log(1/\delta)}{p \gamma}\right)^p \epsilon^{-1}\right\}\right),
\end{gather*}
where $C_4$ equals $3^{1/4} (2^{1/4})$ for the real (complex) Gaussian/Rademacher distribution, as derived in Lem.~\ref{lemma:abs-moment-bound}.
\end{theorem}

Setting $\gamma = 1/p$ yields the formulation of Thm.~\ref{thrm:complex-norm-pres}.
Alternatively, we may allow $\delta \in (0, 1)$ by letting $\gamma \propto 1/p$ go towards zero. However, this leads to a worse dependence on $p$, since $(1/\gamma)^p$ becomes arbitrarily large.

\paragraph{Proof}

The proof is an extension of \citet[Appendix A.2]{Ahle2020} to the case of complex $\mat{S}$. It thus makes use of Lem.~\ref{lemma:abs-moment-bound} in order to prove the theorem. Moreover, we make use of interpolated values for $C_t$ when $t \neq 2k$ for any $k \in \mathbb{N}$ by using an upper bound on $C_t$ in this case, as shown in Section~\ref{sec:app-complex-moments}. Crucially however, this upper-bound interpolation does not harm the sharpness of our results.
Moreover, the original proof by \citet[Appendix A.2]{Ahle2020} requires $t = \frac{\log(1/\delta)}{p \gamma} \geq 4$, which we relax to $t = \frac{\log(1/\delta)}{p \gamma} > 2$ to allow for a larger range of error probabilities, i.e., we require $\delta \in (0, \exp(-2p\gamma))$ instead of $\delta \in (0, \exp(-4p\gamma))$ in the theorem.
We provide the entire modified proof here for completeness.

Our goal is to show that
$
    \| \norm{Sx}_2^2 - \norm{x}_2^2 \|_{L^t} \leq \delta^{1/t} \epsilon \norm{x}_2^2
$
holds. Then we can apply Markov's inequality: ${\rm Pr} \{ X \geq a \} \leq \mathbb{E}[X]/a$ for $a > 0$, where we set $X = |\norm{Sx}_2^2 - \norm{x}_2^2 |^t$ and $a=\epsilon^{t} \norm{x}_2^{2t}$ to obtain
\begin{align}
    \label{eqn:app-norm-inequality}
    {\rm Pr} \{ |\norm{Sx}_2^2 - \norm{x}_2^2 |^t \geq \epsilon^t \norm{x}_2^{2t} \} \leq \delta
    \quad \iff \quad
    {\rm Pr} \{ | \norm{Sx}_2^2 - \norm{x}_2^2 | \leq \epsilon \norm{x}_2^2 \} \geq 1 - \delta.
\end{align}
Without loss of generality, we can assume $\norm{x}_2 = 1$ from now onward, since
\begin{align*}
    \left\| \norm{Sx}_2^2 - \norm{x}_2^2 \right\|_{L^t} \leq \epsilon \delta^{1/t} \norm{x}_2^2
    \iff
    \left\| \left\| \mat{S} \left(\frac{\mat{x}}{\norm{x}_2}\right) \right\|_2^2 - 1 \right\|_{L^t} \leq \epsilon \delta^{1/t}.
\end{align*}
In order to prove $\| \norm{Sx}_2^2 - 1 \|_{L^t} \leq \epsilon \delta^{1/t}$, we write $\mat{S} = (\mat{s}_1, \dots, \mat{s}_D)^\top$ with $\mat{s}_\ell = (\otimes_{i=1}^p \mat{w}_{i, \ell}) / \sqrt{D}$ and $\mat{w}_{i, \ell} \in \mathbb{C}^{d_i}$ i.i.d. as in Section~\ref{sec:complex-bounds}. So we can reformulate
\begin{align}
    \label{eqn:sum_z_ell}
    \left\| \norm{Sx}_2^2 - 1 \right\|_{L^t}
    = \left\| \left( \frac{1}{D} \sum_{\ell=1}^D \left|(\otimes_{i=1}^p \mat{w}_{i, \ell})^\top \mat{x}\right|^2 \right) - 1 \right\|_{L^t}
    = \left\| \frac{1}{D} \sum_{\ell=1}^D Z_\ell \right\|_{L^t}
\end{align}
with $Z_\ell := |(\otimes_{i=1}^p \mat{w}_{i, \ell})^\top \mat{x}|^2 - 1$ being i.i.d. random variables with zero mean, since $\mathbb{E} [|(\otimes_{i=1}^p \mat{w}_{i, \ell})^\top \mat{x}|^2] = \norm{x}_2^2 = 1$. Next, we bound $\| Z_\ell \|_{L^t}$ using Minkowski’s inequality:
\begin{align}
    \label{eqn:z-ell}
    \| Z_\ell \|_{L^t}
    =  \| |(\otimes_{i=1}^p \mat{w}_{i, \ell})^\top \mat{x}|^2 - 1 \|_{L^t}
    \leq \| |(\otimes_{i=1}^p \mat{w}_{i, \ell})^\top \mat{x}|^2 \|_{L^t} + \| -1 \|_{L^t}
    = \| (\otimes_{i=1}^p \mat{w}_{i, \ell})^\top \mat{x} \|_{L^{2t}}^2 + 1.
\end{align}
We can further bound $\| (\otimes_{i=1}^p \mat{w}_{i, \ell})^\top \mat{x} \|_{L^{2t}} \leq C_{2t}^p$ for any $t \geq 1$ by Lem.~\ref{lemma:abs-moment-bound}. Precise values $C_{t'}$ are derived in the Lemma, except for the complex Rademacher case for which we provide values for $t'=2k, k \in \mathbb{N}$. For $t' \neq 2k$, we can use the upper-bound interpolation
\begin{align*}
    C_{t'}
    \leq C_{a}^{\frac{a(b-t')}{t'(b-a)}} C_b^{\frac{b(t'-a)}{t'(b-a)}}
    \quad \text{with} \quad
    a < t' < b,
\end{align*}
where we choose $a$ and $b$ to be the closest even integer values below and above $t$, respectively.

As shown in Fig.~\ref{fig:moment-bound-extended}, $C_{t'}$ grows more slowly than $\sqrt{t}$. $f(t') := C_{t'} / \sqrt{t'}$ is thus a monotonically decreasing function. Now we set $t'=2t > 4$ by our initial assumption $t > 2$, and we obtain $f(t') \leq C_4 / \sqrt{4}$. We can thus bound $C_{2t} \leq C_4 \sqrt{2t} / \sqrt{4}$. This eventually allows us to bound $\| Z_\ell \|_{L^t} \leq (C_4^2 t / 2)^p + 1$ for all $\ell \in \{1, \dots, D\}$ in Eq.~\ref{eqn:z-ell}. Notably, the fact of having interpolated $C_{t'}$ by using an upper bound for complex Rademacher sketches has not influenced this result, since $f(t') \leq C_4 \sqrt{4}$ would remain valid even if the $C_{t'}$ values were smaller for $t' \neq 2k$.

In order to bound $\| \frac{1}{D} \sum_{\ell=1}^D Z_\ell \|_{L^t}$ (\ref{eqn:sum_z_ell}), we need Latala's inequality:
\begin{lemma}[\cite{latala1997}, Corollary 2]
\label{lemma:latala}

If $p \geq 2$ and $X, X_1, \dots, X_n$ are i.i.d. symmetric random variables, then we have
\begin{align*}
    \| X_1 + \dots + X_n \|_{L^t}
    \sim \sup \left\{ \frac{t}{s} \left(\frac{n}{t}\right)^{1/s} \|X\|_{L^s} \Bigg| \max\{2, t/n\} \leq s \leq t \right\}.
\end{align*}
\end{lemma}
Here, $f(x) \sim g(x)$ means $c_1 g(x) \leq f(x) \leq c_2 g(x)$ for all $x$ and some universal constants $c_1, c_2$. By \citet[Remark 2]{latala1997}, the lemma is also valid for zero-mean random variables with slightly worse constants $1/2 \, c_1$ and $2 \, c_2$.


Recall that $\| Z_\ell \|_{L^t} \leq (C_4^2 t / 2)^p + 1 \leq c_3 (C_4^2 t / 2)^p$ for some $c_3 > 0$.
W.l.o.g., we can set $c_3=1$ when substituting $\|X\|_{L^s}$ by $c_3 (C_4^2 t / 2)^p$ inside Lem.~\ref{lemma:latala}. The functional form $h(t) := (K t)^p$ with $K := C_4^2 / 2 > 0$ allows us to greatly simplify Lem.~\ref{lemma:latala} using the following corollary.
\begin{corollary}[\cite{Ahle2020}, Corollary 38]
\label{cor:latala}

If $\|X\|_{L^s} \sim (K s)^p$ for some $p \geq 1, K > 0$, then the supremum in Lem.~\ref{lemma:latala} is attained for the minimal and maximal case of $s$, i.e., $s=\max\{2, t/n\}$ and $s=t$. Lem.~\ref{lemma:latala} then becomes
\begin{align*}
    \| X_1 + \dots + X_n \|_{L^t}
    \sim K^p \max \left\{ \sqrt{tn} 2^p, \left(\frac{n}{t}\right)^{1/t} t^p \right\}.
\end{align*}
\end{corollary}
Using Cor.~\ref{cor:latala} and setting $K=C_4^2/2$, we obtain the following bound on $\| \frac{1}{D} \sum_{\ell=1}^D Z_\ell \|_{L^t}$ (\ref{eqn:sum_z_ell}): 
\begin{align}
    \label{eqn:z-ell-max-bound}
    \left\| \frac{1}{D} \sum_{\ell=1}^D Z_\ell \right\|_{L^t}
    \sim \frac{1}{D} (C_4^2/2)^p \max \left\{ \sqrt{tD} 2^p, \left( \frac{D}{t} \right)^{1/t} t^p \right\}
    = \max \{ \underbrace{C_4^{2p} \sqrt{t/D}}_{(1)}, \underbrace{(C_4^2/2 \, t)^p (D/t)^{1/t} / D}_{(2)} \}.
\end{align}
Recall that our goal is to provide a condition on $D$ for which $\| \frac{1}{D} \sum_{\ell=1}^D Z_\ell \|_{L^t} \leq \epsilon \delta^{1/t}$ holds. Since we can freely choose $t > 2$, we set it to $t = \frac{\log(1/\delta)}{p \gamma} > 2$ for some $\gamma > 0$ from now onward. Then, it is only left to show that terms (1) and (2) in Eq.~\ref{eqn:z-ell-max-bound} are upper-bounded by $\epsilon \delta^{1/t}$. We start with the simpler case (1).

\paragraph{Analysis of case (1)}
Setting $D \geq (C_4 e^{\gamma/2})^{4p} \frac{\log(1/\delta)}{p \gamma} \epsilon^{-2}$ directly gives
\begin{align*}
    (1)
    = (C_4)^{2p} \sqrt{\frac{\log (1 / \delta)}{p \gamma D}}
    \leq (C_4)^{2p} \sqrt{\frac{\log (1 / \delta) \epsilon^2}{p \gamma (C_4 e^{\gamma/2})^{4p} \frac{\log (1/ \delta)}{p \gamma}}}
    = \epsilon e^{-\gamma p}
    = \epsilon e^{-\gamma p t/t}
    = \epsilon \delta^{1/t}.
\end{align*}
%

\paragraph{Analysis of case (2)}
For a simpler analysis, we start by upper bounding $D^{1/t}$ in case (2). For this purpose, we study the condition in which $(2) \geq (1)$ s.t. our error is upper bounded by $(2)$. We have
\begin{align*}
    &\underbrace{(C_4^2)^p \sqrt{t/D}}_{(1)} \leq \underbrace{(C_4^2/2 \, t)^p (D/t)^{1/t} / D}_{(2)} \leq (C_4^2/2 \, t)^p D^{1/t} / D \\
    &\iff 2^p t^{1/2-p} \leq D^{1/t-1/2}
    \iff D^{1/t} \leq \left(\frac{t}{2}\right)^{\frac{2p-1}{t-2}} \left( \frac{1}{2} \right)^{\frac{1}{t-2}}
    \, \stackrel{(t>2)}{\leq} \, \left(\frac{t}{2}\right)^{\frac{2p-1}{t-2}}
    \, \stackrel{(t \rightarrow 2)}{\leq} \, \exp(p-1/2).
\end{align*}
%
%

Thus, if $(2) \geq (1)$, we have $(2) \leq e^{-1/2} (C_4^2/2 \, e t)^p / D$. Setting $D \geq (C_4^2/2 \, t e e^\gamma)^{p} \epsilon^{-1} e^{-1/2}$ finally yields
\begin{align*}
    (2) \leq e^{-1/2} (C_4^2/2 \, e t)^p / D
    \leq \frac{e^{-1/2} (C_4^2/2 \, e t)^p}{(C_4^2/2 \, t e e^\gamma)^{p} \epsilon^{-1} e^{-1/2}}
    = \epsilon e^{-\gamma p}
    = \epsilon \delta^{1/t}.
\end{align*}
%

Setting $D$ to the maximum value of the conditions of case (1) and (2) ensures that $\| \frac{1}{D} \sum_{\ell=1}^D Z_\ell \|_{L^t} \leq \epsilon \delta^{1/t}$ in Eq.~\ref{eqn:z-ell-max-bound}. \hfill \BlackBox

\subsection{A Comparison with \cite{wacker2022}, Theorem 3.4}
\label{sec:app-wacker-comparison}

\citet[Theorem 3.4]{wacker2022} provide an error bound relative to the L1-norm of the form
\begin{align*}
{\rm Pr} \left\{ | \|\mat{S} \mat{a}^{\otimes p} \|_2^2 - \|\mat{a}^{\otimes p}\|_2^2 | \leq \epsilon \|\mat{a}\|_1^{2p} \right\}
\geq 1 - \delta,
\end{align*}
where $\mat{S} \in \mathbb{C}^{D \times d^p}$ is a complex Rademacher sketch and $\mat{a} \in \mathbb{R}^d$.

Bounding the error $\epsilon > 0$ relative to the L1-norm of $\mat{a}$ instead of the L2-norm is problematic. To see this, consider the vector $\mat{a} = (1, \dots, 1)^\top / \sqrt{d} \in \mathbb{R}^d$. It has $\norm{a}_2=1$ and $\norm{a}_1 = \sqrt{d}$. In this case, we have $\norm{a}_1^{2p} = d^p \norm{a}_2^{2p}$. Since the bound by \citet{wacker2022} requires $D=\bigO(\epsilon^{-2})$, this would translate into a guarantee of $D=\bigO(d^{2p})$ to bound the error relative to $\norm{a}_2^{2p}$. Hence, $D$ is already larger than the dimension $d^p$ of $\mat{a}^{\otimes p}$, which defeats the purpose of dimensionality reduction.

\subsection{Proof of Corollary~\ref{cor:approx-matrix} (Approximate Matrix Product)}

\label{sec:app-approx-matrix}

We want to show that
\begin{align*}
    {\rm Pr} \left\{ | (\mat{S}_{\rm CtR} \mat{x})^\top (\mat{S}_{\rm CtR} \mat{y}) - \dotprod{x}{y} | \leq \epsilon \norm{x}_2 \norm{y}_2 \right\}
    \geq 1 - \delta
\end{align*}
holds for any $\mat{x}, \mat{y} \in \mathbb{R}^{d_1 \cdots d_p}$ and
$
\mat{S}_{\rm CtR}
:= ({\rm Re} \{\mat{s}_1\}, \dots, {\rm Re} \{\mat{s}_{D}\}, {\rm Im} \{\mat{s}_1\}, \dots, {\rm Im} \{\mat{s}_{D}\})^\top
\in \mathbb{R}^{2D \times d_1 \cdots d_p}
$
with $\mat{S} = (\mat{s}_1, \dots, \mat{s}_D)^\top$ being the same as in Thm.~\ref{thrm:complex-norm-pres}.

\begin{proof}
Our first goal is to show that the following holds:
\begin{align*}
    \| (\mat{S}_{\rm CtR} \mat{x})^\top (\mat{S}_{\rm CtR} \mat{y}) - \dotprod{x}{y} \|_{L^t} \leq \epsilon \delta^{1/t} \norm{x}_2 \norm{y}_2.
\end{align*}
W.l.o.g., we can assume $\norm{x}_2 = \norm{y}_2 = 1$ from now onward, since both sides of the inequality can be divided by $\norm{x}_2 \norm{y}_2$.

In Section~\ref{sec:app-complex-norm-pres}, we have shown that
$
    \| \norm{Sx}_2^2 - \norm{x}_2^2 \|_{L^t} \leq \delta^{1/t} \epsilon \norm{x}_2^2
$
holds for any $\mat{x} \in \mathbb{R}^{d_1 \cdots d_p}$ and $t>2$.
Recall that
$\| \mat{S}_{\rm CtR} \mat{x} \|_2^2
= \sum_{\ell=1}^D {\rm Re} \{ \mat{s}_\ell^\top \mat{x} \}^2 + {\rm Im} \{ \mat{s}_\ell^\top \mat{x} \}^2 = \norm{Sx}_2^2$, which already implies
$
    \| \| \mat{S}_{\rm CtR} \mat{x}\|_2^2 - \norm{x}_2^2 \|_{L^t} \leq \delta^{1/t} \epsilon \norm{x}_2^2.
$

The rest of the proof follows \citet[Lem. 9]{Ahle2020}.
For two vectors $\mat{a}, \mat{b}$, we have $\|\mat{a}-\mat{b}\|_2^2 = \norm{a}_2^2 + \norm{b}_2^2 - 2 (\dotprod{a}{b})$ and $\| \mat{a} + \mat{b} \|_2^2 = \norm{a}_2^2 + \norm{b}_2^2 + 2 (\dotprod{a}{b})$. Being combined, this gives $\dotprod{a}{b} = \left( \|\mat{a}+\mat{b}\|_2^2 - \|\mat{a}-\mat{b}\|_2^2 \right) / 4$. Hence
\begin{align*}
    \left\| (\mat{S}_{\rm CtR} \mat{x})^\top (\mat{S}_{\rm CtR} \mat{y}) - \dotprod{x}{y} \right\|_{L^t}
    &= \left\| \left\|\mat{S}_{\rm CtR} (\mat{x} + \mat{y}) \right\|_2^2 - \left\|\mat{S}_{\rm CtR} (\mat{x} - \mat{y}) \right\|_2^2
    - \left\|\mat{x} + \mat{y} \right\|_2^2 + \left\|\mat{x} - \mat{y} \right\|_2^2
    \right\|_{L^t} / 4 \\
    &\leq \left(\left\| \left\|\mat{S}_{\rm CtR} (\mat{x} + \mat{y}) \right\|_2^2 - \left\|\mat{x} + \mat{y} \right\|_2^2 \right\|_{L^t}
    + \left\| \left\|\mat{S}_{\rm CtR} (\mat{x} - \mat{y}) \right\|_2^2
     - \left\|\mat{x} - \mat{y} \right\|_2^2
    \right\|_{L^t}\right) / 4 \\
    &\leq \epsilon \delta^{1/t} \left(\left\|\mat{x} + \mat{y} \right\|_2^2 + \left\|\mat{x} - \mat{y} \right\|_2^2 \right) / 4
    = \epsilon \delta^{1/t} \left( \norm{x}_2^2 + \norm{y}_2^2 \right) / 2
    = \epsilon \delta^{1/t}.
\end{align*}
To conclude the proof, we apply Markov's inequality ${\rm Pr}(X \geq a) \leq \mathbb{E}[X] / a$ with $a > 0$, as we did in Section~\ref{sec:app-complex-norm-pres}.

In the cases of matrices $\mat{X} = (\mat{x}_1, \dots, \mat{x}_n) \in \mathbb{R}^{d_1 \cdots d_p \times n}, \mat{Y} = (\mat{y}_1, \dots, \mat{y}_m) \in \mathbb{R}^{d_1 \cdots d_p \times m}$, we set $a = \epsilon^2 \| \mat{X} \|_F^2 \| \mat{Y} \|_F^2$ and $X = \| (\mat{S}_{\rm CtR} \mat{X})^\top (\mat{S}_{\rm CtR} \mat{Y}) - \mat{X}^\top \mat{Y} \|_F^2$ inside the inequality. Then we get
\begin{align}
    \label{eqn:app-frobenius-bound}
    {\rm Pr} \left\{ (\mat{S}_{\rm CtR} \mat{X})^\top (\mat{S}_{\rm CtR} \mat{Y}) - \dotprod{X}{Y} \|_F^2 \geq \epsilon^2 \norm{X}_F^2 \norm{Y}_F^2 \right\}
    \leq \frac{\epsilon^2 \delta \sum_{i=1}^n \sum_{j=1}^m \|\mat{x}_i\|_2^2 \|\mat{y}_j\|_2^2}{\epsilon^2 \| \mat{X} \|_F^2 \| \mat{Y} \|_F^2} = \delta
\end{align}
when $\mat{S}_{\rm CtR}$ has $2D$ rows with $D$ being the same as in Thm.~\ref{thrm:complex-norm-pres}.
\end{proof}

\subsection{From Approximate Matrix Products to Subspace Embeddings}

\label{sec:app-ose}

We now use the inequality (\ref{eqn:app-frobenius-bound}) derived in Section~\ref{sec:app-approx-matrix} to bound the spectral approximation error of the polynomial kernel matrix. We define the target gram matrix $\mat{K} := (\mat{X}^{\otimes p})^\top \mat{X}^{\otimes p} + \lambda \mat{I}_n$, where $\mat{X}^{\otimes p} = (\mat{x}_1^{\otimes p}, \dots, \mat{x}_n^{\otimes p}) \in \mathbb{R}^{d^p \times n}$ is a matrix containing the polynomial feature maps of the data points $\{\mat{x}_i\}_{i=1}^n$, and $\lambda \geq 0$ is a regularization parameter. Our task is to determine $D$ for which we can guarantee
\begin{align}
    \label{eqn:app-spectral-order}
    (1-\epsilon) (\mat{K} + \lambda \mat{I}_n)
    \preceq (\mat{S}_{\rm CtR} \mat{X}^{\otimes p})^\top (\mat{S}_{\rm CtR} \mat{X}^{\otimes p}) + \lambda \mat{I}_n
    \preceq (1+\epsilon) (\mat{K} + \lambda \mat{I}_n)
\end{align}
with probability at least $1-\delta$.

\begin{proof}
We rephrase \citet[Lem. 11]{Ahle2020} here for the case $\lambda > 0$. This ensures that $\mat{K} + \lambda \mat{I}_n$ is positive definite and $(\mat{K} + \lambda \mat{I}_n)^{-1/2}$ exists. The same result for $\lambda=0$ can then be obtained using Fatou's as shown in the original lemma.

By \citet[Prop. 2.1.1.]{Tropp2012}, left and right multiplying the spectral inequality \eqref{eqn:app-spectral-order} by $(\mat{K} + \lambda \mat{I}_n)^{-1/2}$ does not change the positive semi-definite order. So \eqref{eqn:app-spectral-order} becomes
\begin{align*}
    (1-\epsilon) \mat{I}_n \preceq
    (\mat{S}_{\rm CtR} \mat{X}^{\otimes p} (\mat{K} + \lambda \mat{I}_n)^{-1/2})^\top
    (\mat{S}_{\rm CtR} \mat{X}^{\otimes p} (\mat{K} + \lambda \mat{I}_n)^{-1/2})^\top + \lambda (\mat{K} + \lambda \mat{I}_n)^{-1}
    \preceq (1 + \epsilon) \mat{I}_n,
\end{align*}
which is equivalent to
\begin{align}
    \label{eqn:operator-inequality}
    \| (\mat{S}_{\rm CtR} \mat{X}^{\otimes p} (\mat{K} + \lambda \mat{I}_n)^{-1/2})^\top
    (\mat{S}_{\rm CtR} \mat{X}^{\otimes p} (\mat{K} + \lambda \mat{I}_n)^{-1/2})^\top + \lambda (\mat{K} + \lambda \mat{I}_n)^{-1} - \mat{I}_n \|_2 \leq \epsilon.
\end{align}
Now we define $\mat{Z} := \mat{X}^{\otimes p} ((\mat{X}^{\otimes p})^\top \mat{X}^{\otimes p} + \lambda \mat{I}_n)^{-1/2}$ so that
\begin{align*}
    \mat{Z}^\top \mat{Z}
    &= (\mat{K} + \lambda \mat{I}_n)^{-1/2} \mat{K} (\mat{K} + \lambda \mat{I}_n)^{-1/2} \\
    &= (\mat{K} + \lambda \mat{I}_n)^{-1/2} (\mat{K} + \lambda \mat{I}_n - \lambda \mat{I}_n) (\mat{K} + \lambda \mat{I}_n)^{-1/2} \\
    &= \mat{I}_n - \lambda (\mat{K} + \lambda \mat{I}_n)^{-1}.
\end{align*}
Then \eqref{eqn:operator-inequality} becomes $\| (\mat{S}_{\rm CtR} \mat{Z})^\top (\mat{S}_{\rm CtR} \mat{Z}) - \mat{Z}^\top \mat{Z} \|_2 \leq \epsilon$.
and we can apply our bound on the Frobenius norm error \eqref{eqn:app-frobenius-bound}, since it holds for \textit{any} $\mat{X}, \mat{Y}$. So we can set $\mat{X}=\mat{Y}=\mat{Z} \in \mathbb{R}^{d^p \times n}$ and get:
\begin{align}
    \nonumber
    &{\rm Pr} \{ \| (\mat{S}_{\rm CtR} \mat{Z})^\top (\mat{S}_{\rm CtR} \mat{Z}) - \mat{Z}^\top \mat{Z} \|_2 \geq \epsilon \norm{Z}_F \norm{Z}_F \} \\
    &\leq {\rm Pr} \{ \| (\mat{S}_{\rm CtR} \mat{Z})^\top (\mat{S}_{\rm CtR} \mat{Z}) - \mat{Z}^\top \mat{Z} \|_F \geq \epsilon \norm{Z}_F \norm{Z}_F \}
    \leq \delta.
    \label{eqn:app-spectral-to-frob}
\end{align}
Now we have that
\begin{align*}
    \norm{Z}_F^2
    = {\rm tr} ( \mat{Z}^\top \mat{Z} )
    = {\rm tr} ( \mat{I}_n - \lambda (\mat{K} + \lambda \mat{I}_n)^{-1} )
\end{align*}
and ${\rm tr} ( \mat{Z}^\top \mat{Z} ) = \sum_{i=1}^n \lambda_i (\mat{Z}^\top \mat{Z})$ is the sum over eigenvalues $\lambda_i(\mat{Z}^\top \mat{Z})$. This gives
\begin{align*}
    {\rm tr} ( \mat{Z}^\top \mat{Z} )
    = \sum_{i=1}^n 1 - \frac{\lambda}{\lambda + \lambda_i(\mat{K})}
    = \sum_{i=1}^n \frac{\lambda_i(\mat{K})}{\lambda_i(\mat{K}) + \lambda}
    = {\rm tr} (\mat{K} (\mat{K} + \lambda \mat{I})^{-1})
    =: s_{\lambda}(\mat{K}),
\end{align*}
where $0 \leq s_\lambda(\mat{K}) \leq n$ is the $\lambda$-statistical dimension of $\mat{K}$.

Substituting $\epsilon = \epsilon' / \norm{Z}_F^2 = \epsilon' s_\lambda(\mat{K})^{-1}$ for some $\epsilon' > 0$ in \eqref{eqn:app-spectral-to-frob} ensures that $\| (\mat{S}_{\rm CtR} \mat{Z})^\top (\mat{S}_{\rm CtR} \mat{Z}) - \mat{Z}^\top \mat{Z} \|_2 \leq \epsilon'$ is satisfied with probability at least $1-\delta$, when $\mat{S}_{\rm CtR}$ has $2D s_\lambda(\mat{K})^2$ rows, where $D$ is the same as in Thm.~\ref{thrm:complex-norm-pres}.
\end{proof}

\newpage

\section{VARIANCE OF COMPLEX-TO-REAL SKETCHES}

\label{sec:app-ctr-variances}

In this section, we derive the variances of non-structured CtR-sketches.

\subsection{The structure of CtR variances}
\label{sec:ctr-structure}

We start by deriving the general variance structure of CtR-sketches that we will frequently refer to later on. For a complex random variable $z = a + \iu \, b$ with $a,b \in \mathbb{R}$, we have $|z|^2 = a^2 + b^2$ and ${\rm Re} \{z^2\} = a^2-b^2$. Combining both equations gives $a^2 = \frac{1}{2}(|z|^2 + {\rm Re} \{z^2\})$. The scalar $a$ is real-valued and its variance $\mathbb{V}[a] = \mathbb{E}[a^2] - \mathbb{E}[a]^2$ is thus
\begin{align}
    \label{eqn:var-expansion}
    \mathbb{V}[a]
    = \frac{1}{2}{\rm Re} \{\mathbb{E}[|z|^2] + \mathbb{E}[z^2] - 2 \mathbb{E}[a]^2\}.
\end{align}
Let $\Phi_{\rm C} : \mathbb{R}^d \rightarrow \mathbb{C}^D$ be a complex polynomial sketch as defined in Eq.~\ref{eqn:polynomial-estimator} and be $\hat{k}_{\rm C}(\mat{x}, \mat{y}) = \Phi_{\rm C}(\mat{x})^{\top} \overline{\Phi_{\rm C}(\mat{y})} \in \mathbb{C}$ the associated approximate kernel for some $\mat{x}, \mat{y} \in \mathbb{R}^d$. As the kernel estimate is an unbiased estimate of the real-valued target kernel $k(\mat{x}, \mat{y})$, we have
\begin{align*}
    \mathbb{E}[\hat{k}_{\rm C}(\mat{x}, \mat{y})] = \mathbb{E}[{\rm Re} \{\hat{k}_{\rm C}(\mat{x}, \mat{y})\}] + \iu \cdot \mathbb{E}[{\rm Im} \{\hat{k}_{\rm C}(\mat{x}, \mat{y})\}] = k(\mat{x}, \mat{y}).
\end{align*}
From this it follows that $\mathbb{E}[{\rm Im} \{\hat{k}_{\rm C}(\mat{x}, \mat{y})\}] = 0$ and therefore $\mathbb{E}[\hat{k}_{\rm C}(\mat{x}, \mat{y})] = \mathbb{E}[{\rm Re} \{\hat{k}_{\rm C}(\mat{x}, \mat{y})\}] = k(\mat{x}, \mat{y})$. Setting $z=\hat{k}_{\rm C}(\mat{x}, \mat{y})$ and $a = {\rm Re} \{\hat{k}_{\rm C}(\mat{x}, \mat{y})\} =: \hat{k}_{\rm CtR}(\mat{x}, \mat{y})$ in Eq.~\ref{eqn:var-expansion} yields
\begin{align*}
    \mathbb{V}[\hat{k}_{\rm CtR}(\mat{x}, \mat{y})]
    &= \frac{1}{2}{\rm Re} \{\mathbb{E}[|\hat{k}_{\rm C}(\mat{x}, \mat{y})|^2] + \mathbb{E}[\hat{k}_{\rm C}(\mat{x}, \mat{y})^2] - 2 \mathbb{E}[{\rm Re} \{\hat{k}_{\rm C}(\mat{x}, \mat{y})\}]^2\} \\
    &= \frac{1}{2}{\rm Re} \{\mathbb{E}[|\hat{k}_{\rm C}(\mat{x}, \mat{y})|^2] + \mathbb{E}[\hat{k}_{\rm C}(\mat{x}, \mat{y})^2] - 2 \mathbb{E}[\hat{k}_{\rm C}(\mat{x}, \mat{y})]^2\} \\
    &= \frac{1}{2} {\rm Re} \{ \mathbb{V} [\hat{k}_{\rm C}(\mat{x}, \mat{y})] + \mathbb{PV} [\hat{k}_{\rm C}(\mat{x}, \mat{y})] \},
\end{align*}
where $\mathbb{PV} [\hat{k}_{\rm C}(\mat{x}, \mat{y})] := \mathbb{E} [\hat{k}_{\rm C}(\mat{x}, \mat{y})^2] - \mathbb{E}[\hat{k}_{\rm C}(\mat{x}, \mat{y})] \in \mathbb{C}$ is called the {\em pseudo-variance} of $\hat{k}_{\rm C}(\mat{x}, \mat{y})$ \citep[][Chapter 5]{FundProb2018}. In fact, we show next that ${\rm Im } \{\mathbb{PV} [\hat{k}_{\rm C}(\mat{x}, \mat{y})]\} = 0$ for all the sketches discussed in this work. Hence, we can also write $\mathbb{V}[\hat{k}_{\rm CtR}(\mat{x}, \mat{y})] = \frac{1}{2} (\mathbb{V} [\hat{k}_{\rm C}(\mat{x}, \mat{y})] + \mathbb{PV} [\hat{k}_{\rm C}(\mat{x}, \mat{y})])$ for them since $\mathbb{V}[z] \in \mathbb{R}$ for any $z \in \mathbb{C}$.
In order to determine $\mathbb{V}[\hat{k}_{\rm CtR}(\mat{x}, \mat{y})]$, we thus work out $\mathbb{V} [\hat{k}_{\rm C}(\mat{x}, \mat{y})]$ and $\mathbb{PV} [\hat{k}_{\rm C}(\mat{x}, \mat{y})]$ for Gaussian, Rademacher and ProductSRHT sketches in the following.

\subsection{Gaussian and Rademacher sketches}

\label{sec:gaus-rad-variances}

In this section, we work out the variance of Gaussian and Rademacher CtR-sketches. For a set of $D$ i.i.d. random feature samples, we have
\begin{align}
    \label{eqn:iid-expansion}
    \mathbb{V}[\hat{k}_{\rm CtR}(\mat{x}, \mat{y})]
    = \mathbb{V}[{\rm Re}\{\hat{k}_{\rm C}(\mat{x}, \mat{y})\}]
    = \mathbb{V}\left[ {\rm Re}\{\Phi_{\rm C}(\mat{x})^{\top} \overline{\Phi_{\rm C}(\mat{y})}\} \right]
    = \frac{1}{D^2} \sum_{\ell=1}^D \mathbb{V} \left[{\rm Re} \left\{ \prod_{i=1}^p (\mat{w}_{i, \ell}^{\top} \mat{x}) \overline{(\mat{w}_{i, \ell}^{\top} \mat{y}}) \right\} \right].
\end{align}
As $\{\mat{w}_{i, \ell}\}_{\ell=1}^D$ are i.i.d., the variance terms are equal for each $\ell$ in Eq.~\ref{eqn:iid-expansion} and $\mathbb{V}[\hat{k}_{\rm CtR}(\mat{x}, \mat{y})] \propto 1/D$. We can therefore assume $D=1$ and drop the index $\ell$ for simplicity in the following. We then rescale the variances by $1/D$ later.

As our estimator is unbiased, we have $\mathbb{E}[\hat{k}_{\rm C}(\mat{x}, \mat{y})] = k(\mat{x}, \mat{y}) = (\dotprod{x}{y})^p$. Thus, we only need to work out $\mathbb{E}[|\hat{k}_{\rm C}(\mat{x}, \mat{y})|^2]$ and $\mathbb{E}[\hat{k}_{\rm C}(\mat{x}, \mat{y})^2]$ for the variance and pseudo-variance, respectively.

\paragraph{Pseudo-Variance}
We start with $\mathbb{E}[\hat{k}_{\rm C}(\mat{x}, \mat{y})^2]$ to derive the pseudo-variance $\mathbb{PV}[\hat{k}_{\rm C}(\mat{x}, \mat{y})]$ after.
\begin{align}
    \mathbb{E}[\hat{k}_{\rm C}(\mat{x}, \mat{y})^2]
    = \mathbb{E} \left[\left( \prod_{i=1}^p \dotprodi{w}{x}{i} \overline{\dotprodi{w}{y}{i}} \right)^2\right]
    &= \prod_{i=1}^p \mathbb{E} \left[ (\dotprodi{w}{x}{i})^2
    (\overline{\dotprodi{w}{y}{i}})^2 \right]
    = \mathbb{E} \left[ (\dotprod{w}{x})^2
    (\overline{\dotprod{w}{y}})^2 \right]^p \\
    &= \left( \sum_{i=1}^d \sum_{j=1}^d \sum_{k=1}^d \sum_{l=1}^d \mathbb{E} [w_i w_j \overline{w_k} \overline{w_l}] x_i x_j y_k y_l \right)^p
\end{align}
\newpage
$\mathbb{E}_{ij\overline{kl}} := \mathbb{E} [w_i w_j \overline{w_k} \overline{w_l}] \neq 0$, only if:
\begin{enumerate}
    \item $i=j=k=l$: there are $d$ terms $(\mathbb{E}_{ij\overline{kl}}) x_i x_j y_k y_l = \mathbb{E} [|w_i|^4] x_i^2 y_i^2$.
    \item $i=k \neq j=l$: there are $d(d-1)$ terms $(\mathbb{E}_{ij\overline{kl}}) x_i x_j y_k y_l = \mathbb{E}[|w_i|^2] x_i y_i \mathbb{E}[|w_j|^2] x_j y_j = x_i y_i x_j y_j$.
    \item $i=l \neq j=k$: there are $d(d-1)$ terms $(\mathbb{E}_{ij\overline{kl}}) x_i x_j y_k y_l = \mathbb{E}[|w_i|^2] x_i y_i \mathbb{E}[|w_j|^2] x_j y_j = x_i y_i x_j y_j$.
\end{enumerate}
As for both the Gaussian and the Rademacher sketch, we have $\mathbb{E}[|w_i|^2] = 1$ for all $\{w_i\}_{i=1}^d$, we obtain:
\begin{align}
    \label{eqn:second-pseudo-moment}
    \mathbb{E} \left[ \hat{k}_{\rm C}(\mat{x}, \mat{y})^2 \right]
    = \left( \sum_{i=1}^d \mathbb{E} [|w_i|^4] x_i^2 y_i^2
    + 2 \sum_{i=1}^d \sum_{j \neq i}^d x_i y_i x_j y_j \right)^p
\end{align}
We have $\mathbb{E}[|w_i|^4] = 2$ and $\mathbb{E}[|w_i|^4] = 1$ for the Gaussian and Rademacher case, respectively. So the pseudo-variances $\mathbb{V}[\hat{k}_{\rm C}(\mat{x}, \mat{y})] = \mathbb{E}[\hat{k}_{\rm C}(\mat{x}, \mat{y})^2] - \mathbb{E}[\hat{k}_{\rm C}(\mat{x}, \mat{y})]^2$ are given by the following real-valued expressions:
\begin{align}
    \label{eqn:gaussian-pseudo-var}
    &\mathbb{PV} [\hat{k}_{\rm C}(\mat{x}, \mat{y})]
    = \frac{1}{D} \left( \left(2 (\dotprod{x}{y})^2\right)^p - (\dotprod{x}{y})^{2p} \right) & \text{(Gaussian)} \\
    \label{eqn:rademacher-pseudo-var}
    &\mathbb{PV} [\hat{k}_{\rm C}(\mat{x}, \mat{y})]
    = \frac{1}{D} \left( \left( 2 (\dotprod{x}{y})^2 - \sum_{i=1}^d x_i^2 y_i^2 \right)^p - (\dotprod{x}{y})^{2p}\right) & \text{(Rademacher)}
\end{align}
where we added the $1/D$ scaling that we left out before. Note that $\mathbb{E}[ |w_i|^4 ] \geq ( \mathbb{E}[ |w_i|^2 ] )^2 = 1$ by Jensen's inequality, which is why the Rademacher sketch yields the lowest possible pseudo-variance for the estimator studied in Section \ref{sec:real-complex-sketches}.

\paragraph{Variance} We work out $\mathbb{E}[|\hat{k}_{\rm C}(\mat{x}, \mat{y})|^2]$ to derive the variance $\mathbb{V}[\hat{k}_{\rm C}(\mat{x}, \mat{y})]$.
\begin{align}
    \label{eq:complex-weights-expansion}
    \mathbb{E}[|\hat{k}_{\rm C}(\mat{x}, \mat{y})|^2]
    &= \mathbb{E} \left[\left| \prod_{i=1}^p \dotprodi{w}{x}{i} \overline{\dotprodi{w}{y}{i}} \right|^2\right]
    = \mathbb{E} \left[\prod_{i=1}^p | \dotprodi{w}{x}{i}|^2 |\overline{\dotprodi{w}{y}{i}}|^2 \right] \\
    \nonumber
    &= \mathbb{E}\left[  (\sum_{i=1}^d w_i x_i)  (\sum_{j=1}^d \overline{w_j} y_j) ( \sum_{k=1}^d\overline{w_k} x_k ) ( \sum_{l = 1}^d w_l y_l ) \right]^p
    = \left(\sum_{i=1}^d \sum_{j=1}^d \sum_{k=1}^d \sum_{l=1}^d \mathbb{E} [w_i\overline{w_j} \overline{w_k} w_l] x_i y_j x_k y_l\right)^p.
\end{align}
Now we check when $\mathbb{E}_{i\overline{jk}l} := \mathbb{E} \big[ w_i\overline{w_j}\overline{w_k} w_l \big] \neq 0$ holds. The analysis is the same as before with differently placed conjugates leading to different expressions.
\begin{enumerate}
    \item $i=j=k=l$: there are $d$ terms $(\mathbb{E}_{i\overline{jk}l}) x_i y_j x_k y_l = \mathbb{E}[| w_i |^4]  x_i^2 y_i^2$.

    \item $i=j \neq k=l$: there are $d(d-1)$ terms $(\mathbb{E}_{i\overline{jk}l}) x_i y_j x_k y_l = \mathbb{E}[ |w_i|^2] \mathbb{E}[|w_k|^2 ]  x_i x_k y_i y_k$.

    \item $i=k \neq j=l$, there are $d(d-1)$ terms $(\mathbb{E}_{i\overline{jk}l}) x_i y_j x_k y_l = \mathbb{E}[ |w_i|^2] \mathbb{E}[ |w_j|^2 ] x_i^2 y_j^2$.

    \item $i=l \neq j=k$, there are $d(d-1)$ terms $(\mathbb{E}_{i\overline{jk}l}) x_i y_j x_k y_l  = \mathbb{E}[ w_i^2  ] \mathbb{E}[\overline{w_j}^2  ] x_i x_j y_i y_j$.
\end{enumerate}
Therefore,
\begin{align*}
    \mathbb{E}[|\hat{k}_{\rm C}(\mat{x}, \mat{y})|^2]^{1/p}
    &= \sum_{i=1}^d \mathbb{E}[|w_i|^4] x_i^2 y_i^2 + \sum_{i=1}^d \sum_{\substack{j=1 \\ j \neq i}}^d x_i^2 y_j^2 +  \sum_{i=1}^d \sum_{\substack{j=1 \\ j \neq i}}^d x_i x_j y_i y_j + \sum_{i=1}^d \sum_{\substack{j=1 \\ j \neq i}}^d \mathbb{E} [w_i^2] \mathbb{E}[ \overline{w_j}^2] x_i x_j y_i y_j \\
    &= \sum_{i=1}^d \mathbb{E}[|w_i|^4] x_i^2 y_i^2 + \left[ \| \mat{x} \|^2 \| \mat{y} \|^2 - \sum_{i=1}^d x_i^2 y_i^2 \right] +  \left[ (\dotprod{x}{y})^2 - \sum_{i=1}^d x_i^2 y_i^2 \right]
    + \sum_{i=1}^d \sum_{\substack{j=1 \\ j \neq i}}^d \mathbb{E} [w_i^2] \mathbb{E}[ \overline{w_j}^2] x_i x_j y_i y_j 
\end{align*}
Once again, we have $\mathbb{E}[|w_i|^4]=2$ and $\mathbb{E}[|w_i|^4]=1$ for the Gaussian and Rademacher case, respectively. We further have $\mathbb{E} [w_i^2] = \mathbb{E} [\overline{w_i}^2] = \mathbb{E} [{\rm Re} \{w_i\}^2] - \mathbb{E} [{\rm Im} \{w_i\}^2]$ with $-1 \leq \mathbb{E} [w_i^2] \leq 1$ because $\mathbb{E} [|w_i|^2] = \mathbb{E} [{\rm Re} \{w_i\}^2] + \mathbb{E} [{\rm Im} \{w_i\}^2] = 1$. Thus, $\mathbb{E} [w_i^2] \mathbb{E}[ \overline{w_j}^2] = \mathbb{E} [w_i^2]^2 \in [0,1]$, where the extreme cases $0$ and $1$ are achieved by sampling $w_i$ from $\{1, 1, \iu, -\iu\}$ (complex Rademacher) and $\{1, -1\}$ (real Rademacher), respectively.
Therefore, we define the variable $q := (1+\mathbb{E} [w_i^2]^2)$ that equals 1 for the complex case and 2 for the real one.
We finally obtain the following variances $\mathbb{V}[\hat{k}_{\rm C}(\mat{x}, \mat{y})] = \mathbb{E}[|\hat{k}_{\rm C}(\mat{x}, \mat{y})|^2] - |\mathbb{E}[\hat{k}_{\rm C}(\mat{x}, \mat{y})]|^2$:
\begin{align}
    \label{eqn:gaussian-var}
    &\mathbb{V} [\hat{k}_{\rm C}(\mat{x}, \mat{y})]
    = \frac{1}{D} \left( \left( \norm{x}^2 \norm{y}^2 + q (\dotprod{x}{y})^2 \right)^p - (\dotprod{x}{y})^{2p} \right) & \text{(Gaussian)} \\
    \label{eqn:rademacher-var}
    &\mathbb{V} [\hat{k}_{\rm C}(\mat{x}, \mat{y})]
    = \frac{1}{D} \left( \left( \norm{x}^2 \norm{y}^2 + q \sum_{i=1}^d \sum_{j \neq i}^d x_i x_j y_i y_j \right)^p - (\dotprod{x}{y})^{2p} \right) & \text{(Rademacher)}
\end{align}
where we added the $1/D$ scaling that we left out before. Note also that $\mathbb{E}[ |w_i|^4 ] \geq ( \mathbb{E}[ |w_i|^2 ] )^2 = 1$ by Jensen's inequality, which is why the (real/complex) Rademacher sketch yields the lowest possible variance for the estimator studied in Section \ref{sec:real-complex-sketches}.

Thus, when $\sum_{i=1}^d \sum_{j \neq i}^d x_i x_j y_i y_j \geq 0$, sampling $w_i$ uniformly from $\{1, -1, \iu, -\iu \}$ yields the lowest possible CtR-variances
as both the variance as well as the pseudo-variance lower bound are attained. In the opposite case, real Rademacher sketches (sampling $w_i$ from $\{1, -1\}$) yield the lowest variances. This is because $\mathbb{E}[|\hat{k}_{\rm C}(\mat{x}, \mat{y})|^2] \geq 0$ is minimized in this case.

\subsection{Gaussian and Rademacher CtR Variance Advantage over their Real-Valued Analogs}

\label{sec:app-comparison-real-complex}

In the following, we compare Gaussian and Rademacher CtR-sketches against their real-valued analogs assuming that the corresponding feature maps have equal dimensions. Thus, we assign $D$ random features to the real feature map $\Phi_{\rm R}: \mathbb{R}^d \rightarrow \mathbb{R}^{D}$ and only $D/2$ random feature samples to the CtR feature map $\Phi_{\rm CtR}: \mathbb{R}^d \rightarrow \mathbb{R}^{D}$ (Alg.~\ref{alg:ctr-algorithm}) leading to the same output dimension $D$.

We call the corresponding kernel estimates $\hat{k}_{\rm R}(\mat{x}, \mat{y}) = \Phi_{\rm R}(\mat{x})^{\top} \Phi_{\rm R}(\mat{y})$ and $\hat{k}_{\rm CtR}(\mat{x}, \mat{y}) = \Phi_{\rm CtR}(\mat{x})^{\top} \Phi_{\rm CtR}(\mat{y})$.
$\mathbb{V}[\hat{k}_{\rm R}(\mat{x}, \mat{y})]$ is given in Eq.~\ref{eqn:gaussian-var} and \ref{eqn:rademacher-var}, where we set $q=2$.

We further have $\mathbb{V}[\hat{k}_{\rm CtR}(\mat{x}, \mat{y})] = \frac{1}{2} (\mathbb{V}[\hat{k}_{\rm C}(\mat{x}, \mat{y})] + \mathbb{PV}[\hat{k}_{\rm C}(\mat{x}, \mat{y})])$ as shown in Section \ref{sec:ctr-structure}. $\mathbb{V}[\hat{k}_{\rm C}(\mat{x}, \mat{y})]$ is given in Eq.~\ref{eqn:gaussian-var} and \ref{eqn:rademacher-var}, where we set $q=1$. $\mathbb{PV}[\hat{k}_{\rm C}(\mat{x}, \mat{y})]$ is given in Eq.~\ref{eqn:gaussian-pseudo-var} and \ref{eqn:rademacher-pseudo-var}, respectively.

We start with the simpler Gaussian case and study the Rademacher case after.

\subsubsection{Gaussian Case: Proof of Theorem \ref{thrm:gauss-ctr-advantage}.}

\label{sec:proof-gaus-ctr-advantage}

\begin{proof}
Taking into account that $\mat{S}_{\rm CtR}$ has only $D/2$ rows for $\Phi_{\rm R}$ and $\Phi_{\rm CtR}$ to have equal dimensions $D$, the variance difference of their kernel estimates yields:
\begin{align*}
    &\mathbb{V} [ \hat{k}_{\rm R} (\mat{x}, \mat{y}) ] - \mathbb{V} [ \hat{k}_{\rm CtR} (\mat{x}, \mat{y}) ] \\
    &= \frac{1}{D} \left( \left( \norm{x}^2 \norm{y}^2 + 2 (\dotprod{x}{y})^{2} \right)^p - (\dotprod{x}{y})^{2p} \right) - \frac{1}{D} \left( \left( \norm{x}^2 \norm{y}^2 + (\dotprod{x}{y})^2 \right)^p + \left(2 (\dotprod{x}{y})^2 \right)^p - 2 (\dotprod{x}{y})^{2p} \right) \\
    &= \frac{1}{D} \left( \left( \norm{x}^2 \norm{y}^2 + 2 (\dotprod{x}{y})^{2} \right)^p - \left(2 (\dotprod{x}{y})^2 \right)^p \right) - \frac{1}{D} \left( \left( \norm{x}^2 \norm{y}^2 + (\dotprod{x}{y})^2 \right)^p - (\dotprod{x}{y})^{2p} \right) \\
    &= \frac{1}{D} \sum_{k=0}^{p-1} \binom{p}{k} \left(2 (\dotprod{x}{y})^2 \right)^k \left( \norm{x}^2 \norm{y}^2 \right)^{p-k}
    - \frac{1}{D} \sum_{k=0}^{p-1} \binom{p}{k} (\dotprod{x}{y})^{2k} \left( \norm{x}^2 \norm{y}^2 \right)^{p-k} \\
    &= \frac{1}{D} \sum_{k=0}^{p-1} \binom{p}{k} (2^k - 1) (\dotprod{x}{y})^{2k} \left( \norm{x}^2 \norm{y}^2 \right)^{p-k}
    \geq 0
\end{align*}
\end{proof}
Thus, the Gaussian CtR-estimator is always better regardless of the choice of $\mat{x}, \mat{y}$ and $p$ and despite using only half the random feature samples. Note that the variance difference is zero if $p=1$ and increases as $p$ increases. Moreover, the difference is maximized for parallel $\mat{x}$ and $\mat{y}$. In this case, we have $(\dotprod{x}{y}) = \norm{x}\norm{y}$ and the difference becomes
\begin{align*}
    \mathbb{V} [ \hat{k}_{\rm R} (\mat{x}, \mat{y}) ] - \mathbb{V} [ \hat{k}_{\rm CtR} (\mat{x}, \mat{y}) ]
    = \frac{1}{D} \sum_{k=0}^{p-1} \binom{p}{k} (2^k - 1) \left( \norm{x}^2 \norm{y}^2 \right)^{p}
    = \frac{1}{D} \norm{x}^{2p} \norm{y}^{2p} (3^p - 2^{p+1} + 1)
\end{align*}
We analyze the more difficult Rademacher case next.

\subsubsection{Rademacher Case: Proof of Theorem \ref{thrm:rad-ctr-advantage}.}

\label{sec:proof-rad-ctr-advantage}

\begin{proof}
Taking into account that $\mat{S}_{\rm CtR}$ has only $D/2$ rows for $\Phi_{\rm R}$ and $\Phi_{\rm CtR}$ to have equal dimensions $D$, the variance difference of their kernel estimates yields:
\begin{align*}
    &\mathbb{V} [ \hat{k}_{\rm R} (\mat{x}, \mat{y}) ] - \mathbb{V} [ \hat{k}_{\rm CtR} (\mat{x}, \mat{y}) ] \\
    &= \frac{1}{D} \left( \left( \norm{x}^2 \norm{y}^2 + 2 \sum_{i=1}^d \sum_{j \neq i} x_i x_j y_i y_j \right)^p - (\dotprod{x}{y})^{2p} \right) \\
    &\quad - \frac{1}{D} \left\{ \left(\norm{x}^2 \norm{y}^2 + \sum_{i=1}^d \sum_{j \neq i} x_i x_j y_i y_j \right)^p - (\dotprod{x}{y})^{2p}
    + \left( 2 (\dotprod{x}{y})^2 - \sum_{i=1}^d x_i^2 y_i^2 \right)^p - (\dotprod{x}{y})^{2p} \right\}
\end{align*}
Next, we write $(\dotprod{x}{y})^{2p} = ((\dotprod{x}{y})^{2} - \sum_{i=1}^d x_i^2 y_i^2 + \sum_{i=1}^d x_i^2 y_i^2)^p$. In this way, we can factor out the term $a := (\dotprod{x}{y})^{2} - \sum_{i=1}^d x_i^2 y_i^2 = \sum_{i=1}^d \sum_{j \neq i} x_i x_j y_i y_j$ and apply the binomial theorem to {\em all} addends. This gives:
\begin{align*}
    &\mathbb{V} [ \hat{k}_{\rm R} (\mat{x}, \mat{y}) ] - \mathbb{V} [ \hat{k}_{\rm CtR} (\mat{x}, \mat{y}) ]
    = \frac{1}{2D} \sum_{k=0}^p \binom{p}{k} a^{p-k} \\
    &\quad \left(\left( \norm{x}^2 \norm{y}^2 + (\dotprod{x}{y})^{2} - \sum_{i=1}^d x_i^2 y_i^2 \right)^k - \left((\norm{x}^2 \norm{y}^2)^k + (\dotprod{x}{y})^{2k} - (\sum_{i=1}^d x_i^2 y_i^2)^k \right) \right)
\end{align*}
We now show that the following term is always non-negative:
\begin{align}
    B &:= \left(\left( \norm{x}^2 \norm{y}^2 + (\dotprod{x}{y})^{2} - \sum_{i=1}^d x_i^2 y_i^2 \right)^k - \left((\norm{x}^2 \norm{y}^2)^k + (\dotprod{x}{y})^{2k} - (\sum_{i=1}^d x_i^2 y_i^2)^k \right) \right)
\end{align}
For $k=0$ and $k=1$, $B=0$. For $k \geq 2$, we have:
\begin{align*}
    \left( \norm{x}^2 \norm{y}^2 + (\dotprod{x}{y})^{2} - \sum_{i=1}^d x_i^2 y_i^2 \right)^k
    = \sum_{j=0}^k \binom{k}{j} \norm{x}^{2j} \norm{y}^{2j} \left((\dotprod{x}{y})^2 - \sum_{i=1}^d x_i^2 y_i^2\right)^{k-j}
\end{align*}
Plugging this expression into $B$ and cancelling out the addend for $j=k$ yields:
\begin{align*}
    B = \sum_{j=0}^{k-1} \binom{k}{j} \norm{x}^{2j} \norm{y}^{2j} \left((\dotprod{x}{y})^2 - \sum_{i=1}^d x_i^2 y_i^2\right)^{k-j} - \left((\dotprod{x}{y})^{2k} - (\sum_{i=1}^d x_i^2 y_i^2)^k \right)
\end{align*}
Next, we refactor $(\dotprod{x}{y})^{2k} - (\sum_{i=1}^d x_i^2 y_i^2)^k$:
\begin{align*}
    (\dotprod{x}{y})^{2k} - (\sum_{i=1}^d x_i^2 y_i^2)^k
    &= ((\dotprod{x}{y})^2 - \sum_{i=1}^d x_i^2 y_i^2 + \sum_{i=1}^d x_i^2 y_i^2)^k - (\sum_{i=1}^d x_i^2 y_i^2)^k \\
    &= \sum_{j=0}^k \binom{k}{j} (\sum_{i=1}^d x_i^2 y_i^2)^j ((\dotprod{x}{y})^2 - \sum_{i=1}^d x_i^2 y_i^2)^{k-j} - (\sum_{i=1}^d x_i^2 y_i^2)^k \\
    &= \sum_{j=0}^{k-1} \binom{k}{j} (\sum_{i=1}^d x_i^2 y_i^2)^j ((\dotprod{x}{y})^2 - \sum_{i=1}^d x_i^2 y_i^2)^{k-j}
\end{align*}
Plugging this expression into $B$ yields:
\begin{align*}
    B = \sum_{j=0}^{k-1} \binom{k}{j} \left(\norm{x}^{2j} \norm{y}^{2j} - (\sum_{i=1}^d x_i^2 y_i^2)^j\right)
    ((\dotprod{x}{y})^2 - \sum_{i=1}^d x_i^2 y_i^2)^{k-j}
\end{align*}
Finally, we insert $B$ back into the original variance difference $\mathbb{V} [ \hat{k}_{\rm R} (\mat{x}, \mat{y}) ] - \mathbb{V} [ \hat{k}_{\rm CtR} (\mat{x}, \mat{y}) ]$ (remember that $B=0$ if $k<2$):
\begin{align*}
    \mathbb{V} [ \hat{k}_{\rm R} (\mat{x}, \mat{y}) ] - \mathbb{V} [ \hat{k}_{\rm CtR} (\mat{x}, \mat{y}) ]
    = \frac{1}{D} \sum_{k=2}^p \binom{p}{k} a^{p-k} B
    = \frac{1}{D} \sum_{k=2}^p \sum_{j=0}^{k-1} \binom{p}{k} \binom{k}{j} a^{p-j} \left(\norm{x}^{2j} \norm{y}^{2j} - (\sum_{i=1}^d x_i^2 y_i^2)^j\right)
\end{align*}
Finally, we note that $b_j := \norm{x}^{2j} \norm{y}^{2j} - (\sum_{i=1}^d x_i^2 y_i^2)^j = (\sum_{i=1}^d \sum_{\ell=1}^d x_i^2 y_\ell^2)^j - (\sum_{i=1}^d x_i^2 y_i^2)^j \geq 0$ and $\mathbb{V} [ \hat{k}_{\rm R} (\mat{x}, \mat{y}) ] - \mathbb{V} [ \hat{k}_{\rm CtR} (\mat{x}, \mat{y}) ] \geq 0$ if $a = \sum_{i=1}^d \sum_{j' \neq i}^d x_i x_{j'} y_i y_{j'} \geq 0$. \end{proof}

\section{VARIANCE OF OUR PROPOSED (CtR-)ProductSRHT SKETCH}

\label{sec:app-structured-variances}

In Section~\ref{sec:app-structured-sketches}, we proposed a novel (CtR-) ProductSRHT sketch that is a slightly modified version of the TensorSRHT sketch proposed by \citet{Ahle2020}. Unlike previous work, we derive the variance of (CtR-)ProductSRHT and show its statistical advantage over unstructured sketches. The statistical advantage stems from the orthogonality of $\mat{H}$ as well as from the sampling matrices $\{ \mat{P}_i \}_{i=1}^p$ that sample \textit{without replacement}. The statistical advantage is lost when sampling \textit{with replacement} as is done in \citet{Ahle2020}. In this case, the variance falls back to the Rademacher variance.

\subsection{Variances of ProductSRHT as well as CtR-ProductSRHT}

\label{sec:structured-variance-derivation}

As shown in Section \ref{sec:ctr-structure}, the variance of the CtR sketches discussed in this work is of the form:
\begin{align*}
    \mathbb{V}[\hat{k}_{\rm CtR}(\mat{x}, \mat{y})] = \frac{1}{2} (\mathbb{V} [\hat{k}_{\rm C}(\mat{x}, \mat{y})] + \mathbb{PV} [\hat{k}_{\rm C}(\mat{x}, \mat{y})]),
\end{align*}
where $\hat{k}_{\rm C}(\mat{x}, \mat{y})$ is the complex-valued kernel estimate of the polynomial kernel obtained through our sketch. In order to derive the variance of CtR-ProductSRHT, we need to derive the variance $\mathbb{V} [\hat{k}_{\rm C}(\mat{x}, \mat{y})]$ and the pseudo-variance $\mathbb{PV} [\hat{k}_{\rm C}(\mat{x}, \mat{y})]$. We will also derive the variance of real-valued ProductSRHT as a corollary of the variance of complex ProductSRHT.

\subsubsection{Pseudo-variance}

As before, we start with the pseudo-variance and derive the variance after. For the pseudo-variance $\mathbb{PV}[\hat{k}_{\rm C}(\mat{x}, \mat{y})] = \mathbb{E} [\hat{k}_{\rm C}(\mat{x}, \mat{y})^2] - \mathbb{E}[\hat{k}_{\rm C}(\mat{x}, \mat{y})]^2$, we need to work out $\mathbb{E} [\hat{k}_{\rm C}(\mat{x}, \mat{y})^2]$:
\begin{align*}
    \mathbb{E} [\hat{k}_{\rm C}(\mat{x}, \mat{y})^2]
    = \frac{1}{D^2} \sum_{\ell=1}^D \sum_{\ell'=1}^D \prod_{i=1}^p \mathbb{E} \left[ (\dotprodi{w}{x}{i, \ell}) (\overline{\dotprodi{w}{y}{i, \ell}}) (\dotprodi{w}{x}{i, \ell'}) (\overline{\dotprodi{w}{y}{i, \ell'}}) \right]
    = \frac{1}{D^2} \sum_{\ell=1}^D \sum_{\ell'=1}^D \underbrace{\mathbb{E} \left[ (\dotprodi{w}{x}{\ell}) (\overline{\dotprodi{w}{y}{\ell}}) (\dotprodi{w}{x}{\ell'}) (\overline{\dotprodi{w}{y}{\ell'}}) \right]^p}_{e(\ell, \ell')^p}
\end{align*}
We dropped the index $i$ in the last equality for ease of notation, as all $\{\mat{w}_{i, \ell}\}_{i=1}^p$ are i.i.d. samples and the expectation is thus the same for any $i$.
To work out the expectation $e(\ell, \ell')$, we need to distinguish different cases for $\ell$ and $\ell'$.
\begin{enumerate}
    \item $\ell=\ell'$ ($D$ terms): $e(\ell, \ell')^p = \mathbb{E} \left[ (\dotprodi{w}{x}{\ell})^2 (\overline{\dotprodi{w}{y}{\ell}})^2 \right]^p = \left( 2 (\dotprod{x}{y}^2) - \sum_{i=1}^d x_i^2 y_i^2 \right)^p$ \\
    (taken from Eq.~\ref{eqn:second-pseudo-moment} for the Rademacher case)
    \item $\ell \neq \ell'$ ($D(D-1)$ terms):\\
    \begin{align*}
        e(\ell, \ell')^p
        &= \left( \sum_{q=1}^d \sum_{r=1}^d \sum_{s=1}^d \sum_{t=1}^d \mathbb{E} [w_{\ell, q} \overline{w_{\ell, r}} w_{\ell', s} \overline{w_{\ell', t}}] x_q y_r x_s y_t \right)^p \\
        &= \left( \sum_{q=1}^d \sum_{r=1}^d \sum_{s=1}^d \sum_{t=1}^d \mathbb{E} [d_{q} \overline{d_{r}} d_{s} \overline{d_{t}}] \mathbb{E} [h_{p_\ell, q} h_{p_\ell, r} h_{p_{\ell'}, s} h_{p_{\ell'}, t}] x_q y_r x_s y_t \right)^p
    \end{align*}
\end{enumerate}
$d_q, d_r, d_s, d_t$ are uniform samples from $\{1, -1, \iu, -\iu\}$, i.e., complex Rademacher samples, that are independent from the index samples $p_{\ell, q}, p_{\ell, r}, p_{\ell', s}, p_{\ell', t}$, which is why we can factor out the two expectations. We will simplify the above sum by studying when $\mathbb{E}[d_q \overline{d_r} d_s \overline{d_t}] \neq 0$.

We have to distinguish three non-zero cases for $\mathbb{E}[d_q \overline{d_r} d_s \overline{d_t}]$:
\begin{enumerate}
    \item $q=r=s=t$ ($d$ terms): $\mathbb{E}[d_q \overline{d_r} d_s \overline{d_t}] = \mathbb{E} [|d_q|^4] = 1$
    \item $q=r \neq s=t$ ($d(d-1)$ terms): $\mathbb{E}[d_q \overline{d_r} d_s \overline{d_t}] = \mathbb{E} [|d_{q}|^2] \mathbb{E} [|d_{s}|^2] = 1$
    \item $q=t \neq r=s$ ($d(d-1)$ terms):
    $\mathbb{E}[d_q \overline{d_r} d_s \overline{d_t}] = \mathbb{E} [|d_q|^2] \mathbb{E} [|d_r|^2] = 1$
\end{enumerate}
because $\mathbb{E} [|d_q|^4] = \mathbb{E} [|d_q|^2] = 1$.

In Section \ref{sec:shuffling-hadamard}, we show that for $\ell \neq \ell'$ and $q \neq r$, $\mathbb{E} [h_{p_\ell, q} h_{p_\ell, r} h_{p_{\ell'}, r} h_{p_{\ell'}, q}] = -\frac{1}{\lceil D/d \rceil d - 1}$ holds. Therefore, $e(\ell, \ell')^p$ for $\ell \neq \ell'$ yields:
\begin{align*}
    e(\ell, \ell')^p
    &= \left(\sum_{i=1}^d x_i^2 y_i^2 + \sum_{i=1}^d \sum_{j \neq i}^d x_i y_i x_j y_j - \frac{1}{\lceil D/d \rceil d - 1} \sum_{i=1}^d \sum_{j \neq i}^d x_i y_i x_j y_j\right)^p \\
    &= \left( (\dotprod{x}{y})^2 - \frac{1}{\lceil D/d \rceil d - 1} \left[ (\dotprod{x}{y})^2 - \sum_{i=1}^d x_i^2 y_i^2 \right] \right)^p
\end{align*}
In fact, $e(\ell, \ell')^p$ does not depend on $\ell$ and $\ell'$ anymore after working out the expectations involved. Plugging $e(\ell, \ell')^p$ back into $\mathbb{E}[\hat{k}_{\rm C}(\mat{x}, \mat{y})^2]$ yields the following pseudo-variance for ProductSRHT:
\begin{align}
    \nonumber
    \mathbb{PV}[\hat{k}_{\rm C}(\mat{x}, \mat{y})^2]
    &= \frac{1}{D} \left[ \left( 2 (\dotprod{x}{y}^2) - \sum_{i=1}^d x_i^2 y_i^2 \right)^p - (\dotprod{x}{y})^{2p} \right] \\
    \nonumber
    &\quad\quad\quad + \left(1 - \frac{1}{D} \right) \left[ \left( (\dotprod{x}{y})^2 - \frac{1}{\lceil D/d \rceil d - 1} \left[ (\dotprod{x}{y})^2 - \sum_{i=1}^d x_i^2 y_i^2 \right] \right)^p - (\dotprod{x}{y})^{2p} \right] \\
    \label{eqn:pvar-tensor-srht}
    &= \frac{1}{D} \mathbb{PV}_{\rm Rad.}^{(p)} - \left(1 - \frac{1}{D}\right) \left[ (\dotprod{x}{y})^{2p} - \left((\dotprod{x}{y})^2 - \frac{\mathbb{PV}_{\rm Rad.}^{(1)}}{\lceil D / d \rceil d - 1}\right)^p \right]
\end{align}
$\mathbb{PV}_{\rm Rad.}^{(p)}$ and $\mathbb{PV}_{\rm Rad.}^{(1)}$ are the Rademacher pseudo-variance (\ref{eqn:rademacher-pseudo-var}) for a given degree $p$ and $p=1$, respectively.

\subsubsection{Variance}

\label{sec:tensor-srht-var-derivation}

Next we work out the variance $\mathbb{V} [\hat{k}_{\rm C}(\mat{x}, \mat{y})]$:
\begin{align*}
    \mathbb{V} \left[ \frac{1}{D} \sum_{\ell=1}^D \prod_{i=1}^p (\mat{w}_{i,\ell}^{\top} \mat{x}) (\overline{\mat{w}_{i,\ell}^{\top} \mat{y}}) \right]
    &= \frac{1}{D^2} \sum_{\ell=1}^D \sum_{\ell'=1}^D {\rm Cov} \left( \prod_{i=1}^p (\dotprodi{w}{x}{i, \ell}) (\overline{\dotprodi{w}{y}{i, \ell}}), \prod_{i=1}^p (\dotprodi{w}{x}{i, \ell'}) (\overline{\dotprodi{w}{y}{i, \ell'}}) \right)
\end{align*}
Again, we distinguish the cases $\ell = \ell'$ and $\ell \neq \ell'$:
\begin{enumerate}
    \item $\ell = \ell'$ ($D$ terms):
    \begin{align*}
        &{\rm Cov} \left( \prod_{i=1}^p (\dotprodi{w}{x}{i, \ell}) (\overline{\dotprodi{w}{y}{i, \ell}}), \prod_{i=1}^p (\dotprodi{w}{x}{i, \ell}) (\overline{\dotprodi{w}{y}{i, \ell}}) \right)
        = \mathbb{V} \left[ \prod_{i=1}^p (\dotprodi{w}{x}{i, \ell}) (\overline{\dotprodi{w}{y}{i, \ell}}) \right] \\
        &= \left( \norm{x}^2 \norm{y}^2 + (\dotprod{x}{y})^{2} - \sum_{i=1}^d x_i^2 y_i^2\right)^p - (\dotprod{x}{y})^{2p}
        \quad \text{(Using the complex Rademacher variance (\ref{eqn:rademacher-var}))}
    \end{align*}
    \item $\ell \neq \ell'$ ($D(D-1)$ terms). We discuss this case in detail below.
\end{enumerate}
\begin{align}
    \label{eqn:cov_ell_ellp}
    &{\rm Cov} \left( \prod_{i=1}^p (\dotprodi{w}{x}{i, \ell}) (\overline{\dotprodi{w}{y}{i, \ell}}), \prod_{i=1}^p (\dotprodi{w}{x}{i, \ell'}) (\overline{\dotprodi{w}{y}{i, \ell'}}) \right)
    = \mathbb{E} \left[ \prod_{i=1}^p (\dotprodi{w}{x}{i, \ell}) (\overline{\dotprodi{w}{y}{i, \ell}}) (\overline{\dotprodi{w}{x}{i, \ell'}) (\overline{\dotprodi{w}{y}{i, \ell'}}}) \right] - (\dotprod{x}{y})^{2p} \\
    \nonumber
    &= \mathbb{E} \left[ (\dotprodi{w}{x}{\ell}) (\overline{\dotprodi{w}{y}{\ell}}) (\overline{\dotprodi{w}{x}{\ell'}) (\overline{\dotprodi{w}{y}{\ell'}})} \right]^p - (\dotprod{x}{y})^{2p}
    = \underbrace{\mathbb{E} \left[ (\dotprodi{w}{x}{\ell}) (\overline{\dotprodi{w}{y}{\ell}}) (\overline{\dotprodi{w}{x}{\ell'}}) (\dotprodi{w}{y}{\ell'}) \right]^p}_{e_2(\ell, \ell')^p} - (\dotprod{x}{y})^{2p}
\end{align}
Next, we turn to the expression $e_2(\ell, \ell')^p$ that is almost the same as $e(\ell, \ell')^p$ for the pseudo-variance, the only difference being the complex conjugates that are placed differently:
\begin{align*}
    e_2(\ell, \ell')^p
    &= \left( \sum_{q=1}^d \sum_{r=1}^d \sum_{s=1}^d \sum_{t=1}^d \mathbb{E} [w_{\ell, q} \overline{w_{\ell, r}} \overline{w_{\ell', s}} w_{\ell', t}] x_q y_r x_s y_t \right)^p \\
    &= \left( \sum_{q=1}^d \sum_{r=1}^d \sum_{s=1}^d \sum_{t=1}^d \mathbb{E} [d_{q} \overline{d_{r} d_{s}} d_{t}] \mathbb{E} [h_{p_\ell, q} h_{p_\ell, r} h_{p_{\ell'}, s} h_{p_{\ell'}, t}] x_q y_r x_s y_t \right)^p
\end{align*}
We distinguish 4 cases for $\mathbb{E} [d_q \overline{d_r} \overline{d_s} d_t]$:
\begin{enumerate}
    \item $q=r=s=t$ ($d$ terms): $\mathbb{E} [d_q \overline{d_r} \overline{d_s} d_t] = \mathbb{E} [|d_q|^4] = 1$
    \item $q=r \neq s=t$ ($d(d-1)$ terms): $\mathbb{E} [d_q \overline{d_r} \overline{d_s} d_t] = \mathbb{E} [|d_{q}|^2] \mathbb{E} [|d_{s}|^2] = \mathbb{E} [|d_{q}|^2]^2 = 1$
    \item $q=s \neq r = t$ ($d(d-1)$ terms): $\mathbb{E} [d_{q} \overline{d_{r}} \overline{d_{s}} d_{t}] = \mathbb{E} [|d_{q}|^2] \mathbb{E} [|d_{r}|^2] = \mathbb{E} [|d_{q}|^2]^2 = 1$
    \item $q=t \neq r = s$ ($d(d-1)$ terms): $\mathbb{E} [d_{q} \overline{d_{r}} \overline{d_{s}} d_{t}] = \mathbb{E} [d_{q}^2] \mathbb{E} [\overline{d_{r}}^2] = 0$
\end{enumerate}
We showed case (4) on purpose although it is zero for complex Rademacher samples $d_q, d_r \in \mathbb{C}$. For real Rademacher samples, we have $\mathbb{E} [d_{q}^2] = \mathbb{E} [\overline{d_{r}}^2] = 1$ instead. This observation will allow us to work out the variance of complex and real ProductSRHT at the same time. Furthermore, we have $\mathbb{E} [h_{p_\ell, q} h_{p_\ell, r} h_{p_{\ell'}, r} h_{p_{\ell'}, q}] = -\frac{1}{\lceil D / d \rceil d}$ for any $q \neq r$ and $\ell \neq \ell'$ as already noted for the pseudo-variance. The derivation of this quantity is shown in Section \ref{sec:shuffling-hadamard}.

So $e_2(\ell, \ell')$ reduces to:
\begin{align*}
    e_2(\ell, \ell')
    &= \underbrace{\sum_{i=1}^d x_i^2 y_i^2}_{\textrm{Case (1)}}
    + \underbrace{\sum_{i=1}^d \sum_{j \neq i}^d x_i x_j y_i y_j}_{\textrm{Case (2)}}
    - \underbrace{\frac{1}{\lceil D / d \rceil d} \sum_{i=1}^d \sum_{j \neq i}^d x_i^2 y_j^2}_{\textrm{Case (3)}}
    - \underbrace{\frac{1}{\lceil D / d \rceil d} \sum_{i=1}^d \sum_{j \neq i}^d \mathbb{E} [d_{i}^2] \mathbb{E} [\overline{d_{j}}^2] x_i x_j y_i y_j}_{\textrm{Case (4)}} \\
    &= (\dotprod{x}{y})^2 - \frac{1}{\lceil D / d \rceil d} \sum_{i=1}^d \sum_{j \neq i}^d x_i^2 y_j^2 + \mathbb{E} [d_{i}^2] x_i x_j y_i y_j,
\end{align*}
where $\mathbb{E}[d_i^2] = 0$ for the complex case and $\mathbb{E}[d_i^2] = 1$ for the real case. Plugging back $e_2(\ell, \ell')$ for the case $\ell \neq \ell'$ back into Eq.~\ref{eqn:cov_ell_ellp} and solving for $\mathbb{V}[\hat{k}_{\rm C}(\mat{x}, \mat{y})]$ yields:
\begin{align}
    \label{eqn:var-tensor-srht}
    \mathbb{V}[\hat{k}_{\rm C}(\mat{x}, \mat{y})]
    = \mathbb{V}_{\rm Rad.}^{(p)}
    - \left(1 - \frac{1}{D}\right) \left[(\dotprod{x}{y})^{2p} - \left( (\dotprod{x}{y})^2 - \frac{\mathbb{V}_{\rm Rad.}^{(1)}}{\lceil D/d \rceil d-1} \right)^p \right]
\end{align}
with $\mathbb{V}_{\rm Rad.}^{(p)}$ and $\mathbb{V}_{\rm Rad.}^{(1)}$ being the Rademacher variance (\ref{eqn:rademacher-var}) for a given degree $p$ and $p=1$, respectively. We set $q=2$ for the real case and $q=1$ for the complex case inside Eq.~\ref{eqn:rademacher-var}.

Inserting the expressions for the variance (\ref{eqn:var-tensor-srht}) and pseudo-variance (\ref{eqn:pvar-tensor-srht}) into Eq.~\ref{eqn:ctr-var}, gives the variance of CtR-ProductSRHT.

\subsubsection{Shuffling the Rows of Stacked Hadamard Matrices}

\label{sec:shuffling-hadamard}

In this section, we prove an important equality that was used in the derivation of the variance formulas of ProductSRHT in the previous sections. It can be seen as the key lemma that leads to a reduced variance compared to Rademacher sketches. It shows the statistics of randomly sampled rows (without replacement) inside stacked orthogonal Hadamard matrices that give close-to-orthogonal as opposed to i.i.d. samples in our proposed ProductSRHT sketch. We prove the equality
\begin{align*}
    \mathbb{E} [h_{p_\ell, q} h_{p_\ell, r} h_{p_{\ell'}, r} h_{p_{\ell'}, q}]
    = -\frac{1}{\lceil D/d \rceil d - 1}
\end{align*}
for $\ell \neq \ell'$ and $q \neq r$ being fixed indices. $\mat{h}_{p_\ell}^{\top}$ and $\mat{h}_{p_{\ell'}}^{\top}$ are the $p_\ell$-th and $p_{\ell'}$-th row of the Hadamard matrix $\mat{H}$, respectively (see Section~\ref{sec:app-structured-sketches}). The indices $q$ and $r$ refer to elements inside these row vectors. $p_\ell$ and $p_{\ell'}$ are themselves the $\pi(\ell)$-th and $\pi(\ell')$-th entries of the vector $\mat{p}_i \in \mathbb{R}^{\lceil D/d \rceil d}$ for a given $i \in \{1, \dots, p\}$. Here, we look at a given index $i$ and drop the index for ease of presentation. We do the same for the permutation function $\pi(\cdot)$. Recall that $\{\mat{p}_i\}_{i=1}^p$ is used to construct the sampling matrices $\{\mat{P}_i\}_{i=1}^p$ in Alg.~\ref{alg:product-srht-algorithm}.

The following proof is closely related to \citet[][Proof of Proposition 8.2]{Choromanski2017} and \citet[][Lemma B.1]{wacker2022}. The difference here is that we consider the sampling of rows (without replacement) inside {\em stacked} Hadamard matrices as we will see next, whereas the other works only consider the sampling of rows inside a {\em single} Hadamard matrix.

\begin{proof}

The sampling procedure for the rows $\mat{h}_{p_\ell}^\top$ and $\mat{h}_{p_{\ell'}}^\top$ can be described as follows. We stack the Hadamard matrix $\mat{H} \in \mathbb{R}^{d \times d}$ $\lceil D/d \rceil$ times on top of itself to yield a new matrix $\mat{H}^{\lceil D/d \rceil} \in \mathbb{R}^{\lceil D/d \rceil d \times d}$. We then shuffle its rows randomly to yield the shuffled matrix $\mat{H}_{\mat{p}}^{\lceil D/d \rceil \times d}$. $\mat{h}_{p_{\ell}}^{\top}$ and $\mat{h}_{p_{\ell'}}^{\top}$ are then the $\ell$-th and $\ell'$-th row of $\mat{H}_{\mat{p}}^{\lceil D/d \rceil}$. In fact, the shuffled matrix $\mat{H}_{\mat{p}}^{\lceil D/d \rceil}$ can be constructed from the index vector $\mat{p}$ that contains the order of the rows of $\mat{H}$ to be used.

Since the columns of $\mat{H}$ are orthogonal, the same is true for $\mat{H}^{\lceil D/d \rceil}$ and $\mat{H}_{\mat{p}}^{\lceil D/d \rceil}$. So the inner product of two distinct columns $q$ and $r$ of $\mat{H}_{\mat{p}}^{\lceil D/d \rceil}$ yields $\sum_{\ell=1}^{\lceil D/d \rceil d} h_{p_\ell, q} h_{p_\ell, r} = 0$. As $h_{p_\ell, q}, h_{p_\ell, r} \in \{1, -1\}$, half of $\{ h_{p_\ell, q} h_{p_\ell, r} \}_{\ell=1}^{\lceil D/d \rceil d}$ must be equal to $1$ and $-1$, respectively.
From this we get the marginal probabilities
\begin{align*}
    {\rm Pr}(h_{p_\ell, q} h_{p_\ell, r} = 1) = {\rm Pr}(h_{p_\ell, q} h_{p_\ell, r} = -1) = 0.5
\end{align*}
for any $q \neq r$ being fixed, where the probabilities are taken over the indices $p_{\ell}$ and $p_{\ell}$, i.e., the shuffling operation. Next, we obtain the following conditional probabilities using the same logic as before:
\begin{align*}
    &{\rm Pr}(h_{\ell', q} h_{\ell', r} = 1 | h_{\ell, q} h_{\ell, r} = 1)
    = {\rm Pr}(h_{\ell', q} h_{\ell', r} = -1 | h_{\ell, q} h_{\ell, r} = -1)
    = \frac{(\lceil D/d \rceil d) / 2 - 1}{\lceil D/d \rceil d - 1} \\
    &{\rm Pr}(h_{\ell', q} h_{\ell', r} = 1 | h_{\ell, q} h_{\ell, r} = -1)
    = {\rm Pr}(h_{\ell', q} h_{\ell', r} = -1 | h_{\ell, q} h_{\ell, r} = 1)
    = \frac{(\lceil D/d \rceil d) / 2}{\lceil D/d \rceil d - 1}
\end{align*}
Using these conditional probabilities along with the marginal probabilities ${\rm Pr}(h_{p_\ell, q} h_{p_\ell, r})$ allows us to solve $\mathbb{E} [h_{p_\ell, q} h_{p_\ell, r} h_{p_{\ell'}, r} h_{p_{\ell'}, q}]$ via the law of total expectation:
\begin{align*}
    \mathbb{E} [h_{p_\ell, q} h_{p_\ell, r} h_{p_{\ell'}, r} h_{p_{\ell'}, q}]
    &= \mathbb{E}_{p_{\ell}} [\mathbb{E}_{p_{\ell'}} [h_{p_\ell, q} h_{p_\ell, r} h_{p_{\ell'}, r} h_{p_{\ell'}, q} | h_{p_{\ell}, r} h_{p_{\ell}, q}]] \\
    &= \frac{1}{2} \left(\mathbb{E}_{p_{\ell'}} [h_{p_{\ell'}, r} h_{p_{\ell'}, q} | h_{p_{\ell}, r} h_{p_{\ell}, q} = 1] - \mathbb{E}_{p_{\ell'}} [h_{p_{\ell'}, r} h_{p_{\ell'}, q} | h_{p_{\ell}, r} h_{p_{\ell}, q} = -1] \right) \\
    &= \frac{1}{2} \left( \left(\frac{(\lceil D/d \rceil d) / 2 - 1}{\lceil D/d \rceil d - 1} - \frac{(\lceil D/d \rceil d) / 2}{\lceil D/d \rceil d - 1}\right)
    - \left( \frac{(\lceil D/d \rceil d) / 2}{\lceil D/d \rceil d - 1} - \frac{(\lceil D/d \rceil d) / 2 - 1}{\lceil D/d \rceil d - 1} \right)
    \right) \\
    &= -\frac{1}{\lceil D/d \rceil d - 1}
\end{align*}

\end{proof}

\newpage

\section{FURTHER EXPERIMENTS}

\label{sec:app-further-experiments}

In this section, we provide further experiments complementing our evaluation in Section~\ref{sec:experiments} of the main paper.

\subsection{Empirical Variance Comparison of (CtR-) Rademacher Sketches}
\label{sec:further-var-comp}

We first study the practical effect of the non-negativity condition $a = \sum_{i=1}^d \sum_{j' \neq i}^d x_i x_{j'} y_i y_{j'} \geq 0$ in Thm.~\ref{thrm:rad-ctr-advantage}. Fig.~\ref{fig:ctr-rad-r-rad} shows the results of an empirical variance comparison of CtR-Rademacher sketches against their real analogs.

Fig.~\ref{fig:ctr-rad-r-rad-a} shows the case, where the condition $a \geq 0$ always holds (non-negative data) and Fig.~\ref{fig:ctr-rad-r-rad-b} the case, where $a \geq 0$ does not always hold (zero-centered data). While the CtR sketch offers lower variance ratios for CIFAR-10 and MNIST in most cases even if $a \geq 0$ does not always hold, we see that $a \geq 0$ is needed to \textit{guarantee} an advantage of the CtR sketch. For Letter and Mocap with zero-centered data (Fig.~\ref{fig:ctr-rad-r-rad-b}), around half the variances ratios are less than one and half are more than one, suggesting that real Rademacher sketches perform similarly to CtR-Rademacher sketches in this case. For non-negative data (Fig.~\ref{fig:ctr-rad-r-rad-a}), the relative gains of CtR-sketches improve drastically. That is, all variance ratios are less than one, with an increasing gain for larger $p$.

\begin{figure}[ht]
\begin{subfigure}{.5\textwidth}
    \centering
    \includegraphics[width=1\linewidth]{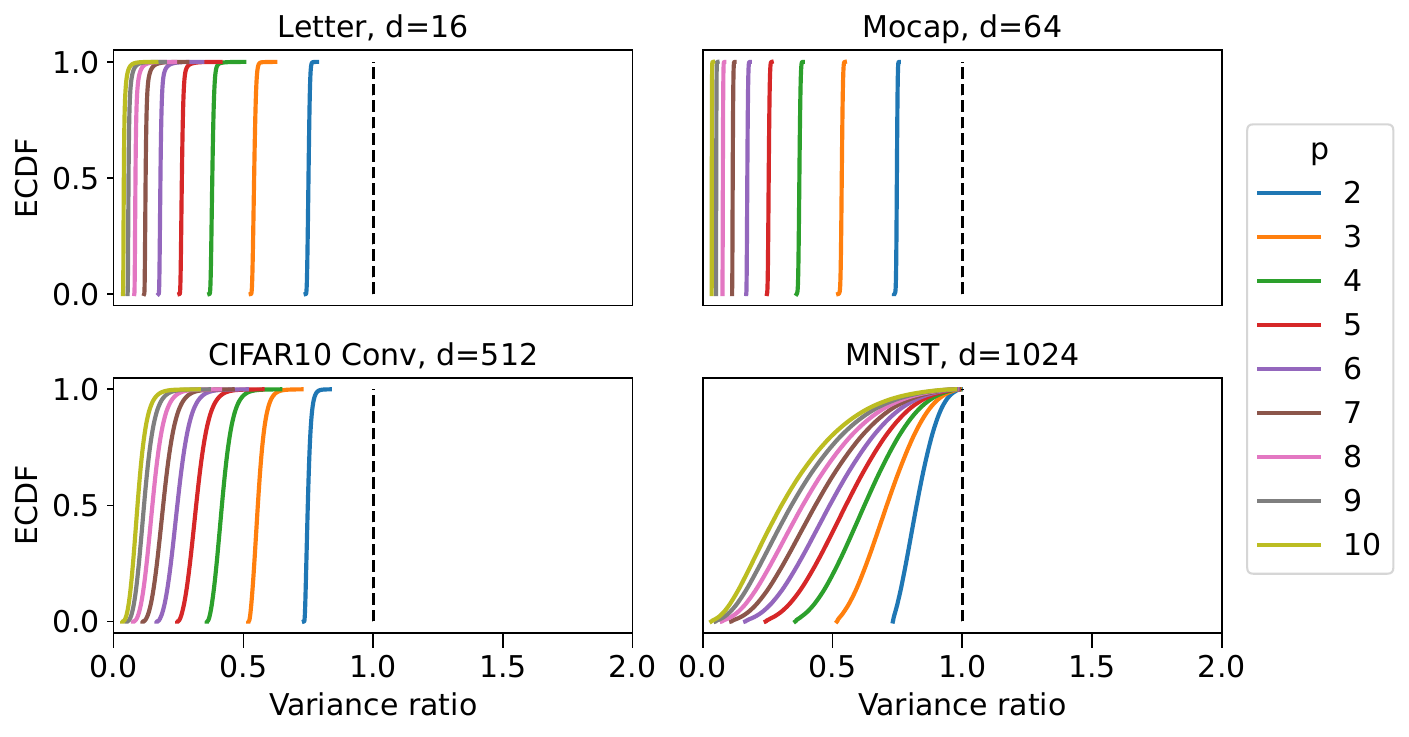}
    \caption{Non-negative data.}
    \label{fig:ctr-rad-r-rad-a}
\end{subfigure}%
\begin{subfigure}{.5\textwidth}
    \centering
    \includegraphics[width=1\linewidth]{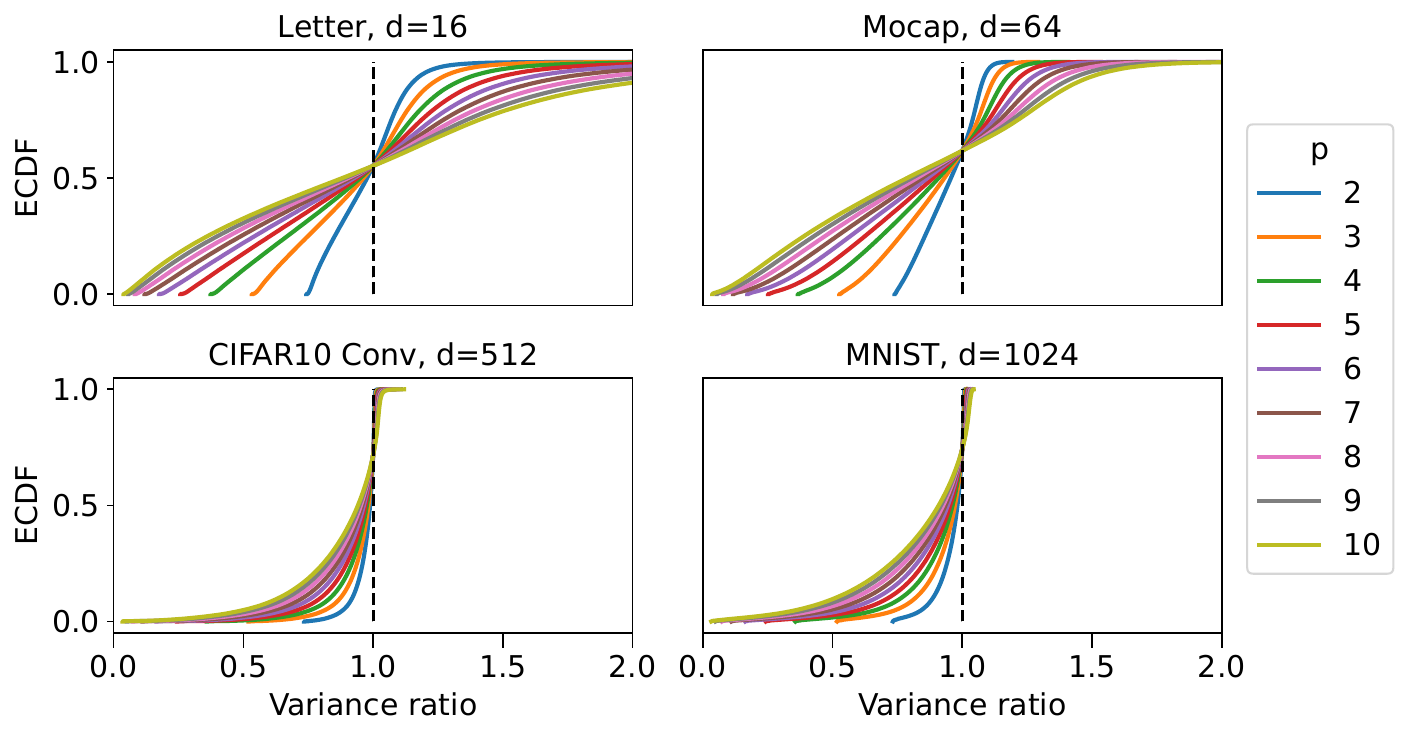}
    \caption{Zero-centered data.}
    \label{fig:ctr-rad-r-rad-b}
\end{subfigure}
\caption{ECDF of Var(CtR-Rademacher) / Var(Rademacher) for pairwise evaluations of the variance ratio evaluated on a subset of each dataset.}
\label{fig:ctr-rad-r-rad}
\end{figure}

\subsection{Closed-Form GP Classification}

\label{sec:further-downstream-comp}

We carry out a set of additional GP classification experiments to complement Section~\ref{sec:gp-classification}. The experiments are the same as in Section~\ref{sec:gp-classification}, but compare a larger range of values for $D$ and two additional data sets: Letter and Mocap \citep{Dua:2019}. Moreover, we add experiments for zero-centered data. The following is a brief summary of the plots:

\begin{itemize}
    \item Fig.~\ref{fig:gp-mnist-cifar10-big} shows MNIST/CIFAR-10 experiments for $p=3,7$ using unit-normalized non-negative data (same as Fig.~\ref{fig:final-comparison} for a larger range of $D$).
    \item Fig.~\ref{fig:gp-letter-mocap-big} shows Letter/Mocap experiments for $p=3,7$ using unit-normalized non-negative data.
    \item Fig.~\ref{fig:gp-mnist-cifar10-big-0mean} shows MNIST/CIFAR-10 experiments for $p=3,7$ using unit-normalized zero-centered data.
    \item Fig.~\ref{fig:gp-letter-mocap-big-0mean} shows Letter/Mocap experiments for $p=3,7$ using unit-normalized zero-centered data.
\end{itemize}

In general, we find that relative performance gains of CtR-sketches over their real analogs are larger for non-negative than for zero-centered data. This makes sense because of the condition of Thm.~\ref{thrm:rad-ctr-advantage}. However, they still lead to some improvements even for zero-centered data. Gains over SRF on the other hand increase for zero-centered data, in particular regarding kernel approximation errors.

\subsection{Online Learning for Fine-Grained Visual Recognition}

\label{sec:app-fine-grained}

Fig.~\ref{fig:online-learning-fine-grained} shows an online learning experiment on the CUB-200 \citep{welinder2010} data set. We follow the experimental setup in \citet{Gao2016}, but only train the classification layer of the VGG-M \citep{ChatfieldSVZ14} convolutional neural network. This option is referred to as \textit{no fine-tuning} in the original paper.

We use an Adam optimizer with decaying learning rate starting from $10^{-3}$, where the learning rate is divided by $10$, when the validation loss stagnates. The mini-batch size is $32$ and we train over $50$ epochs. The sketch dimension is $D=2^{13}$ and $p=3$, $a=2$ in our experiments (see Section~\ref{sec:experimental-setup}).

Our final test errors are lower than 36.42\% and 31.53\% for Rademacher and TensorSketch, respectively, given in \citet[Table 4]{Gao2016}. An exception is SRF that requires unit-normalized features and hence loses important information, leading to around 75\% test error. Since the polynomial degree $p=3$ is small, CtR-ProductSRHT does not achieve an advantage over ProductSRHT in terms of test errors. ProductSRHT is also slightly faster. When using CRAFT maps on the other hand, both CtR-ProductSRHT and ProductSRHT perform similarly well, and are significantly faster than TensorSketch.

\begin{figure}[ht]
\begin{center}
\begin{subfigure}{.5\textwidth}
    \centering
    \includegraphics[width=1\linewidth]{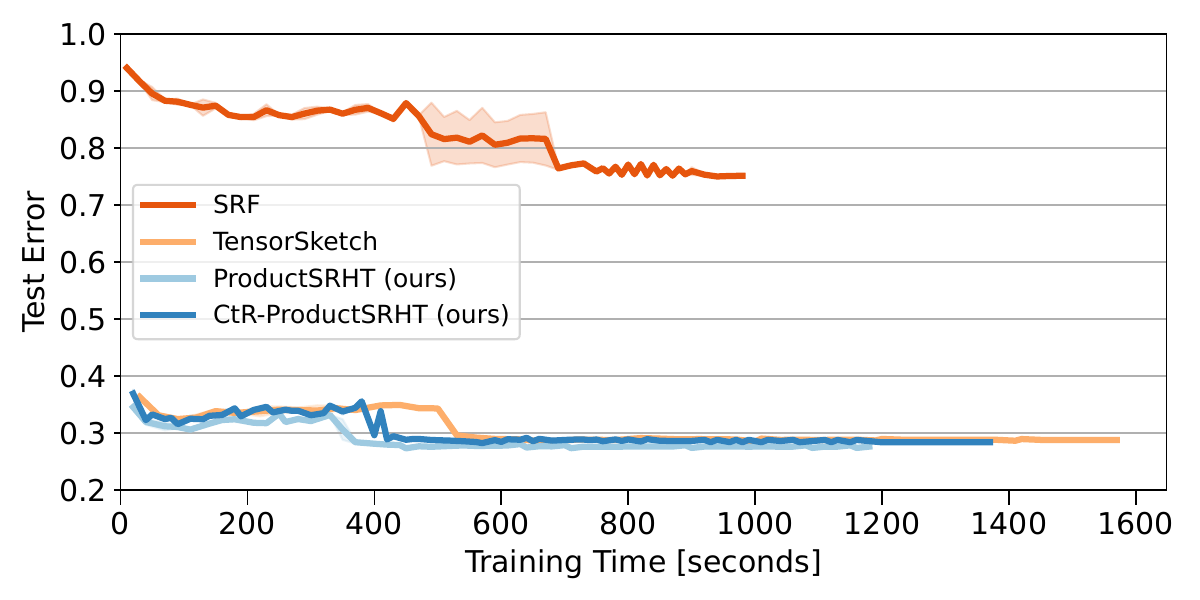}
    \caption{(CtR-) ProductSRHT vs. TensorSketch/SRF.}
    \label{fig:fine-grained}
\end{subfigure}%
\begin{subfigure}{.5\textwidth}
    \centering
    \includegraphics[width=1\linewidth]{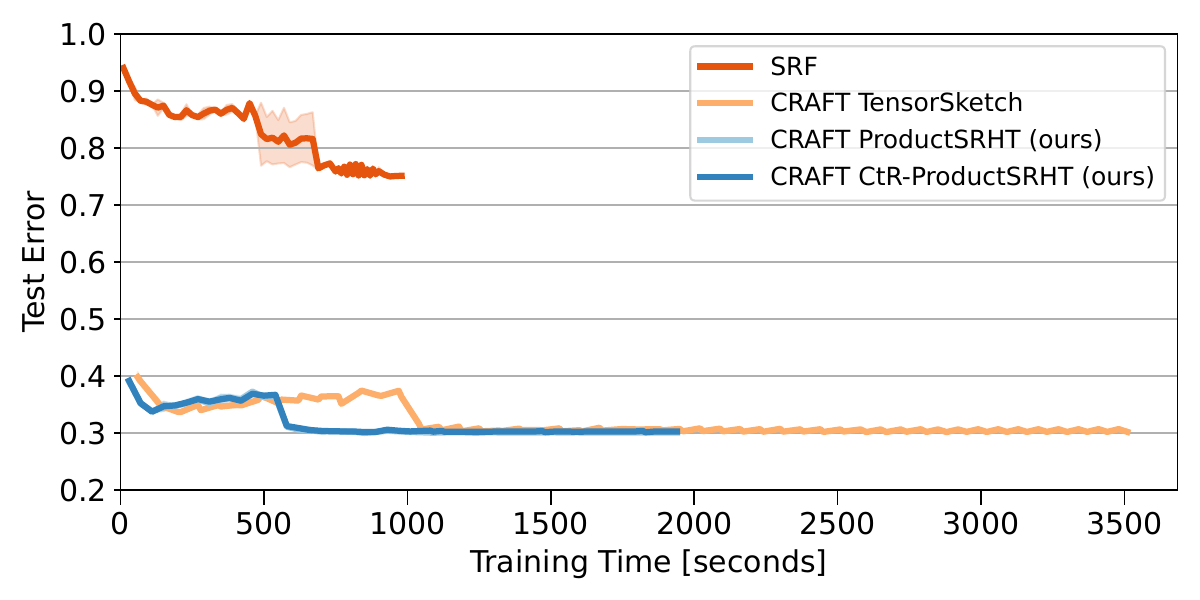}
    \caption{Same as (a) with CRAFT maps.}
    \label{fig:fine-grained-craft}
\end{subfigure}
\caption{Stochastic optimization following \citet{Gao2016} for the CUB-200 data set \textit{without} fine-tuning of the VGG-M convolutional layers.}
\label{fig:online-learning-fine-grained}
\end{center}
\end{figure}


\begin{figure}[ht]
\begin{center}
\centerline{\includegraphics[width=1.0 \textwidth]{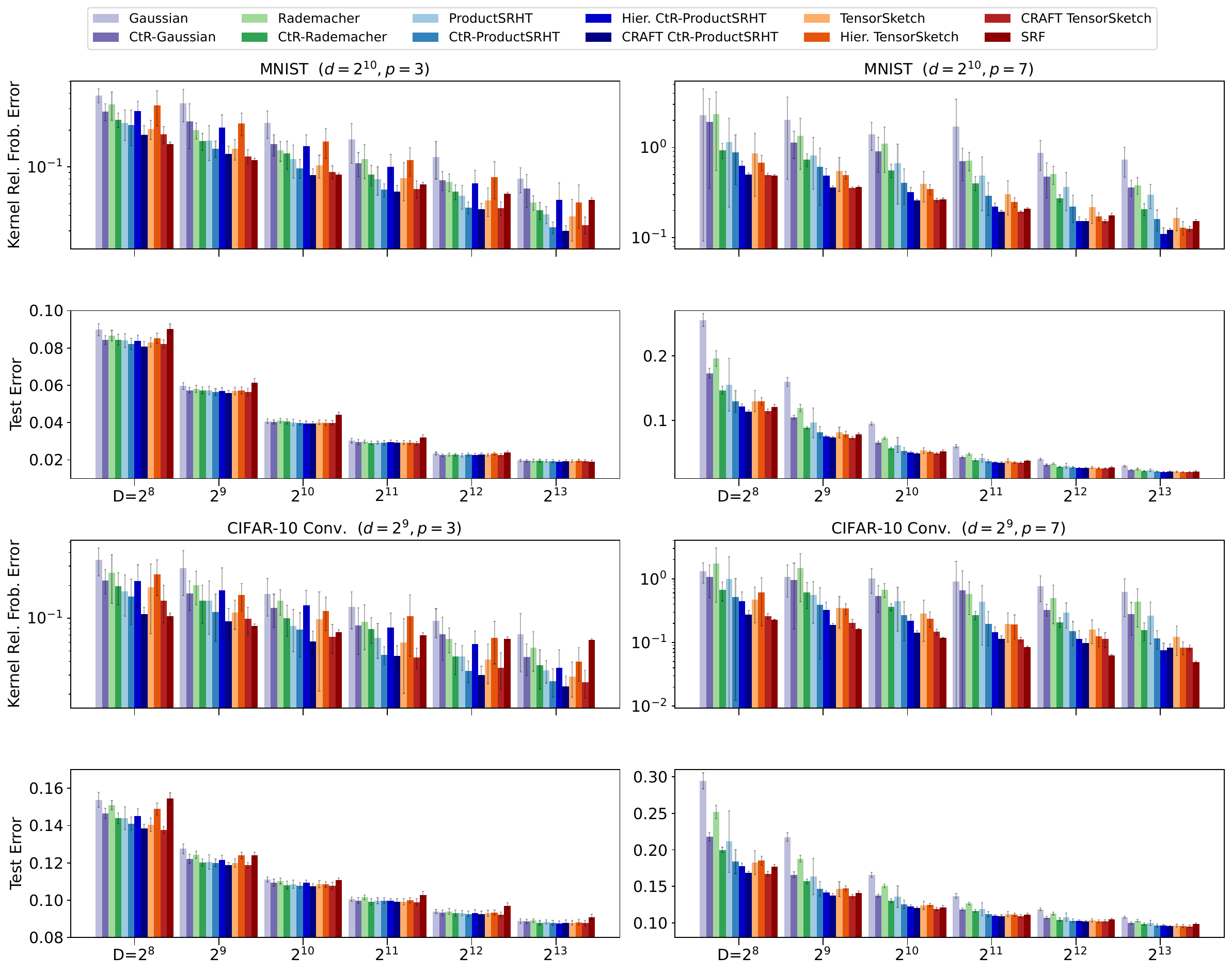}}
\caption{MNIST and CIFAR-10 comparison for $p=3$ and $p=7$ with unit-normalized data averaged over 20 seeds.
}
\label{fig:gp-mnist-cifar10-big}
\end{center}
\end{figure}

\begin{figure}[ht]
\begin{center}
\centerline{\includegraphics[width=1.0 \textwidth]{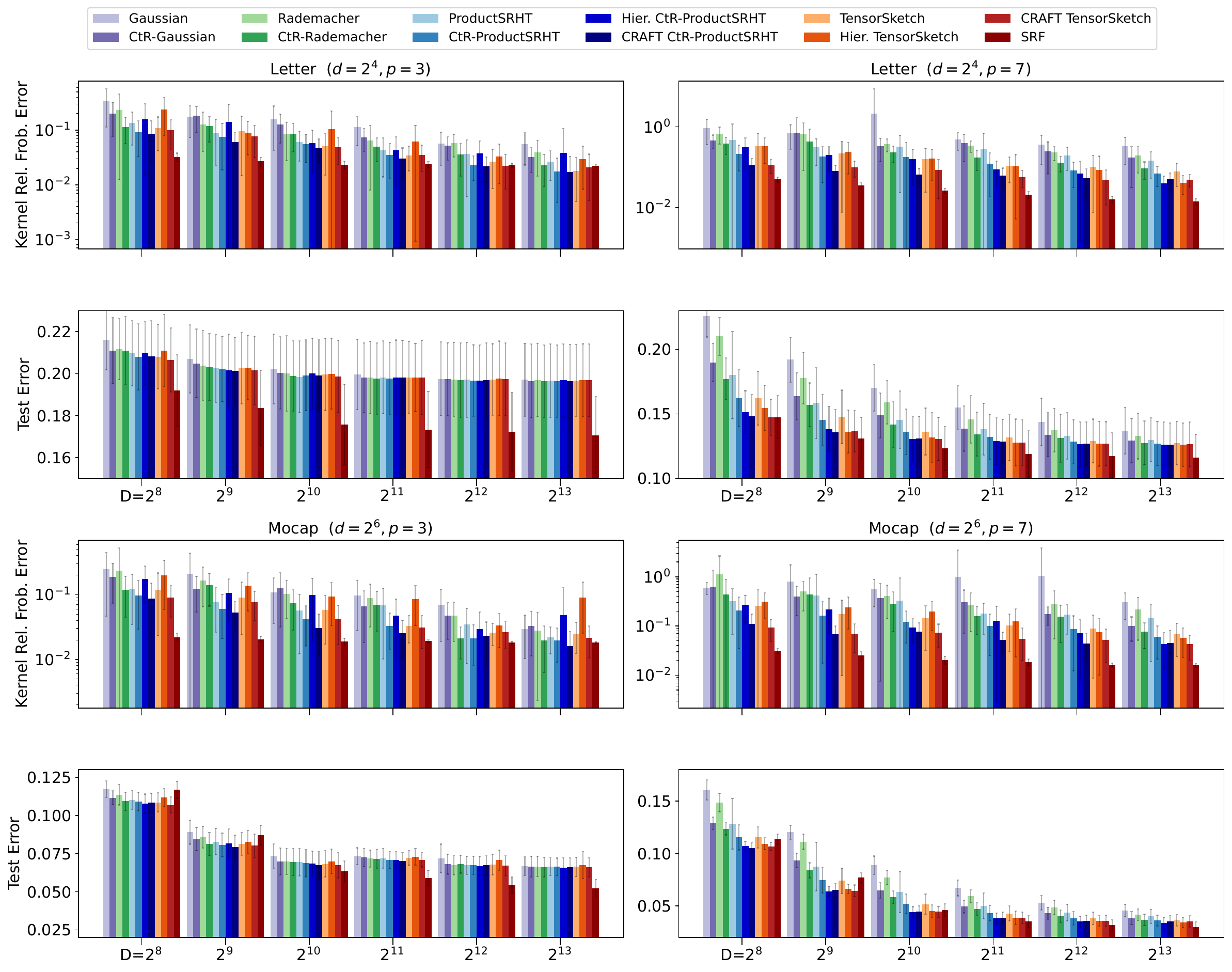}}
\caption{Letter and Mocap comparison for $p=3$ and $p=7$ with unit-normalized data averaged over 20 seeds.
}
\label{fig:gp-letter-mocap-big}
\end{center}
\end{figure}

\begin{figure}[ht]
\begin{center}
\centerline{\includegraphics[width=1.0 \textwidth]{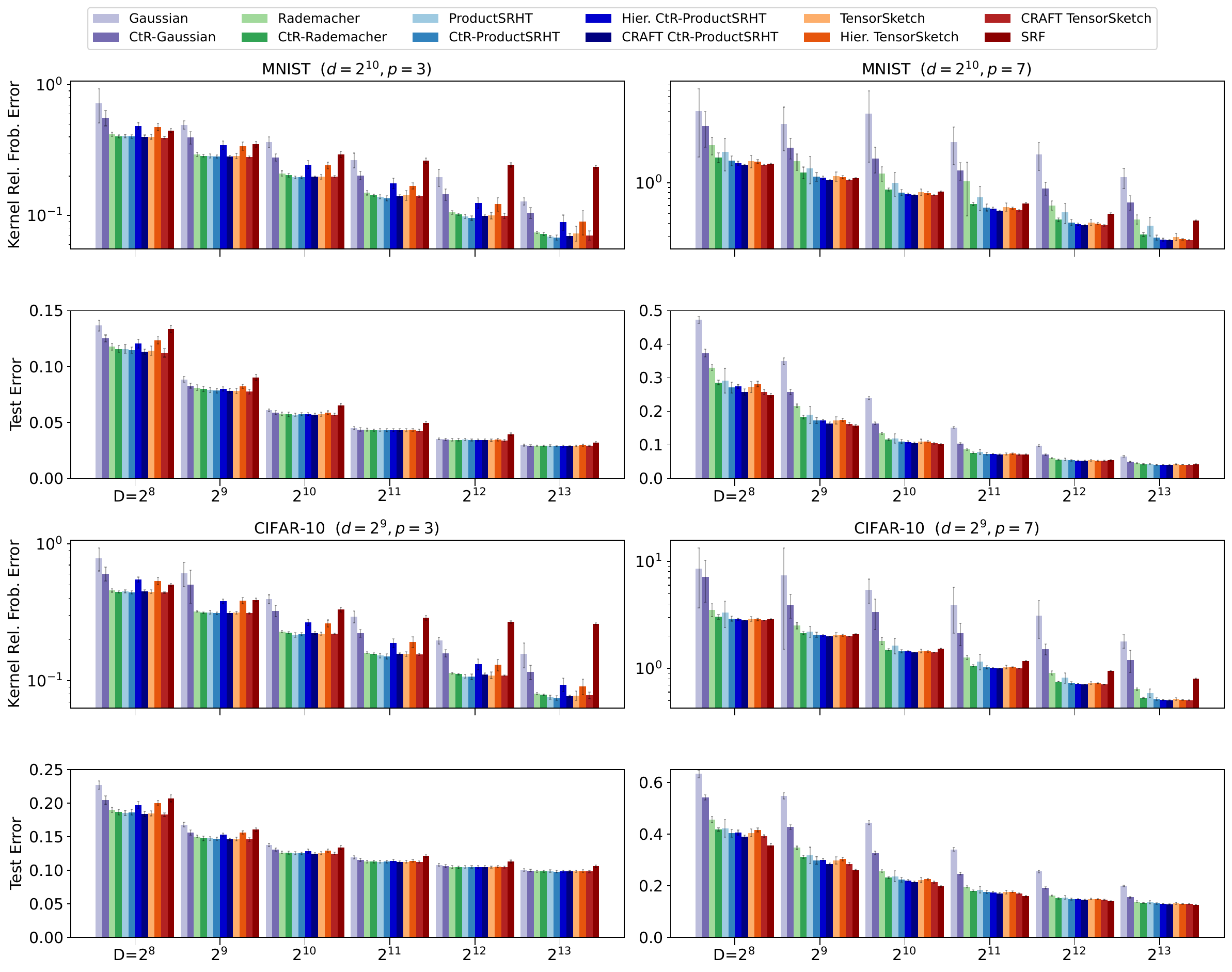}}
\caption{MNIST and CIFAR-10 comparison for $p=3$ and $p=7$ averaged over 20 seeds. The data is centered through a subtraction of the training mean and unit-normalized afterwards.
}
\label{fig:gp-mnist-cifar10-big-0mean}
\end{center}
\end{figure}

\begin{figure}[ht]
\begin{center}
\centerline{\includegraphics[width=1.0 \textwidth]{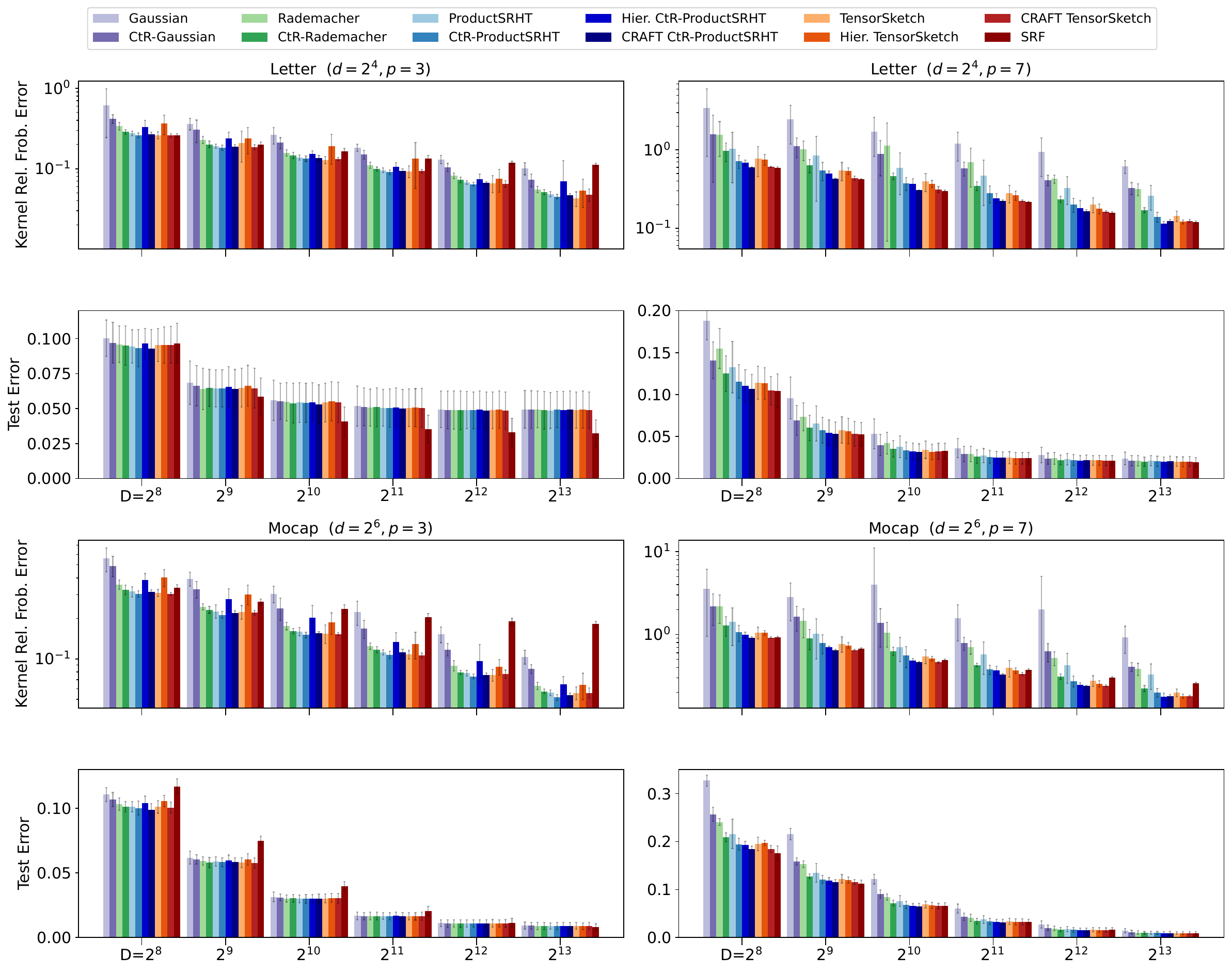}}
\caption{Letter and Mocap comparison for $p=3$ and $p=7$ averaged over 20 seeds. The data is centered through a subtraction of the training mean and unit-normalized afterwards.
}
\label{fig:gp-letter-mocap-big-0mean}
\end{center}
\end{figure}






\end{document}